\documentclass[11pt]{article} % Anonymized submission
\usepackage[utf8]{inputenc}
\usepackage[T1]{fontenc}
\usepackage[english]{babel}
\usepackage[]{algorithm2e}
\usepackage{mwe}
\usepackage{amsmath,mathtools}
\usepackage[dvipsnames]{xcolor}

\usepackage{authblk}
\usepackage{amsthm,bm}
\usepackage{lineno}
\usepackage{tikz,pgfplots}
\pgfplotsset{compat=1.5}
\usepgfplotslibrary{fillbetween}
\usepackage{subfigure}
\usetikzlibrary{intersections}
\usetikzlibrary{patterns}
\usepackage{multirow}

\usepackage[shortlabels]{enumitem}
\usepackage[colorlinks=true, linkcolor=blue, citecolor=blue, hypertexnames=false]{hyperref}  
\usepackage{nameref}
\usepackage{cleveref}
\usepackage{color}              % Need the color package
\usepackage{color-edits}
\addauthor{Will}{blue}

\usepackage[right=1.0in, top=1in, bottom=1.0in, left=1.0in]{geometry}

\usepackage{lmodern}

\usepackage{graphicx}
\usepackage{lscape}
\usepackage[final]{pdfpages}
\usepackage{amsthm}
\usepackage[authoryear,round]{natbib}
\usepackage{amsfonts}
\usepackage{mathtools}
\usepackage{bbm}
\usepackage{dsfont}

\usepackage{bbm}
\newtheorem{lemma}{Lemma}
\newtheorem{proposition}{Proposition}
\newtheorem{corollary}{Corollary}
\newtheorem{theorem}{Theorem}
\newtheorem*{theorem*}{Theorem}

\newtheorem{definition}{Definition}

\newtheorem{remark}{Remark}

\usepackage{setspace}
\newcommand {\beq}{\begin{equation}}
\newcommand {\eeq}{\end{equation}}
\newcommand {\beqn}{\begin{equation*}}
\newcommand {\eeqn}{\end{equation*}}
\newcommand {\bear}{\begin{eqnarray}}
\newcommand {\eear}{\end{eqnarray}}
\newcommand {\bearn}{\begin{eqnarray*}}
\newcommand {\eearn}{\end{eqnarray*}}

\DeclareMathOperator*{\argmin}{arg\,min}

\usepackage[normalem]{ulem}

\newcommand{\X}{\mathcal{X}}
\newcommand{\Y}{\mathcal{Y}}

\newcommand{\ddpol}{\pi(\mathbf{x}, \mathbf{y})}
\newcommand{\wreg}[2][]{\mathsf{Reg}_n(#1,#2)}

\newcommand{\actVr}{a}
\newcommand{\mesOut}{F_{0}}

\newcommand{\mesIn}[1]{F_{#1}}
\newcommand{\ERM}{\pi^{\mathrm{ERM}}}
\newcommand{\LERM}{k^*\text{-}\mathrm{ERM}^{\dagger}}
\newcommand{\ERMt}{\mathrm{ERM}^{\dagger}}

\newcommand{\rad}{\epsilon}

\newcommand{\numS}{n}

\newcommand{\cu}{c_u}
\newcommand{\co}{c_o}

\title{
From Contextual Data to Newsvendor Decisions:\\
On the Actual Performance of Data-Driven Algorithms
}

		\author[1]{Omar Besbes}
		\author[2]{Will Ma}
		\author[3]{Omar Mouchtaki}
		\affil[1]{Columbia University, Graduate School of Business, \texttt{obesbes@columbia.edu}}
		\affil[2]{Columbia University, Graduate School of Business, \texttt{wm2428@gsb.columbia.edu}}
		\affil[3]{NYU Stern School of Business, \texttt{om2166@stern.nyu.edu}}

\date{first version: February 16, 2023; last revised: September 21, 2025}
\begin{document}

\maketitle

\begin{abstract}%
In this work, we study how the relevance/quality and quantity of past data influence performance  by analyzing a contextual Newsvendor problem, in which a decision-maker trades off between underage and overage costs under uncertain demand. We consider a setting in which past demands observed under ``close by'' contexts come from   close by distributions and analyze the  performance of data-driven algorithms through a notion of  \textit{context-dependent}  worst-case expected regret. We analyze  the broad class of Weighted Empirical Risk Minimization (WERM) policies which weigh past data according to their similarity in the contextual space. This class includes classical policies such as ERM, $k$-Nearest Neighbors and kernel-based policies. 
Our main methodological contribution is to characterize \textit{exactly} the worst-case regret of any WERM policy on any given configuration of contexts. 
To the best of our knowledge, this provides the first understanding of tight performance guarantees in any contextual decision-making problem, with past literature focusing on upper bounds via concentration inequalities.
We instead take an optimization approach, and isolate a structure in the Newsvendor loss function that allows to reduce the infinite-dimensional optimization problem over worst-case distributions to a simple line search.
This in turn allows us to  unveil fundamental insights that  were obfuscated by previous general-purpose bounds. We characterize  \textit{actual} guaranteed performance as a function of the contexts, as well as granular insights on the learning curve of algorithms.

\textbf{keywords:} data-driven stochastic optimization, weighted empirical risk minimization, k-nearest neighbors, contextual newsvendor, worst-case regret, data quality

\end{abstract}

\section{Introduction}\label{sec:intro}

A fashion retailer is planning how many units to produce for a seasonal product (e.g.  pairs of speckled gray gloves), in anticipation of demand, for the next selling season. It has historical sales numbers from recent years for it, as well as similar SKUs (e.g., speckled navy gloves).
How should the retailer use this past data collected in contexts different from the one in which it is currently operating?
Motivated by this question, this work aims to broadly understand the impact of data size and data relevance on performance in the contextual Newsvendor. The Newsvendor problem is a prototypical model of decision-making in the face of uncertainty. This problem captures the common trade-off which emerges when making decisions to balance between an underage cost incurred when the resource is insufficient compared to the realized outcome and an overage cost which materializes when some resource is wasted, with this resource being the gloves that are produced and not sold in the example above. This model is widely applied across different areas such as inventory management, capacity planning, scheduling or overbooking. Beyond the operational aspect of the model, the Newsvendor loss is of particular interest for its statistical interpretation; it generalizes the $\ell_1$ loss function to allow for different weights on the error types (overestimating or underestimating) and coincides with the quantile loss used in quantile regression.

In this contextual Newsvendor problem, past data of the form $(x_i, y_i)_{i=1}^n$, where $x_i$ denotes the context and $y_i$ the demand observed under context $x_i$, are used to take an action $a$ for a future (``out-of-sample'') context $x_0$. The loss of the action $\ell(a,y_0)$ depends only on an unknown quantity $y_0$ associated with $x_0$.
In this example, $y_0$ is the demand of speckled gray gloves to materialize next winter, and the ideal production action $a=y_0$ would have loss $\ell(a,y_0)=0$. However, two main challenges are present: $i.)$  $y_0$ is not known in advance; and $ii.)$ it is drawn according to a distribution unknown to the decision-maker.
The observed $y_1,\ldots,y_n$ help to predict the distribution of $y_0$, and the contexts $x_0,x_1,\ldots,x_n$ are relevant only to the extent that they indicate how each $y_i$ may be informative to predict the distribution of $y_0$.
In particular, we assume a geometry (formalized later) where contexts $x_i$ that are closer to $x_0$ are likely to have $y_i$ values that are drawn from distributions ``closer'' to that of $y_0$.
In this example, the context $x_0$ could be described by the vector (speckled, gray, 2024), and closer $x_i$ would have more similar SKU features and come from more recent weeks.

It is generally presumed that the retailer can make a better production decision if it has more data and more contexts $x_i$ that are close to $x_0$, and should place greater emphasis on the $y_i$ values associated with these nearby $x_i$'s when making its decision.
However, despite the prevalence of contextual decision-making in theory and practice, many fundamental questions are still not well understood.
How do the dissimilarities between the available contexts $x_1,\ldots,x_n$ and $x_0$, or the quality of the historical data for the focal decision,  affect performance?
What levels of performance can be achieved as a function of the data at hand, and how many samples are needed to operate well?
In the present paper, we explore a framework to study these questions systematically for the Newsvendor problem.

\noindent
\subsection{Main Contributions}
\label{sec:contributions}

To analyze how the quantity and relevance of historical data affect the performance of data-driven policies, we adopt a context-dependent notion of performance. We model demand distributions as unknown and context-specific, with the distance between distributions bounded by the dissimilarity of their contexts, a condition we refer to as the local condition (\Cref{def:local_const}). Policy performance is measured through expected regret, defined as the excess loss relative to the best action under the out-of-sample distribution, and we seek guarantees that hold in the worst case over distributions satisfying the local condition. By characterizing worst-case regret for different configurations of past and out-of-sample contexts, this framework captures a range of practical scenarios and makes it possible to isolate the impact of data relevance and data quantity on performance. Within this setting, we study the widely used class of Weighted Empirical Risk Minimization (ERM) policies, which encompasses ERM, k-Nearest Neighbors, and Kernel methods.

\subsubsection{Performance characterization: learning without concentration}

Contrasting the typical approach of upper-bounding learning regret via concentration inequalities, we take a more precise approach of
trying to quantify the \textit{worst-case} learning regret.  
A priori, this optimization problem of identifying the hardest-to-learn distributions is difficult to solve, yet it is necessary for deriving insights about actual worst-case performance.
Our main methodological contribution is solving this optimization problem using the special structure of the Newsvendor loss function, a method we dub ``learning without concentration.''

In particular, we investigate in \Cref{sec:main_result} the special structure of the Newsvendor loss function, in conjunction with that of  Weighted ERM policies and the Kolmogorov distance to derive structural properties of the optimization problem at hand. In that setting, we characterize \textit{exactly} the structure of worst case distributions which maximize regret for any number of samples (i.e., any data quantity) and any past context (i.e., any data relevance). Our main result, formally presented later as \Cref{cor:regret_WERM}, is informally stated below.

\begin{theorem*}[Main result, informal version]
\label{thm:informal_main}
Assume the data-generation process satisfies the local condition (see \Cref{def:local_const}).
For any sample size and any historical contexts, the problem of determining the worst-case regret of any Weighted ERM policy can be reduced from a non-convex constrained infinite-dimensional optimization problem to a one-dimensional optimization problem. In turn, the worst-case regret can be exactly quantified.
\end{theorem*}

To prove our result, we move beyond the specific class of Weighted ERM policies and adopt a more general perspective by introducing the broader class of ``separable'' policies. This generalization allows us to clearly identify the essential factors driving our result: specifically, how the Newsvendor loss interacts with certain properties of these policies.
In particular, we show that, under the Newsvendor loss, the cumulative distribution function (cdf) of the decision prescribed by a separable policy can be expressed as a polynomial of the cumulative distribution functions of the historical samples observed. 
We then show that the worst-case performance for a given separable policy is achieved when the historical and out-of-sample distributions are Bernoulli distributions (Proposition \ref{prop:reduction_bern}). This allows to reduce the initial infinite dimensional problem to an $(\numS+1)$-dimensional optimization problem over the possible means of the Bernoulli distributions. 
Our second reduction establishes that the worst-case sequence of historical distributions admits a given structure as a function of the (unknown) out-of-sample distribution (Proposition \ref{prop:worst_history}). This result is obtained by exploiting the monotonic behavior of the regret of a non-decreasing separable policy with respect to the means of the Bernoulli distributions associated with historical data, allowing to characterize the sequence of worst-case historical means. In turn, the initial problem can be reduced to  a one-dimensional optimization problem over a bounded set which can be solved exactly through a line search.
This completes the reduction for the class of non-decreasing separable policies. Finally, a critical step of our approach consists in proving that the class of non-decreasing separable policies encompasses Weighted ERM (\Cref{prop:counting_to_sep,prop:WEM_to_counting}).

Our analysis provides the strongest results when the local condition is defined through the Kolmogorov distance (\Cref{def:local_const}). Importantly, our optimization-based approach is flexible and extends to other choices of local conditions. For example, in \Cref{sec:apx_Wasserstein} we consider the Wasserstein distance and show that our method yields considerably tighter performance guarantees than state-of-the-art concentration-based approaches, even when the exact worst-case distributions cannot be characterized.

\subsubsection{New insights on learning behavior}

Our results enable us to quantify key objects of interest and revisit the motivating questions: How does the dissimilarity between historical contexts $x_1,\ldots,x_n$ and the out-of-sample context $x_0$ impact performance? What performance levels are achievable given the available data, and how many samples are necessary to operate effectively?

To contextualize our findings, we compare them against the prior state of the art. As articulated earlier, previous general-purpose bounds, which are applicable to general loss functions, have been the workhorse of a large portion of machine learning and sample complexity results. Over the past decades, they enabled to provide an understanding (through upper bounds) of the impact of sample size on the the regret of central policies. Despite their wide applicability, we will show that such bounds may yield overly conservative estimates, failing to accurately reflect the true worst-case performance in specific contexts.

\begin{figure}[h!]
\centering
\begin{tikzpicture}[scale = 0.9]
\begin{axis}[
            title={},
            xmin=0,xmax=200,
            ymin=0,ymax=0.14,
            scaled y ticks={base 10:2},
            width=10cm,
            height=8cm,
            table/col sep=comma,
            xlabel = number of samples $\numS$,
            ylabel = worst-case regret,
            grid=both,
            skip coords between index={0}{1},
            legend pos=south east]

\addplot[gray, line width=.8mm, domain=1:200, samples=600] {1.7/x + 0.08};  % no 'smooth'
\addlegendentry{ERM (shape of previous bounds)}
%\addlegendimage{ultra thick,line width=.4mm,gray}

\addplot [blue,mark=square,mark options={scale=.3}] table[x={Ns},y={epsilon=0.1}] {Data/paper_all_data_SAA_q9.csv};
\addlegendentry{ERM (our bound)}

\addplot [red,mark=square,mark options={scale=.3}] table[x={Ns},y={tweak}] {Data/paper_all_data_mix_q9.csv};
\addlegendentry{$\LERM$}
\end{axis}
\end{tikzpicture}
\caption{\textbf{Illustrative summary of our main insights.} We note that the gray curve depicts the \textit{shape} of ``previous bounds'' on the same scale as the bounds we derive; the actual previous bounds are in fact much higher.}
\label{fig:summary}
\end{figure}

\Cref{fig:summary} summarizes our main findings in a prototypical scenario where all past contexts are identical and differ from the out-of-sample context by a fixed dissimilarity~$\zeta$. We compare our exact results for ERM with the best-known general-purpose regret bound of \citet{mohri2012new}, based on Rademacher complexities and uniform convergence, and uncover three key insights. 

First, general-purpose bounds vastly overestimate the number of samples required to achieve a given performance, whereas our analysis reduces these requirements by 2 to 3 orders of magnitude (see \Cref{tab:detailed} in \Cref{sec:mmBound}). 
Second, our results expand the frontier of what is achievable: performance levels previously thought unattainable by ERM, even with infinitely many samples, can in fact be achieved with only tens of samples. Third, and most importantly, the actual \textit{shape of the learning curve}, defined as the policy's worst-case regret as a function of the sample size, is fundamentally different from what general-purpose bounds suggest. 
While the latter depict a smooth, monotonically decreasing curve approaching a non-zero limit and implying that more data always helps, our exact analysis shows that when $\zeta > 0$, the learning curve of ERM is oscillatory (a behavior previously observed and avoidable through randomization; see \Cref{sec:apx_improve_mix}). More importantly, it is \emph{globally non-monotonic}: performance improves initially, attains a global optimum at a finite sample size, and then deteriorates as additional samples are added. This reveals the existence of an \emph{effective sample size}, beyond which more data harms rather than helps. 
Motivated by this, we introduce the $\LERM$ policy, which uses a subset of available samples, and we show numerically that it achieves near-optimal performance across algorithms by comparing its performance to a universal lower bound on the worst-case regret. Taken together, these findings imply that \emph{after a few samples, data relevance, as opposed to data quantity, is the first-order limitation to improved decision-making.}

Beyond these key insights, we illustrate in \Cref{sec:many_dissimilarities} the generality of our approach by evaluating the worst-case performance of various subclasses of Weighted ERM policies when sample dissimilarity degrades linearly over time. This setting models an important special case of our framework in which the outcome distributions drift over time. 
We observe in this case that the effective sample size is even smaller than the one obtained when all past samples have the same dissimilarity.

Finally, we examine numerically the robustness of our insights beyond the worst-case setting. In \Cref{sec:beyond_worst_case}, we evaluate the performance of Weighted ERM policies on instances that are non-Bernoulli and where the distribution remains fixed as the sample size varies. Our results show that the key insights from the worst-case analysis remain broadly applicable even in these fixed-instance settings.

In summary, our paper provides new theoretical insights into data-driven decision-making, highlighting the limitations of general-purpose bounds and demonstrating the practical value of tailored analyses for understanding achievable performance as a function of data relevance and data size.

\subsection{Literature review}
\label{sec:review}

\noindent
\textbf{Data-driven stochastic optimization.}
Our work broadly falls under the area of data-driven stochastic optimization, where samples are used to infer an unknown out-of-sample distribution.
In the absence of contexts, pioneering works \citep{kleywegt2002sample,kim2015guide} analyze ERM, also referred to as Sample Average Approximation.
In the presence of contexts,
a wide range of statistical learning methods have been developed and analyzed
\citep{vapnik1999overview,bousquet2002stability,gyorfi2002distribution}.
Our paper differs by taking an optimization approach and identifying exact worst-case regret for a specific decision problem.

More specifically related
are papers on contextual decision problems \citep{kao2009directed,donti2017task,elmachtoub2021smart,bertsimas2022data,kannan2022data}. For a broader survey on data-driven decision-making in revenue management we refer the reader to the recent survey \cite{chen2023data}.
Closest to us is a line of work using ``local'' methods which approximate the empirical objective with Dirichlet processes   \citep{hannah2010nonparametric}, k-Nearest Neighbor and Random Forest in single stage \citep{bertsimas2020predictive} or multi-stage problems \citep{bertsimas2019predictions}, and Nadaraya-Watson (NW) kernels \citep{ban2019big,bertsimas2019predictions,fu2021data,srivastava2021data}. \cite{kallus2022stochastic}  introduce forest based weights that are tailored to the downstream optimization problem.
These papers consider methods similar in spirit to what we call Weighted ERM policies. They focus on consistency results or concentration-inequality based upper bounds for general stochastic optimization problem. In contrast, we develop a characterization of the performance of Weighted ERM policies which holds even in settings where consistency cannot be achieved by any data-driven policy. Our methodology is also different as it allows to quantify the impact of the geometry of past contexts on the performance of policies.

\noindent
\textbf{Newsvendor problem.} Our work also relates to the analysis of the Newsvendor problem under partial information about the demand distribution \citep{scarf1958min,gallego1993distribution,perakis2008regret,see2010robust}.
When there are no contexts in the Newsvendor problem, \cite{levi2007approximation,levi2015data} establish probabilistic bounds on the relative regret of ERM, and lower bounds with matching rate were derived in \cite{cheung2019sampling}. Recently, \cite{besbes2021big} characterize exactly the worst-case expected relative regret of ERM for any sample size, and derive a minimax optimal policy and associated performance. Our work significantly generalizes such analysis to the contextual setting, where historical distributions may be different and coupled through the context. We also analyse a broader class of policies to capture Weighted ERM policies. We compare in more detail our analysis to previous ones in \Cref{sec:main_result}.

Turning to contextual Newsvendor, \cite{ban2019big} provide finite sample bounds on the performance of ERM with a linear hypothesis class and for some Kernel methods. \cite{qi2021distributionally} analyze robust policies and \cite{gaooptimal} establish a closed-form solution for a DRO problem under Wasserstein uncertainty and derive generalization bounds for their method.  Under a ``local'' model similar to ours, \cite{fu2021data} derive finite bounds for robustified Kernel methods.
All these works rely on concentration-based arguments, in contrast our ``learning without concentration'' analysis is based on an optimization approach and further leverages the structure of the Newsvendor loss function at hand. Furthermore, these bounds are derived in settings in which context vectors are sampled from a distribution and the performance of policies is evaluated by taking an expectation over contexts. Therefore, the bounds cannot be used to derive context-dependent guarantees, which is the object of our work.  Recently  
\cite{huber2019data,oroojlooyjadid2020applying,qi2022practical} derive approaches for contextual Newsvendor based on neural-networks. These are focused on empirical performance, whereas our focus is on provable guarantees.

Other approaches have also been developed to account for heterogeneity without explicitly incorporating contextual information. Inspired by Bayesian settings, \cite{gupta2022data} explore the impact of data-pooling methods, \cite{gupta2021small} analyze the small-data large-scale setting in which the decision-maker is facing several problems which are similar and \cite{besbes2022beyond} analyze, for multiple loss functions including Newsvendor, a context-free model of heterogeneity. Our work generalizes their model of heterogeneity and we provide tight finite sample guarantees whereas their results are asymptotic.

\noindent \textbf{Learning curves.} Our work also sheds light on the behavior of learning algorithms and on their learning curves (see the survey by \cite{mohr2022learning}). Data relevance is acknowledged as being an important limitation for learning but there is no theoretical quantification of this impact. \cite{cortes1994limits} empirically explore this issue and \cite{batini2009methodologies,gudivada2017data} survey qualitative aspects related to relevance of data in practical applications. Furthermore, some of the ``bad behaviors'' of ERM unveiled in our work for the minimax setting are related to the ones surveyed in \cite{loog2022survey} for fixed distributions.

Finally, we note that our work broadly relates to
quantile regression \citep{bhattacharya1990kernel,chaudhuri1991nonparametric,koenker2005quantile,koenker2017quantile},
domain adaptation \citep{redko2020survey,farahani2021brief}, transfer learning \citep{pan2010survey,zhuang2020comprehensive}, active learning \citep{settles2009active}, few shot learning \citep{wang2020generalizing}, and learning in non-stationary environments \citep{rakhlin2011online,mohri2012new,bilodeau2020relaxing,blanchard2022universal}. The work closest related to ours, \cite{mohri2012new}, is discussed in \Cref{sec:insights}.

\subsection{Notation and proofs} For any set $A$, $\Delta \left( A \right)$ denotes the set of probability measures on $A$. Furthermore when $A \subset \mathbb{R}$, and we consider $F \in \Delta \left( A \right)$, we actually use $F$ as the cumulative distribution function of the random variable. For any $\mu \in \mathbb{R}$, with some abuse of notation, we denote by $\mathcal{B}(\mu)$ the Bernoulli distribution with mean equal to $0$ if $\mu \leq 0$, $1$ if $\mu \geq 1$ and $\mu$ in all other cases. For every $x \in \mathbb{R}$, we let $x^{+}:=\max(0,x)$ be the positive part operator.
\textit{All proofs are deferred to the Appendix.}

\section{Problem Formulation}
\label{sec:formulation}

We consider a problem in which,  given a distribution of outcome $\mesOut$ supported on the  space $\Y$, the goal of the decision-maker is to select a decision in the same space $\actVr \in \Y$  to minimize the expected loss
\begin{equation*}
L(\actVr,\mesOut) := \mathbb{E}_{y \sim \mesOut} [ \ell(\actVr,y)],
\end{equation*}
where $\ell$ is the Newsvendor loss defined for every action $\actVr \in \Y$ and outcome $y \in \Y$ as,
\begin{equation*}
\ell(\actVr,y) = \co \cdot ( \actVr - y)^{+} + \cu \cdot (y - \actVr)^{+}.
\end{equation*}
The Newsvendor loss is a form of weighted mismatch cost. One can think of the outcome as being the demand and the decision as being the inventory ordered for this demand. $\co$ denotes the unit cost of going ``over'', i.e., ordering a unit of inventory that goes undemanded; whereas $\cu$ denotes the unit cost of going ``under'' by missing out on a sale due to lack of inventory. We impose without loss of generality the normalization condition $\co + \cu = 1$ and define the critical quantile $q := \cu/(\co+\cu)$. 
This quantile plays a crucial role since, given a known distribution supported on $\Y$ with cumulative distribution $\mesOut$, one can show that a solution which minimizes the expected loss $L(\cdot,\mesOut)$ is the $q^{th}$ quantile of the demand distribution defined as
\begin{equation} \label{eqn:nvRule}
    a^*_{\mesOut} := \inf \{ a \in \Y \; \vert \; \mesOut(a) \geq q \}.
\end{equation}

\noindent
\textbf{Data-generation process.}
In practice, the demand distribution is not known to the decision-maker.
Instead we assume that she observes a dataset in which each data point is a couple $(x,y) \in \X \times \Y$ such that $x$ is the context in which the outcome $y$ has been observed. In this work, we assume that $\Y$ is a bounded interval and assume without loss of generality\footnote{The analysis directly generalizes to support $[0,M]$ for any $M>0$ with appropriate modifications. When $\Y$ is unbounded, no policy can achieve a finite worst-case expected regret (defined in \eqref{eq:fixed_design}).} that $\Y = [0,1]$.

Consider a sequence of $\numS$ historical context vectors $(x_{i})_{i \in \{1,\ldots,\numS\}} \in \X^{\numS}$ observed by the decision-maker. For every $i \in \{1,\ldots,\numS\}$ we denote  by $F_{i}$ the demand distribution conditional on the context $x_{i}$.
We assume that all outcome samples are independently distributed. Hence the vector of historical outcome observations $\bm{y} \in \Y^{\numS}$ is drawn from the product distribution $F_{1} \times \ldots \times F_{\numS}$.

Without any assumption on the distributions $(F_{i})_{i \in \{1,\ldots,\numS\}}$ and $F_{0}$, past outcomes may be useless and completely unrelated to the future outcomes generated from the out-of-sample distribution. We now introduce the local condition, which serves to quantify relations between different distributions.
\begin{definition}[Local condition]
\label{def:local_const}
Given a non-negative function $d(\cdot,\cdot): \X \times \X \to \mathbb{R}$, we say that two distributions $F$ and $F'$ associated to the contexts $x,x' \in \X$ satisfy the local condition if $\| F - F'\|_K \leq d(x,x')$, where $\| \cdot \|_K$ is the Kolmogorov distance defined as $\|F - F'\|_K = \sup_{y \in \mathbb{R}} \left| F(y)  - F'(y) \right|.$ We call $d(\cdot,\cdot)$ the dissimilarity function.
\end{definition}

We note that \Cref{def:local_const} is a local and non-parametric structural condition which imposes that similar contexts (where similarity is defined through the function $d$) induce similar conditional outcome distributions measured by their Kolmogorov distance. As a consequence, we colloquially say that for a given $i \geq 1$, outcome $y_i$ has a high relevance when $F_0$ and $F_i$ satisfies the local condition and the dissimilarity $d(x_0,x_i)$ is low. In \Cref{sec:apx_Wasserstein}, we illustrate how our approach can be generalized when the local condition is based on the Wasserstein distance.

\noindent
\textbf{Data-driven Policy and Objective.} 
Given an observed new context $x_{0} \in \X$, and a dataset containing $\numS$ past context vectors in $\X$ with $\numS$ samples (outcome realizations) in $\Y$, we formally define a (potentially randomized)\footnote{For simplicity of notation, when $\pi$ is deterministic we will denote by $\pi(\bm{x},\bm{y})$ the action $\actVr$ played by $\pi(\bm{x},\bm{y})$ as opposed to the point-mass distribution which puts all mass at $\actVr$. }  data-driven policy  as a mapping $\pi$,
\begin{equation*}
\pi : \begin{cases}
\X^{\numS+1} \times \Y^\numS \to \Delta \left(\Y \right)\\
(\mathbf{x}, \mathbf{y}) \mapsto \pi(\mathbf{x}, \mathbf{y}),
\end{cases}
\end{equation*}
where $\mathbf{x} = (x_0, x_1, \ldots, x_\numS)$ contains the new out-of-sample context $x_0$ and the previously observed ones,  $\mathbf{y}$ contains the $\numS$ observed outcomes, and $\Delta(\Y)$ is the space of the probability distributions supported on $\Y$.

A data-driven inventory policy maps past observations and  the new observed context to a randomized inventory decision without the knowledge of the underlying past demand distributions nor the out-of-sample one. 
The goal of the decision-maker is to select a data-driven policy to minimize the expected regret which represents the difference between the cost incurred by the decision-maker and the cost of an oracle who knows exactly the underlying demand distribution. Formally, the expected regret is defined for any decision $\actVr \in \Y$ and distribution $\mesOut$ supported on $\Y$ as,
\begin{equation*}
R(\actVr,\mesOut) = L(\actVr,\mesOut) -  L(a^*_{\mesOut},\mesOut).
\end{equation*}
If we have a randomized action defined by a distribution $\nu \in \Delta(\Y)$, then we overload notation to let $L(\nu,\mesOut) = \mathbb{E}_{a\sim\nu}[L(a,y)],$ and we let  $R(\nu,\mesOut) =   L(\nu,\mesOut) -  L(a^*_{\mesOut},\mesOut)$.

Our goal in this work is to understand the expected regret of data-driven policies in a fixed-design setting for arbitrary distributions satisfying the local condition (see \Cref{def:local_const}). To accomplish this, we first consider a natural case in which the decision-maker assumes that samples observed in the same context are drawn from the same distribution. 
Formally, we model this setting by imposing that our dissimilarity function satisfies $d(x,x) = 0$ for every $x \in \X$. Then, given $\mathbf{x} = (x_{0}, x_1,\ldots,x_{\numS}) \in \X^{\numS+1}$, an out-of-sample context vector and a sequence of historical context vectors, the worst-case regret of a data-driven policy $\pi$ is then defined as,
\begin{equation}
\label{eq:multiple_samples}
\sup_{F_{0} \in \Delta \left( \Y \right)}  \sup_{ \substack{F_{1},\ldots, F_{\numS} \in \Delta\left( \Y \right)\\
\|F_{i} - F_{j}\|_K \leq d(x_i,x_j)\, \forall i,j \in \{0,\ldots,\numS\}}} \mathbb{E}_{\bm{y} \sim F_{1} \times \ldots \times F_{\numS}} \big[R(\ddpol,F_{0}) \big].
\end{equation}
The worst-case regret in \eqref{eq:multiple_samples} reflects the assumption that samples observed in the same context are drawn from a common distribution because for every $i,j \in \{1,\ldots,\numS\}$ such that $x_i = x_j$, the local condition implies that $\|F_i - F_j \|_K \leq d(x_i,x_j) = 0$, and thus $F_i = F_j$.

To facilitate generalizations (particularly for settings with partially observed features; see \Cref{rem:confounders}), we next define a stronger notion of regret that allows Nature to choose potentially different distributions for each data point, even when the contexts coincide. For every $\numS \geq 1$ and every $\mathbf{x} = (x_{0}, x_1,\ldots,x_{\numS}) \in \X^{\numS+1}$, we define the worst-case regret as
\begin{equation}
\label{eq:fixed_design}
\wreg[\pi]{\mathbf{x}} = \sup_{F_0 \in \Delta \left( \Y \right)}  \sup_{ \substack{F_1,\ldots, F_\numS \in \Delta\left( \Y \right)\\
\|F_0 - F_i\|_K \leq d(x_0,x_i)\, \forall i \in \{1,\ldots,n\} } } \mathbb{E}_{y_1\sim F_1,\ldots,y_n\sim F_\numS} \big[R(\ddpol,F_0) \big].
\end{equation}
The formulation \eqref{eq:fixed_design} differs from \eqref{eq:multiple_samples} by relaxing the local conditions on all pairs of distributions which do not include the out-of-sample one. In particular, \eqref{eq:fixed_design} appears to be more conservative than \eqref{eq:multiple_samples}. For example, if one wants to model a setting in which the decision-maker has observed $\numS$ samples in the same historical context, then \eqref{eq:multiple_samples} forces Nature to select a single historical distribution $F_{1}$ from which the $\numS$ observations are drawn i.i.d., whereas \eqref{eq:fixed_design} allows Nature to select a different distribution for each of the $\numS$ samples, as long as each historical distribution is reasonably close to the out-of-sample one. 

Interestingly, our main result shows that for all policies of interest considered in this paper, the two definitions ultimately yield identical worst-case regrets whenever the dissimilarity function $d$ satisfies the triangular inequality (see \Cref{cor:multiple_sample_to_one}).

Furthermore, we note that in our definition of the worst-case regret, the contexts are fixed and a policy is evaluated against all possible historical and out-of-sample distributions that satisfy the local condition.
We highlight that the data-generation process we study differs from the common one in the literature in which context vectors are randomly sampled from a distribution supported on $\X$. In those settings, the regret guarantees obtained hold uniformly across various distributions of contexts. Hence, these approaches do not allow to characterize the impact of the configuration of \textit{actual} contexts observed on the performance. In contrast, our formulation considers fixed contexts and aims at understanding the regret of policies for different context configurations. Therefore, $\wreg[\pi]{\mathbf{x}}$ can be interpreted as the \textit{context-dependent} robust (minimal) value associated with the data at hand under a particular policy $\pi$. This object quantifies, for a fixed sample size and a fixed configuration of the context vectors, the worst-case performance of a policy of interest. We note that when $d(x_i,x_j) = 0$ for all $i,j \in \{0,\ldots,\numS\}$, the local conditions implies that all distributions should be the same and therefore one retrieves the  setting in which future  and past demands are drawn from the same distribution.

\begin{remark}[Partially Observed Contexts]\label{rem:confounders}
We note that our model, through its flexibility in the dissimilarity function, captures problems in which the decision-maker only observes partially the contexts. Indeed, we do not impose in \eqref{eq:fixed_design} that samples observed in the same context have the same distribution.
 Moreover, we do not impose the dissimilarity function to be a distance. Therefore, we can capture partial observation of contexts by considering a dissimilarity function $d$ of the form  $d(x,x') = \tilde{d}(x,x') + \rad$, where $\tilde{d}$ is the dissimilarity function capturing the  dissimilarity between observed contexts and $\rad$ is the amplitude of dissimilarity driven by unobserved factors. This special case allows to explicitly account for two types of deterioration in data relevance: unobserved factors, and dissimilarity of observed contexts.\end{remark}

\noindent
\textbf{Focal class of policies: Weighted ERM.} In this work, we anchor our analysis and discussion around  the following central class of policies which minimize a weighted empirical loss.\footnote{We note that our framework and some of our results extend to a superset of these  policies; see \Cref{sec:OS}.}

\begin{definition}[Weighted Empirical Risk Minimization for Newsvendor]
\label{def:WERM}
A Weighted Empirical Risk Minimization (Weighted ERM) policy is defined by a sequence of non-negative weights $\mathbf{w} = (w_i)_{i \in \{1,\ldots,\numS\}}$ which could depend on the fixed contexts $\mathbf{x}$ and prescribes the action,
\begin{equation*}
\pi^{\mathbf{w}}(\mathbf{x},\mathbf{y}) = \inf \left \{ a \text{ s.t. }  \frac{\sum_{i=1}^\numS w_i (\bm{x}) \cdot \mathbbm{1} \left \{ y_i \leq a \right\}}{\sum_{j=1}^\numS w_j(\bm{x})}  \geq \frac{\cu}{\cu+\co}  \right \}. 
\end{equation*}
\end{definition}
Given a sequence of non-negative weights $\mathbf{w} = (w_i)_{i \in \{1,\ldots,\numS\}}$, the action prescribed by Weighted ERM is a minimizer of the weighted empirical loss  $\sum_{i=1}^\numS w_i \cdot \ell(\actVr,y_i)$ (see Proposition 1 in \citet{ban2019big}).  
This class of policies encompasses many classical policies used in Machine Learning and data-driven decision-making (see, e.g.,  \cite{bertsimas2020predictive}). Notable special cases of Weighted ERM policies are the classical ERM policy (when all the weights are equal) which we denote by $\pi^{\mathrm{ERM}}$ and the class of k-Nearest Neighbors (when the weight is $1$ for the $k$-closest contexts and $0$ otherwise). When setting $w_i = K( d(x_i,x_0))$ for a kernel function $K: \mathbb{R} \to \mathbb{R}_{+}$, we retrieve policies minimizing the Nadaraya-Watson estimator of the expected loss \citep{nadaraya1964estimating,watson1964smooth}. 
Intuitively, the weights can be selected to leverage the closeness of contexts (and associated  closeness of distributions imposed by the local condition).

\section{Main Results: Policies and Performance Characterization} \label{sec:main_result}

In this section, we present our main theoretical results, which characterize the performance of Weighted ERM policies, as well as more general policies.  In \Cref{sec:concentration}, we first review a classical concentration-based approach for bounding the performance of policies. In \Cref{sec:space_reductions}, we present the alternative optimization approach we pursue in this paper. In turn, we analyze  ``separable policies'', whose properties allow to develop a sequence of reductions that enables to directly characterize their performance. In \Cref{sec:OS}, we then establish that Weighted ERM policies are separable policies.

\subsection{Limitation of concentration-based arguments}
\label{sec:concentration}
We first provide a proof sketch of a very common concentration-based argument present in the literature, when bounding the performance of a policy such as Weighted ERM. 
Given a family of context vectors $\mathbf{x} \in \X^{\numS+1}$ and previously observed outcomes $\mathbf{y}\in \Y^\numS $, we denote by $\hat{F}_{\mathbf{w}}$ the weighted empirical distribution defined  as
$\hat{F}_{\mathbf{w}}(z) = \sum_{i=1}^\numS w_i \cdot \mathbbm{1} \left \{ y_i \leq z \right\}/\sum_{j=1}^\numS w_j $ for every $z \in \Y$, on which the decision of the weighted ERM policy is based. Then the regret can be decomposed as follows,
\begin{align*}
R(\mathbf{\pi}^{\mathbf{w}}(\mathbf{x},\mathbf{y}), F_{0}) &= L(\mathbf{\pi}^{\mathbf{w}}(\mathbf{x},\mathbf{y}), F_{0}) - L(a^*_{F_{0}}, F_{0})\\
&= \left[ L(\mathbf{\pi}^{\mathbf{w}}(\mathbf{x},\mathbf{y}), F_{0}) - L(\mathbf{\pi}^{\mathbf{w}}(\mathbf{x},\mathbf{y}), \hat{F}_{\mathbf{w}}) \right]\\
&\quad + \left[ L(\mathbf{\pi}^{\mathbf{w}}(\mathbf{x},\mathbf{y}), \hat{F}_{\mathbf{w}}) - L(a^*_{F_{0}}, \hat{F}_{\mathbf{w}}) \right] + \left[ L(a^*_{F_{0}},\hat{F}_{\mathbf{w}})  - L(a^*_{F_{0}}, F_{0}) \right]\\
&\stackrel{(a)}{\leq} 2 \cdot \sup_{ a \in \Y} \left | L(a, F_{0}) - L(a, \hat{F}_{\mathbf{w}}) \right |,
\end{align*}
where inequality $(a)$ holds because by construction of the Weighted ERM policy the second difference is negative. Given this decomposition, most arguments aim at deriving uniform bounds (across the action space) on the difference between the loss evaluated on the distributions $\hat{F}_{\mathbf{w}}$ and $F_{0}$.
This difference is controlled by introducing the weighted mixture of historical distributions $\overline{F}_{\mathbf{w}} = \sum_{i=1}^\numS {w_i} \cdot F_{i}/\sum_{j=1}^\numS w_j$ , and using the triangular inequality to obtain that,
\begin{equation}
\label{eq:uniform_concentration}
\sup_{ a \in \Y} \left | L(a, F_{0}) - L(a, \hat{F}_{\mathbf{w}}) \right | \leq \sup_{ a \in \Y} \left | L(a, F_{0}) - L(a, \overline{F}_{\mathbf{w}}) \right | + \sup_{ a \in \Y} \left | L(a, \overline{F}_{\mathbf{w}}) - L(a, \hat{F}_{\mathbf{w}}) \right|.
\end{equation}
The first term is then interpreted as the limiting regret with infinitely many samples and the second term captures the loss incurred because the empirical distribution is constructed with finite samples.
The first term does not depend on the realizations of $y_1,\ldots,y_n$ and is generally non-zero because the weighted mixture $\overline{F}_{\mathbf{w}}$ does not faithfully capture the unknown distribution $F_{0}$; however note that it is zero in the i.i.d.\ setting.  This is the setting of classical statistical learning, for which we refer to \citet[Chap.3 and Chap.4]{mohri2018foundations} and \citet[Chap. 26]{shalev2014understanding}.
The second term in~\eqref{eq:uniform_concentration} is bounded by showing that the empirical loss concentrates around the loss under the weighted mixture. Common bounds on the second term are decreasing and converge to $0$ as the number of samples grows. We compare our result to bounds that use~\eqref{eq:uniform_concentration} in \Cref{sec:mmBound}.

This concentration-based analysis is powerful as it can be applied to a general class of problems. However, we note that this line of argument does not account for the actual loss function at hand.  Also,  the bound on $R(\mathbf{\pi}^{\mathbf{w}}(\mathbf{x},\mathbf{y}), F_{0})$ one obtains  is based on a  uniform bound across actions and does not take into consideration the particular action taken by the policy of interest. Finally, as explained in~\eqref{eq:uniform_concentration}, this line of argument would separately analyze the limiting regret term and the finite-sample loss term, instead of directly analyzing the finite sample regret. A natural question is then how much is lost based on such a bounding approach.  We will see in \Cref{sec:shape} that both the scale and the shape suggested by these bounds are actually an artifact of the analysis rather than reflecting the actual performance of the algorithm.

\subsection{Learning without concentration: an optimization approach}
\label{sec:space_reductions}
In contrast to the approach outlined in \Cref{sec:concentration} for upper-bounding the value of 
\begin{align} \label{eqn:7891}
\mathbb{E}_{y_1\sim F_{1},\ldots,y_n\sim F_{\numS}} \big[R(\ddpol,F_{0}) \big],
\end{align}
we follow in this work an optimization approach to exactly evaluate the supremum value of~\eqref{eqn:7891} subject to constraints.  Recall that 
for a given policy $\pi$, the object of interest $\wreg[\pi]{\mathbf{x}}$ defined in \eqref{eq:fixed_design} can be phrased as the following optimization problem:
\begin{equation}
\label{eq:integral_form}
    \sup_{F_{0} \in \Delta \left( \Y \right)}  \sup_{ \substack{F_{1},\ldots, F_{\numS} \in \Delta\left( \Y \right)\\
\|F_{0} - F_{i}\|_K \leq d(x_0,x_i)\, \forall i } } 
\int_{\Y^\numS}
\left( L(\ddpol,F_{0}) -  L(a^*_{F_0},F_{0}) \right)
dF_{1}(y_1)\ldots dF_{\numS}(y_\numS).
\end{equation}
We note that \eqref{eq:integral_form} is a distributionally robust optimization (DRO) problem: an optimization problem where decision variables are distributions. Common approaches in the DRO literature analyse policies which use the samples to construct an uncertainty ball in which the out-of-sample distribution may lie and then solve a minimax problem over that uncertainty ball. Furthermore, the objective function in these minimax problems is usually linear in the distribution (for instance in the case of the expected loss $L(\cdot,\cdot)$). In contrast, the problem \eqref{eq:integral_form} involves an objective which has an intricate dependence in all the historical distributions: they affect the decision learned by the policy, and hence appear as integrating measures and result in a non-concave objective. Furthermore, it is a multivariate infinite-dimensional problem which involves several distributions constrained through the local conditions as opposed to a single out-of-sample distribution.

While this is a non-concave constrained infinite dimensional optimization problem,  we develop a sequence of reductions that allow to solve this problem. Our methodology is conceptually related to the one recently derived by \cite{besbes2021big} in which the authors analyze the worst-case relative regret in the i.i.d. setting.  At a high level their proof relies on two critical arguments. The first step to simplify \eqref{eq:integral_form} consists in showing that, in the newsvendor problem, for certain data-driven policies the objective may be transformed in the following form:
\begin{equation}
\label{eq:cdf_integral}
  \int_{\Y^\numS}
\left( L(\ddpol,F_{0}) -  L(a^*_{F_{0}},F_{0}) \right)
dF_{1}(y_1)\ldots dF_{\numS}(y_\numS) =  \int_{\Y} \Psi^\pi(F_{0}(y),\ldots,F_{\numS}(y))dy,
\end{equation}
where $\Psi^\pi$ is a continuous function from $[0,1]^{\numS+1}$ to $\mathbb{R}$. The second step uses the form of the objective in \eqref{eq:cdf_integral} to show that the initial problem can be solved via pointwise optimization.

Given the contextual nature of our problem in the present paper, we depart from the i.i.d. setting, and we aim at understanding the performance of a much broader set of policies including Weighted ERM ones. We note that \cite{besbes2021big} analyzed, in the i.i.d. setting, policies which prescribe a decision equal to an order statistic of the past samples. A natural question is whether their approach can be: i.) generalized to the contextual case; and ii.) directly applied to Weighted ERM policies. The answer to ii.) is negative as we show next:
\begin{lemma} \label{lem:WERM-OS}
There exists a Weighted ERM policy which is not an order statistic policy (formally introduced in \Cref{def:OS}).
\end{lemma}

Therefore, we need to extend the approach to a broader set of policies, while accounting for the contextual aspect of the problem.
To that end, we will start from a general class of policies, separable policies, that satisfy \eqref{eq:cdf_integral} and show through a sequence of reductions that \eqref{eq:integral_form} can be simplified under this class. We then develop in \Cref{sec:OS} a series of results to understand how broad these policies are, proving that they encompass Weighted ERM.
\begin{definition}[Separable policies]
\label{def:separable_policies}
We say that a data-driven policy $\pi$ is a separable policy if there exists a function $P^{\pi} : [0,1]^{\numS} \to [0,1]$, which could depend on the fixed context $\mathbf{x}$, such that for every distributions  $F_1,\ldots,F_{\numS} \in \Delta(\Y)$ and  any $y \in \mathbb{R}$, we have that
\begin{equation*}
\mathbb{P}_{\mathbf{y} \sim F_{1} \times \ldots \times F_{\numS}, a \sim \pi(\bm{x},\bm{y})} \left( a \leq y \right) = P^{\pi}(F_{1}(y), \ldots, F_{\numS}(y)).  
\end{equation*}
We refer to $P^{\pi}$ as the function associated to the separable policy $\pi$. 
Furthermore, we say that the policy $\pi$ is a non-decreasing separable policy if for every $i 
\in \{1,\ldots,\numS\}$ and every $h_1,\ldots h_{i-1},h_{i+1}, \ldots, h_{\numS} \in [0,1]$, the function
\begin{equation*}
h \mapsto P^{\pi}(h_1,\ldots h_{i-1},h,h_{i+1}, \ldots, h_{\numS})
\end{equation*}
is non-decreasing.
\end{definition}

Separable policies are ones for which the cumulative distribution of the decision at a given point $y \in \Y$, which is induced by the distribution of the samples used by the policy, can be expressed as a function $P^\pi$ of the cumulative distributions of the samples evaluated at the same point $y$. 
Among all separable policies a natural subclass to consider is the class of non-decreasing ones which are intuitively defined as policies for which the probability of selecting a decision less than or equal to $y$ increases when the probability of observing a sample less than or equal to $y$ increases.
In \Cref{sec:OS} we prove that several central policies belong to the class of separable policies. Furthermore, to provide some intuition about which policies are not separable, we prove in \Cref{prop:mean_not_separable} (see \Cref{sec:apx_A}) that the policy which selects the average of all samples observed is not separable.

In the remainder of this section we analyze the worst-case regret $\wreg[\pi]{\mathbf{x}}$, defined in \eqref{eq:fixed_design}, of separable policies.
Our next lemma formalizes the fact that for any separable policy, the objective function in Problem \eqref{eq:integral_form} can be simplified.
\begin{lemma}
\label{lem:cdf_integral_reduction}
For every $\numS \geq 1$, any sequence of contexts $\mathbf{x} = (x_i)_{i \in \{0,\ldots,\numS\}} \in \X^{\numS+1}$, any separable policy $\pi$ and every family of  distributions $(F_{i})_{i \in \{0,\ldots,\numS\}} \in \Delta(\Y)$, we have that
    \begin{equation*}
\mathbb{E}_{ \mathbf{y} \sim F_{1} \times \ldots \times F_{\numS} } \big[R(\pi(\mathbf{x},\mathbf{y}),F_{0}) \big]= (\cu+\co) \cdot \int_0^{1} \Psi^\pi(F_{0}(y), F_{1}(y), \ldots, F_{\numS}(y)) dy,
\end{equation*}
where $\Psi^\pi$ is a mapping from $[0,1]^{\numS+1}$ to $\mathbb{R}$ which satisfies, for every $(z_0,\ldots,z_{\numS}) \in [0,1]^{\numS+1}$, 
\begin{equation*}
\Psi^\pi(z_0,\ldots,z_{\numS}) = P^{\pi} (z_1,\ldots,z_\numS)  \cdot(q-z_0)  + \max \{ z_0 -q,0 \}.
\end{equation*}
\end{lemma}

\Cref{lem:cdf_integral_reduction} shows that  \eqref{eq:cdf_integral} holds for any separable policy. We next leverage this result to show that the multivariate infinite dimensional optimization problem \eqref{eq:integral_form} can be reduced to a simpler one over the space of Bernoulli distributions. Recalling that $\mathcal{B}(\cdot)$ denotes the Bernoulli distribution, the following holds.
\begin{proposition}
\label{prop:reduction_bern}
For every $\numS \geq 1$, any sequence of contexts $\mathbf{x} = (x_i)_{i \in \{0,\ldots,\numS\}} \in \X^{\numS+1}$  and any separable policy $\pi$ we have,
\begin{equation*}
\wreg[\pi]{\mathbf{x}}  = \sup_{\mu_0 \in [0,1]} \sup_{\substack{\mu_1,\ldots,\mu_\numS \in [0,1]\\ | \mu_i - \mu_0 | \leq d(x_0,x_i) \, \forall i}} \mathbb{E}_{ \mathbf{y} \sim \mathcal{B}(\mu_1) \times \ldots \times \mathcal{B}(\mu_{\numS})} \big[R(\pi(\mathbf{x},\mathbf{y}),\mathcal{B}(\mu_{0})) \big].
\end{equation*}
\end{proposition}
Proposition \ref{prop:reduction_bern} shows that for any separable policy, the worst-case distributions, both historical and out-of-sample,  over all possible distributions supported on $[0,1]$ are Bernoulli distributions. Therefore, the optimization problem \eqref{eq:fixed_design}, which involves the worst-case over all distributions, can be reduced to a significantly simpler one, in which the adversary only needs to decide on the means of different Bernoulli distributions (for the historical and out-of-sample distributions of outcomes).

Despite the significant reduction in complexity above, we note that optimizing over the space of Bernoulli means is still a non-convex optimization problem with dimension $(\numS+1)$. 
Our next result enables to further simplify the problem by characterizing the worst-case sequence of historical means $(\mu_i)_{i \geq 1}$ as a function of the out-of-sample mean $\mu_0$. 

\begin{proposition}
\label{prop:worst_history}
For every $\numS \geq 1$, any sequence of contexts $\mathbf{x} = (x_i)_{i \in \{0,\ldots,\numS\}} \in \X^{\numS+1}$  and any non-decreasing separable policy $\pi$ we have,
\begin{align*}
\sup_{\substack{\mu_1,\ldots,\mu_\numS \in [0,1]\\ | \mu_i - \mu_0 | \leq d(x_0,x_i) \, \forall i}} &\mathbb{E}_{ \mathbf{y} \sim \mathcal{B}(\mu_1) \times \ldots \times \mathcal{B}(\mu_{\numS})} \big[R(\pi(\mathbf{x},\mathbf{y}),\mathcal{B}(\mu_{0})) \big]\\
&=
\begin{cases}
 \mathbb{E}_{\mathbf{y} \sim \mathcal{B}(\mu_0 + d(x_0,x_1)) \times \ldots \times \mathcal{B}\left( \mu_0 + d(x_0,x_\numS) \right)} \left[  R(\pi(\mathbf{x},\mathbf{y}),\mathcal{B}(\mu_0)) \right] \quad \text{if $\mu_0 \in [0,1-q]$,}\\
 ~\\
 \mathbb{E}_{\mathbf{y} \sim \mathcal{B}(\mu_0 - d(x_0,x_1)) \times \ldots \times \mathcal{B}\left( \mu_0 - d(x_0,x_\numS) \right)} \left[ R(\pi(\mathbf{x},\mathbf{y}),\mathcal{B}(\mu_0)) \right] \quad \text{if $\mu_0 \in [1-q,1]$.}
\end{cases}
\end{align*}
\end{proposition}
Proposition \ref{prop:worst_history} establishes that for a given out-of-sample Bernoulli mean $\mu_0$, the worst-case sequence of conditional historical distributions can be explicitly characterized---they are the furthest possible (while still satisfying the local constraint) from the out-of-sample distribution. However the direction in which they differ depends on the economics of the problem and on the out-of-sample mean $\mu_0$.  When the mean of the out-of-sample Bernoulli distribution is below $1-q$, i.e., when it is optimal to set an inventory level of $0$, the worst-case sequence of distributions tends to inflate the mean as much as possible, to values of $\mu_0+d(x_0,x_1),\ldots,\mu_0+d(x_0,x_\numS)$, in order to push the decision-maker to carry more inventory than $0$. A similar interpretation holds when the mean of the out-of-sample Bernoulli distribution is above $1-q$. The two expressions are identical if $\mu_0=1-q$.

In the proof of Proposition \ref{prop:worst_history}, we actually derive a stronger statement as we show a monotonicity property of the regret of non-decreasing separable policies as a function of the Bernoulli means. To prove this property, we need the separable policy to be non-decreasing.

By combining Proposition \ref{prop:reduction_bern} and Proposition \ref{prop:worst_history}, we obtain
a characterization of the worst-case regret for any non-decreasing separable policy. We present this as the following \namecref{thm:regret_MOS},
which will later be used to derive our main result about Weighted ERM policies.
\begin{theorem}
\label{thm:regret_MOS}
For every $\numS \geq 1$, any sequence of contexts $\mathbf{x} = (x_i)_{i \in \{0,\ldots,\numS\}} \in \X^{\numS+1}$  and any non-decreasing separable policy $\pi$ we have,
\begin{align*}
\wreg[\pi]{\mathbf{x}}  &= \max \Big\{ \sup_{\mu_0 \in [0,1-q]} \mathbb{E}_{\mathbf{y} \sim \mathcal{B}(\mu_0 + d(x_0,x_1)) \times \ldots \times \mathcal{B}\left( \mu_0 + d(x_0,x_\numS) \right)} \left[  R \left(\pi(\mathbf{x},\mathbf{y}),\mathcal{B}(\mu_0) \right) \right],\\
&\quad \sup_{\mu_0 \in [1-q,1]} \mathbb{E}_{\mathbf{y} \sim \mathcal{B}(\mu_0 - d(x_0,x_1)) \times \ldots \times \mathcal{B}\left( \mu_0 - d(x_0,x_\numS) \right)} \left[  R \left(\pi(\mathbf{x},\mathbf{y}),\mathcal{B}(\mu_0) \right) \right] \Big \}.
\end{align*}
\end{theorem}
\Cref{thm:regret_MOS} directly follows from \Cref{prop:reduction_bern,prop:worst_history} and has two notable implications. First, it
shows that computing the worst-case regret of a non-decreasing separable policy can be done efficiently as the non-convex infinite dimensional optimization problem \eqref{eq:fixed_design} can actually be reduced to a one-dimensional optimization problem on a line segment $[0,1]$.  Furthermore, it establishes that Bernoulli distributions form a family of hard distributions for separable policies and characterizes the worst sequence of historical distributions as a function of the out-of-sample distribution.

The next corollary, relates the value of the optimization problem with multiple samples as defined in \eqref{eq:multiple_samples} and the more conservative one in \eqref{eq:fixed_design}.
\begin{corollary}
\label{cor:multiple_sample_to_one}
    If the dissimilarity $d$ satisfies the triangular inequality, then for every $\numS \geq 1$, any sequence of contexts $\mathbf{x} = (x_i)_{i \in \{0,\ldots,\numS\}} \in \X^{\numS+1}$  and any non-decreasing separable policy $\pi$, the value of the problem in \eqref{eq:multiple_samples} is equal to that of \eqref{eq:fixed_design}.  
\end{corollary}

\Cref{cor:multiple_sample_to_one} shows that when the dissimilarity function satisfies the triangular inequality, the two optimization formulations in \eqref{eq:multiple_samples} and \eqref{eq:fixed_design} are in fact equivalent. This equivalence can be understood by examining the worst-case distributions characterized in \Cref{thm:regret_MOS}. For any $i,j \in \{1,\ldots,\numS\}$, these distributions yield worst-case means that satisfy  
$|\mu_i - \mu_j| = |d(x_0,x_i) - d(x_0,x_j)|$,
which by the triangular inequality implies $|\mu_i - \mu_j| \leq d(x_i,x_j)$. Hence, the candidate Bernoulli distributions remain feasible even under the stricter constraints of \eqref{eq:multiple_samples}. As a result, the worst-case regret in both formulations coincides. This observation highlights that, under mild structural assumptions on $d$, the seemingly more conservative formulation \eqref{eq:fixed_design} yields the same worst-case regret, and our result extends to the more benign setting where samples observed in the same context must be i.i.d.

We note that \Cref{thm:regret_MOS} relies on the assumption that the local condition is defined using the Kolmogorov distance. However, our optimization-based proof technique can also yield tighter analyses for other distances. For instance, although we do not characterize the exact performance of Weighted ERM policies under the Wasserstein distance, we show in \Cref{thm:upper_bound_bern_Wasserstein} that, for any configuration of contexts, a Lagrangian relaxation combined with the proof technique of \Cref{prop:reduction_bern} reduces the original non-convex infinite-dimensional problem to a minimax problem involving only $2\numS+1$ variables. Building on the ideas of \Cref{prop:worst_history}, we further simplify this formulation in specific context configurations and for the ERM policy, leading to significantly improved guarantees compared to concentration-based bounds.

\begin{remark}[Randomized policies]
\label{rem:random}
It is worth noting that our main theorem also holds for any possible randomization over non-decreasing separable policies as this class of policies is closed under mixtures as formalized in \Cref{lem:mixture_closed}. This extension to randomized policies will be valuable in various settings; we return to this in \Cref{sec:shape}.
\end{remark}

\subsection{Relations between classes of policies}
\label{sec:OS}
\Cref{thm:regret_MOS} applies to any non-decreasing separable policy, but at this stage we have not yet showed that this abstract class contains any policies of interest.  In particular, it is a priori non-obvious whether Weighted ERM policies are separable and non-decreasing.  To show this, we first introduce the intermediate class of counting policies (\Cref{def:counting_policies}), which can be shown to be separable and non-decreasing and provide a definition that is easier to work with.
We then show that Weighted ERM policies (\Cref{def:WERM}) and order statistic policies (\Cref{def:OS}) are special cases of counting policies.
For reference, \Cref{fig:classes} illustrates the relationships between all the classes of policies we analyze in the present paper. 

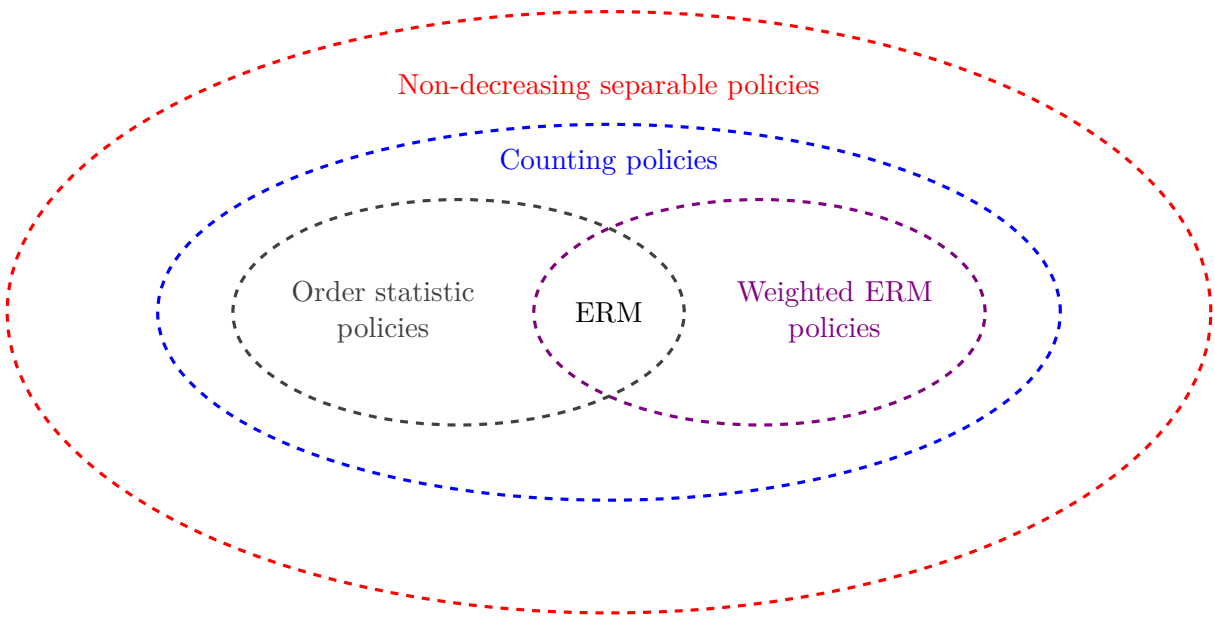
\begin{figure}[h!]
    \centering
\begin{tikzpicture}
    % Set C
    \draw[red, very thick, dashed] (0,0) ellipse (8cm and 4cm);
    \node[red] at (0,3) {Non-decreasing separable policies};

    % Set B
    \draw[blue,very thick,dashed] (0,0) circle (6cm and 2.5cm);
    \node[blue] at (0,2) {Counting policies};

    % Set A
    \draw[violet, very thick, dashed] (2,0) ellipse (3cm and 1.5cm);
    \node[violet,align=center] at (3,0) {Weighted ERM \\ policies};

    % Set D
    \draw[darkgray, very thick, dashed] (-2,0) ellipse (3cm and 1.5cm);
    \node[darkgray,align=center] at (-3,0) {Order statistic \\ policies};

    \node at (0,0) {ERM};
\end{tikzpicture}
    \caption{\textbf{Classes of policies analyzed.} The figure represents the different classes of policies that we analyze and the relationships that we show in terms of containment.}
    \label{fig:classes}
\end{figure}

\begin{definition}[Counting policies]
\label{def:counting_policies}
We say that a deterministic policy $\pi$ is a counting policy if there exists a function $\kappa^{\pi} : \{0,1\}^{\numS} \to \{0,1\}$ which could depend on the fixed contexts $\mathbf{x}$, such that, for every $y \in [0,1]$, and any $\mathbf{y} \in \Y^\numS$,
\begin{equation*}
\mathbbm{1} \{ \pi(\mathbf{x}, \mathbf{y}) \leq y \} = \kappa^{\pi} \left( \mathbbm{1} \{ y_1 \leq y \},\ldots,\mathbbm{1} \{ y_{\numS} \leq y  \} \right).
\end{equation*}
We refer to $\kappa^{\pi}$ as the counting function associated to $\pi$. 
\end{definition}
Counting policies are intuitively defined as ones for which the decision of whether the inventory should be lower than any given value $y \in \Y$ is only a function of the sequence of variables $(\mathbbm{1} \{ y_i \leq y \})_{i \in \{1,\ldots, \numS\}}$, which indicate the past demand samples that have a value lower than $y$. 
%\revision{Moreover, we note that the function $\kappa^{\pi}$ does not depend on $y$ fixed across all values of $y \in [0,1]$.}
Importantly, proving that a policy is a counting policy is simpler than showing that it is a separable policy because counting policies are defined through the decision prescribed by the policy as a function of the \textit{realization} of past samples, whereas the definition of separable policy involves the distribution of decisions implied by the policy. Furthermore, showing that a policy is a counting policy suffices to establish that it is a non-decreasing separable policy, as formalized by the next result.
\begin{proposition}\label{prop:counting_to_sep}
Every counting policy is a non-decreasing separable policy. 
\end{proposition}
The fact that the function $\kappa^\pi$ is identical for every $y$ allows us to show that counting policies must be non-decreasing.  From there, we prove \Cref{prop:counting_to_sep} by showing that non-decreasing counting policies induce non-decreasing separable policies.

Counting policies enable us to analyze Weighted ERM policies as we show that under the Newsvendor loss, the set of Weighted ERM policies is included in the set of counting policies.

\begin{proposition}
\label{prop:WEM_to_counting}
For any sequence of non-negative weights $\mathbf{w} = (w_i)_{i \in \{1,\ldots,\numS\}}$, the associated Weighted ERM policy is a counting policy. Furthermore its counting function $\kappa^{\pi}$ is defined for every $\mathbf{b} \in \{0,1\}^{\numS}$ as
\begin{equation*}
\kappa^{\pi}(\mathbf{b}) =  \begin{cases}
1 \quad \text{ if $\frac{\sum_{i=1}^\numS w_i \cdot b_{i} }{\sum_{i=1}^{\numS} w_{i}} \geq \frac{\cu}{\cu+\co}$,}\\
0 \quad \text{otherwise}.
\end{cases}
\end{equation*}
\end{proposition}

Finally, by combining \Cref{prop:counting_to_sep}, \Cref{prop:WEM_to_counting}, and \Cref{thm:regret_MOS}, we obtain the following \namecref{cor:regret_WERM}, which constitutes our main result.  It allows us to evaluate the worst-case regret of any Weighted ERM policy. 

\begin{theorem}
\label{cor:regret_WERM}
For every $\numS \geq 1$, any sequence of contexts $\mathbf{x} = (x_i)_{i \in \{0,\ldots,\numS\}} \in \X^{\numS+1}$  and any Weighted ERM policy $\pi$ we have,
\begin{align*}
\wreg[\pi]{\mathbf{x}}  &= \max \Big\{ \sup_{\mu_0 \in [0,1-q]} \mathbb{E}_{\mathbf{y} \sim \mathcal{B}(\mu_0 + d(x_0,x_1)) \times \ldots \times \mathcal{B}\left( \mu_0 + d(x_0,x_\numS) \right)} \left[  R \left(\pi(\mathbf{x},\mathbf{y}),\mathcal{B}(\mu_0) \right) \right],\\
&\quad \sup_{\mu_0 \in [1-q,1]} \mathbb{E}_{\mathbf{y} \sim \mathcal{B}(\mu_0 - d(x_0,x_1)) \times \ldots \times \mathcal{B}\left( \mu_0 - d(x_0,x_\numS) \right)} \left[  R \left(\pi(\mathbf{x},\mathbf{y}),\mathcal{B}(\mu_0) \right) \right] \Big \}.
\end{align*}
\end{theorem}

We note that our definition of Weighted ERM policies selects the smallest action which minimizes the empirical loss, aligning with the standard definition of Weighted ERM in the Newsvendor setting. However, \Cref{cor:regret_WERM} remains valid even for policies that employ more complex tie-breaking rules when multiple actions minimize the empirical loss. We present this more general result in \Cref{sec:tie-break}.

\section{New Insights on the Learning Behavior of Algorithms}

\label{sec:insights}

In this section, we leverage our exact analysis of data-driven policies to derive new insights on their performance as a function of  quantity and relevance of data.
While the previous section is general, we now explore different prototypical special cases of context configurations. In \Cref{sec:mmBound} and \Cref{sec:shape}, we focus on the implications of our result for a special case in which all past dissimilarities are identical, i.e., $d(x_1,x_0) = \cdots =d(x_n,x_0)=\zeta$ for a fixed $\zeta \geq 0$. 
In such a case the data points can be treated symmetrically,  and it is natural to study the (unweighted) ERM policy.
In \Cref{sec:many_dissimilarities} we evaluate the performance of various policies when the dissimilarity deteriorates over time.
In \Cref{sec:misspecification} we investigate the impact of misspecified dissimilarities.

\subsection{Exact sample complexity of ERM and achievable performance} \label{sec:mmBound}
In this \namecref{sec:mmBound}, our goal is to understand the worst-case regret that can be achieved by ERM given a certain configuration of contexts. This question is of a very different nature from ones that can be asked in the setting in which past samples are drawn i.i.d. from the out-of-sample distribution. Indeed, when past contexts are different from the out-of-sample one, past samples are not fully indicative of the out-of-sample distribution $F_{0}$ and therefore one may not necessarily achieve a vanishing regret even with arbitrarily large sample sizes.  A second question consists in understanding the number of samples required to obtain a regret guarantee lower than a given target whenever such a target can be achieved.

We will explore both questions while contrasting the answers implied by previous state-of-the-art bounds to the ones we derive.
In particular, we compare our results with  \cite{mohri2012new} which, to the best of our knowledge, is the state-of-the-art bound which can be applied to our setting, as they derive a performance guarantee that holds when past outcomes are independently sampled from distributions which may differ from the out-of-sample distribution.
We show in \Cref{sec:apx_Mohri} that one can leverage their distribution-dependent bound to obtain the following guarantee on the regret of ERM, as a function of the given contexts $\mathbf{x} = (x_i)_{i \in \{0,\ldots,\numS\}}$.
\begin{equation}
\label{eq:expected_mohri_main}
4  C_{\numS} + \frac{2\max(q,1-q)}{\numS}  \sum_{i=1}^\numS d\left(x_i,x_0\right) + 4\max(q,1-q) \sqrt{\frac{2\pi}{\numS}} \cdot \left \{ \Phi \left(\frac{2\max(q,1-q)}{\sqrt{\numS}} \right) - \Phi(0) \right \}.
\end{equation}
In~\eqref{eq:expected_mohri_main}, $\Phi$ denotes the cumulative distribution function of the standard gaussian and
$C_\numS$ is a notion of Rademacher complexity evaluated on the worst-case set of distributions identified by \Cref{thm:regret_MOS}.

We present in \Cref{tab:detailed} the number of samples $n$ required to guarantee that ERM has a regret below some target, in the special setting where $d(x_1,x_0) = \cdots =d(x_n,x_0)=\zeta$ for a fixed $\zeta \geq 0$.
We fix $q$ to be 0.9.
As a point of reference, these regret targets are given as a percentage of the ``no-data regret'' achieved by a decision-maker who only knows the support of the demand distribution, which has been shown to equal $q\cdot(1-q)$ when this support is normalized to [0,1] \citep{perakis2008regret}.
\begin{table}[h!]
\centering
\begin{tabular}{llllllll}
& & \multicolumn{6}{c}{Regret target (as \% of no-data regret)}\\
\cline{3-8}
$\zeta$& & 100\% & 90\% & 75\% & 50\% & 25\% & 10\% \\ 
\hline
%0 & $N^{\mathrm{GP}}$    &  2,238 &  4,549  &8,860  &35,962& 100,000+  \\
0 & $N^{\mathrm{GP}}$    &  279 & 338 &  475  &1,032  &3,298& 24,673  \\
& $N^{\mathrm{exact}}$ & 3  & 3 & 4 & 5 & 14 & 37 \\
 \hline
%& $N^{\mathrm{GP}}$    &  6,251&  18,042  &  100,000+& $\infty$ & $\infty$ \\
0.02  & $N^{\mathrm{GP}}$    &  734 & 1,047 &  2,114  &24,461  &inf.& inf.  \\
& $N^{\mathrm{exact}}$ &  3& 3 &  4&  5& 16 & inf. \\
\hline
0.04  & $N^{\mathrm{GP}}$    &  6,228 & 24,493 &  inf.  & inf. &inf.& inf.  \\
& $N^{\mathrm{exact}}$ &  3& 4 &  4&  6& inf. & inf. \\
\end{tabular}
\caption{\textbf{Number of samples that ensures  a target  regret for  ERM.} The table reports the number of samples needed to guarantee a regret target when dissimilarities are $\zeta \in \{0,0.02,0.04\}$. $N^{\mathrm{GP}}$ (resp.\ $N^{\mathrm{exact}}$)  is the number of samples required to achieve a target regret implied by the general-purpose guarantee~\eqref{eq:expected_mohri_main} of \citet{mohri2012new} (resp.\ our exact quantification of $\wreg[\ERM]{\mathbf{x}}$ 
in \Cref{cor:regret_WERM}). We report ``inf.'' when the bound indicates that the target is unachievable even with infinitely many samples.}
\label{tab:detailed}
\end{table}

In \Cref{tab:detailed}, we first present results for the special case  in which past outcomes are drawn i.i.d. from the out-of-sample distribution ($\zeta  =0$) for reference. In this case, general-purpose bounds provide an overly pessimistic understanding of the number of samples required to achieve a given performance. Indeed, our characterization shows that the difference between the actual number of samples required and the one implied by previous bounds can be multiple orders of magnitude. We note that similar insights have been previously established by \cite{besbes2021big}, under a different performance metric of relative regret.

When $\zeta$ is positive,  we observe that the discrepancy between the actual number of samples required and the one implied by previous bounds is even more acute.  Indeed, previous approaches imply that, when historical distributions are different from the out-of-sample one, e.g., $\zeta = 0.02$, ERM requires $734$ samples (up from $279$ samples in the i.i.d. case) to match the no-data regret. This suggests that ERM is a policy which would require a  large number of samples to become efficient and that this sample complexity considerably deteriorates as relevance of data slightly decreases. However, our exact analysis shows that these bounds are overly conservative and do not capture the performance of ERM at all. In fact, ERM only requires $3$ samples to achieve this same performance\footnote{We remark here that ERM does not use knowledge of the support of demand, whereas the minimax no-data policy does. This explains why it takes a few samples to match its performance.}, and the sample complexity only increases to 14 if we are targeting 25\% instead of 100\% of the no-data regret.

Even more notably, when the past distributions differ from the out-of-sample one ($\zeta = 0.02$), we observe that certain levels of performance are unachievable by the ERM policy even with infinitely many samples.  This is a natural consequence of only accessing a past distribution that is different from the out-of-sample one on which performance will be evaluated, and even the best data-driven policy in this case has a non-vanishing regret. 
\Cref{tab:detailed} shows that previous state-of-the-art bounds provide an incorrect picture of the achievable performance. Indeed, such bounds imply that, when $q=0.9$, ERM cannot have a worst-case regret lower than $25\%$ of the no-data regret, even with an infinite number of samples. In stark contrast, our bound shows that this performance is actually achievable and can be obtained with as few as 16 samples! 
The analysis above highlights that, while general-purpose upper bounds are very powerful, specialized ones are needed to fully uncover the value of data and the impact of data relevance on performance.

\subsection{Shape of the learning curve of ERM and effective sample size}\label{sec:shape}

\subsubsection{Learning curve: unveiling new insights}
In this \namecref{sec:shape} we focus on the more granular question associated with the shape of how the worst-case regret evolves with the number of samples, which we call the ``learning curve''.
We consider the following representative example throughout: dissimilarities $d(x_i,x_0)$ are $\zeta=0.1$ for all $i$, and $q=0.9$.

First, we plot in \Cref{fig:ERM} the regret guarantee of ERM as a function of $n$.  Note that unlike  \Cref{sec:mmBound}, the regret is now expressed in absolute terms instead of as a percentage of the no-data regret.

\begin{figure}[h!]
\centering
\subfigure[Learning curve of ERM]{
\begin{tikzpicture}[scale=.74]
\begin{axis}[
            title={},
            xmin=0,xmax=200,
            ymin=0.0,ymax=1.5,
            width=10cm,
            height=8cm,
            table/col sep=comma,
            xlabel = number of samples $\numS$,
            ylabel = worst-case regret,
            grid=both,
            %skip coords between index={0}{1},
            legend pos=north east]
            %label style={font=\footnotesize, text left}
%\draw[blue,ultra thick, domain=0:600, smooth, variable=\x, red] plot ({\x}, {90});
\addplot [gray,thick,mark=square,mark options={scale=.4}] table[x={Ns},y={eps = 0.1}] {Data/mohri_bound_correct_eps_e-1_q9.csv};
    \addlegendentry{ERM (M \&MM12)}
 \addplot [blue,thick,mark=square,mark options={scale=.4}] table[x={Ns},y={epsilon=0.1}] {Data/paper_all_data_SAA_q9.csv};
    \addlegendentry{ERM (Our exact analysis)}
\end{axis}
\end{tikzpicture}
\label{fig:ERM}
}
\subfigure[Improving ERM]{
\begin{tikzpicture}[scale=.74]
\begin{axis}[
            title={},
            xmin=0,xmax=200,
            ymin=0.0,ymax=0.15,
            width=10cm,
            height=8cm,
            table/col sep=comma,
            xlabel = number of samples $\numS$,
            ylabel = worst-case regret,
            grid=both,
            %skip coords between index={0}{1},
            legend pos=north east]
            %label style={font=\footnotesize, text left}

\addplot [blue,thick,mark=square,mark options={scale=.2}] table[x={Ns},y={epsilon=0.1}] {Data/paper_all_data_SAA_q9.csv};
    \addlegendentry{ERM}
    \addplot [red,thick,mark=square,mark options={scale=.2}] table[x={Ns},y={tweak}] {Data/paper_all_data_mix_q9.csv};
    \addlegendentry{$\LERM$}
\draw[black,line width = 1mm, domain=0:200, smooth, variable=\x, black, dashed] plot ({\x}, {50});
   \addlegendentry{Lower bound (any policy)}
   \addlegendimage{ultra thick,dashed,black}
\end{axis}
\end{tikzpicture}
\label{fig:LERM_tilde}
}
\caption{\textbf{Performance of ERM and alternative policies.} (a) The figure depicts the worst-case regret of the ERM policy as implied by the bound presented by \cite{mohri2012new} and by ours. (b) The figure compares the performance of ERM and $\LERM$ to the lower bound achievable by any data-driven policy (see \Cref{rem:lb}) . In these plots, $\zeta = .1$ and $q=.9$.}
\label{fig:learning_curves}
\end{figure}
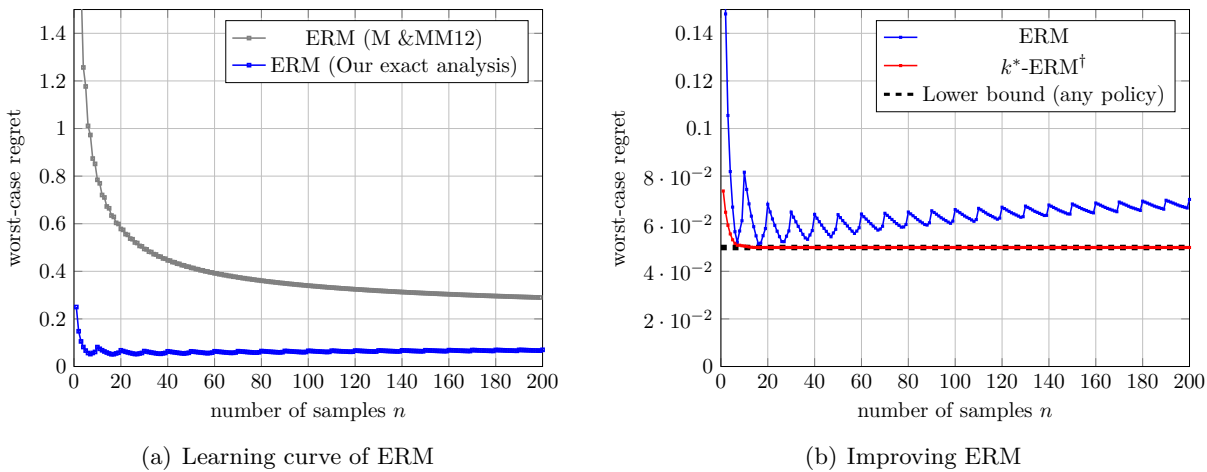

We observe in \Cref{fig:ERM} that the scale of the curve implied by previous bounds is very different from the actual worst-case performance of ERM which is obtained by our exact analysis. This observation is consistent with the findings from \Cref{sec:mmBound}. Moreover, this figure highlights a second shortcoming of previous bounds: they do not capture the correct shape of the learning curve of ERM. Indeed, the worst-case regret implied by \cite{mohri2012new} is decreasing as the number of samples grows and suggests that the performance of ERM improves as we aggregate more samples.
By contrast, our exact analysis unveils that the worst-case regret of ERM is non-monotonic as a function of the sample size.

In \Cref{fig:LERM_tilde},  we zoom in on the performance of ERM, displaying only the guarantee from our exact analysis, and observe two types of non-monotonicity.  First, ERM exhibits a ``local'' non-monotonicity, in that the learning curve is jagged with many local minima. Such type of behavior  was already observed in the i.i.d. setting (where $\zeta=0$)  by \citet{besbes2021big}, under a different metric of regret. \Cref{fig:LERM_tilde} reveals another fundamental phenomenon,  not present in the i.i.d. setting: ERM  exhibits a ``global'' non-monotonicity. The performance improves considerably with the first few samples,  and then deteriorates after adding more samples. In other words, we observe that the learning curve of ERM admits a unique global minimum  at a finite $n$ ($n=15$ in the example), and increasing $n$ to $\infty$ tends to move away from this global minimum!

 We also show in \Cref{sec:beyond_worst_case}  that most of the insights derived through the worst-case analysis are still widely applicable when the instance is fixed across sample sizes and the demand distribution is not a Bernoulli distribution.

\subsubsection{Achievable performance and effective sample size}

We have just observed various phenomena associated with the performance of ERM, which lead to a natural question: How should one improve ERM in settings when historical distributions are different from the out-of-sample one?

To provide some context on the types of performance achievable in such settings, we first state a universal lower bound for any policy, even with infinite samples.
\begin{remark}[Lower Bound]
\label{rem:lb}
For $\zeta \leq \min(q,1-q)$, the regret of any data-driven policy is larger or equal to $\zeta/2$ for any sample size.
\end{remark}
This result is implied by the proof of Proposition 5 in \cite{besbes2022beyond} who analyze general decision-problems with infinite data sizes. We investigate corrections to ERM and how close they can get to this lower bound.

The non-monotonicities in the shape of ERM's learning curve, and in particular the observed phenomenon of global non-monotonicity, lead us to propose the following alternate policy and subsequently define a notion of effective sample size.
The alternate policy, which we refer to as $\LERM$, modifies ERM in two ways. First,  we use a mixture of order statistics inspired by \cite{besbes2021big} to avoid the local non-monotonicity. This was sufficient to obtain a minimax policy in their setting but it does not correct the global non-monotonicity when $\zeta > 0$ (see \Cref{fig:SAA_vs_mix} in \Cref{sec:apx_improve_mix}). The second idea consists in using the optimal number of samples (denoted by $k^*$) even when having access to more.  We formally define the $\LERM$ policy in \Cref{sec:apx_improve_mix}. %\newob{Strnage to only define in Appendix?}

$\LERM$ can be evaluated by leveraging the full generality of our \Cref{thm:regret_MOS}, which applies to randomized counting policies (details in \Cref{rem:random}).
We plot its worst-case regret as derived by our exact analysis, also in \Cref{fig:LERM_tilde}, along with the lower bound on the regret of any data-driven policy.
We observe that $\LERM$ considerably improves the performance over ERM and resolves both types of non-monotonicities. Quite notably, we also note that  $\LERM$ is in fact near-optimal for the worst-case regret criterion (it achieves a worst-case regret within 1.001\%  of the lower bound) as soon as more than 15 samples are available. We therefore interpret $k^*$, the optimal choice of sample size for $\LERM$, as a notion of effective sample size which contains almost all the information necessary to perform well. After this point, using more samples  does not allow to improve performance by any meaningful amount, even if one were to use a different policy.

In \Cref{tab:effective_samples}, we investigate the effective sample size (sample size used by $\LERM$) as a function of  $\zeta$ for values of $\zeta$ ranging from $0.01$ to $0.1$. 
\begin{table}[h!]
\centering
\begin{tabular}{ccccccc}
 & \multicolumn{6}{c}{$\zeta$}  \\ 
 \cline{2-7}
 & 0.01& 0.02 & 0.03 & 0.04 & 0.05 & 0.1 \\
\hline
effective samples size & 589 &  330&  202&  95 & 58 & 15 \\
\end{tabular}
\caption{\textbf{Effective sample sizes for $\LERM$.} The table reports, for different values of dissimilarity, the effective sample size, i.e., the number of samples used by  $\LERM$, even when having access to more samples.}
\label{tab:effective_samples}
\end{table}

We note that for all the values reported in \Cref{tab:effective_samples}, the worst-case regret of $\LERM$ when using a sample size equal to the effective sample size is lower than $1.001\%$ of that of the best data-driven policy using infinitely many samples. 
Consequently, in the presence of data which deviates from the i.i.d. setting, restricting attention to the effective sample size while using $\LERM$ enables to match near-optimal performance across all data regime with very few samples.

These insights sharply contrast with the behavior suggested by previous monotonic bounds.  The performance of a learning policy should not be thought of as a decomposition between asymptotic error and finite sample error as described in \Cref{sec:concentration}; furthermore, one may achieve substantially better performance by using the effective number of samples rather than all samples.

\subsection{Application to varying dissimilarities: modeling time dependence}
\label{sec:many_dissimilarities}
We conclude this section with an illustration of the generality of our theory to analyze different instances of the contextual newsvendor. We next consider a setting which models a distribution drift over time. In this case, we set for every $i \in \{1,\ldots,\numS\}$ the dissimilarity $d(x_0,x_i) = i \cdot \Delta$, with $\Delta > 0$ representing the drift parameter. In this setting, one may think of $y_i$ as the demand observed $i$ weeks before the inventory decision and $\Delta$ is an upper-bound on the drift in demand distribution that may occur between two consecutive weeks.

We evaluate the worst-case regret for two subclasses of Weighted ERM policies with weights given by two widely used Nadaraya-Watson kernels.  The first one is the class of Exponential Weighted ERM policies (EWERM) which corresponds to the kernel $K_h(z) = \exp(z/h)$. Formally, it assigns weights of the form $w_i = \gamma^i$ for a fixed $\gamma \in (0,1]$\footnote{We note that when the dissimilarity satisfies $d(x_0,x_i) = i \cdot \Delta$, one can reparametrize the exponential kernel as follows: $K_{h}(d(x_0,x_i)) = \exp \left(  \frac{d(x_0,x_i)}{h} \right) = \exp \left(  \frac{\Delta}{h} \cdot i \right) = \gamma^i$.}. 
We denote by $\pi^{\gamma}$ such a policy. The second one is the class of $k$-NN-ERM policies which corresponds to the kernel $K_h(z) = \mathbbm{1} \left\{ z \leq h \right \}$. For the instance considered in this section, it equivalently assigns weights of the form
\begin{equation*}
    w_i =\begin{cases}
        1 \quad \text{if $i \leq k$}\\
        0 \quad \text{otherwise,}
    \end{cases}
\end{equation*}
for a fixed $k \in \{1,\ldots,\numS\}$. We denote by $\pi^{k}$ such a policy. We next evaluate the worst-case performance of these policies for various values of the drift parameter $\Delta$. Given a fixed drift, we denote by $\mathbf{x}_\Delta$ the sequence of contexts satisfying the dissimilarity relation defined above. 

For a fixed $\numS$ we solve the following two problems,
\begin{equation}
\label{eq:best_gamma}
    \min_{\gamma \in [0,1]} \wreg[\pi^\gamma]{\mathbf{x}_\Delta}
\end{equation}
and,
\begin{equation}
\label{eq:best_k}
    \inf_{k \in \{1,\ldots,\numS\}} \wreg[\pi^k]{\mathbf{x}_\Delta}.
\end{equation}
Furthermore we denote by $\gamma^*(\Delta)$ (resp. $k^*(\Delta)$) the value of $\gamma$ (resp. $k$) which achieves the minimum value of problem \eqref{eq:best_gamma} (resp. \eqref{eq:best_k}). The value of \eqref{eq:best_gamma} (resp. \eqref{eq:best_k}) is the minimal worst-case regret achievable by using the best parameter for each subclass of policies.

We report in \Cref{tab:temportal} the value of these two problems and the values of the parameters achieving the minimal worst-case regret.
\begin{table}[h!]
\centering
\begin{tabular}{ccc|cc}
  & \multicolumn{2}{c|}{EWERM} & \multicolumn{2}{c}{$k$-NN-ERM}  \\ 
 \cline{2-5}
 $\Delta$ & $\gamma^*(\Delta)$& worst-case regret & $k^*(\Delta)$ & worst-case regret \\
\hline
\hline
$0.0010$ & 0.95 & 0.016 &  27&  0.014  \\
$0.0025$ & 0.91 & 0.023  &  17&  0.018  \\
$0.0050$ & 0.88 & 0.031 &  8&  0.025  \\
\end{tabular}
\caption{\textbf{Best performance and parameters for EWERM and $k$-NN-ERM.} The table reports, the best choice of parameter and the achievable worst-case regret for the two classes of policies for different values of $\Delta$. ($\numS = 100$ and $q=0.9$)}
\label{tab:temportal}
\end{table}

We note in \Cref{tab:temportal} that when the relevance of data is heterogeneous, the number of samples used by each policy is actually even smaller than observed in \Cref{sec:shape}. For instance, the value of $k^*(\Delta)$ in \Cref{tab:temportal} can be interpreted as a notion of effective samples size comparable to the one in \Cref{tab:effective_samples} (even though the settings are not exactly the same). When $\Delta = 10^{-3}$, we observe that $k^*(\Delta) =27$, implying that the least relevant data used by the policy has a dissimilarity equal to $0.027$. We note that from \Cref{tab:effective_samples}, $\LERM$ would have used at least $200$ samples when all the dissimilarities are the same with a value lower than $0.03$. This difference in terms of effective sample size reveals that, in the presence of drifting distributions, the data points with low dissimilarity significantly decrease the value of data with higher dissimilarity. 

We also observe that the number of samples used by $k$-NN-ERM is consistent with the number of samples used by EWERM. For a given $\gamma \in (0,1)$, one may interpret $1/(1-\gamma)$ as a proxy for the effective sample size of the policy $\pi^\gamma$. Therefore, one may ask whether $1/(1-\gamma^*(\Delta))$ is comparable to $k^*(\Delta)$. We observe in \Cref{tab:temportal} that when $\Delta = 0.001$, $1/(1-\gamma^{*}(\Delta)) = 20$ which is relatively similar to $k^*(\Delta)$ as it is equal to $27$. This conclusion holds across all values of $\Delta$. 

\subsection{Sensitivity with respect to the dissimilarity}
\label{sec:misspecification}

In \Cref{sec:many_dissimilarities} we have considered a setting with time drift and investigated how a decision-maker should choose the hyper-parameter for widely used families of Weighted ERM policies (such as the $k$-NN-ERM policies) in order to minimize the worst-case regret. In what follows, we study the sensitivity of these results with respect to the dissimilarity.

We report in \Cref{tab:misspecification} the performance of the $k$-NN-ERM policy for various choices of $k$ and $\Delta$. 
\begin{table}[h!]
\centering

\begin{tabular}{cccc}
  & \multicolumn{3}{c}{worst-case regret}   \\ 
 \cline{2-4}
 $\Delta$ & $k = 8$ & $k = 17$ & $k = 27$  \\
\hline
\hline
$0.0010$ & 0.018 &  0.016 &  \textbf{0.014}   \\
\hline
$0.0025$ & 0.021 & \textbf{0.018}  &  0.020 \\
\hline
$0.0050$ & \textbf{0.025} &  0.026 & 0.036  \\
\end{tabular}
\caption{\textbf{Worst-case regret of $k$-NN-ERM.} The table reports the worst-case regret for different values of $k$ and $\Delta$. The bolded number corresponds to the best worst-case regret across all value of $k \in \{1,\ldots,100\}$ ($\numS = 100$ and $q=0.9$).}
\label{tab:misspecification}
\end{table}

We note that the performance of the $k$-NN-ERM which uses $k^*(\Delta)$ by solving the minimax problem \eqref{eq:best_k} is robust even when using an approximate value of $\Delta$. For instance, if the decision-maker believes that $\Delta = 0.0025$ instead of $\Delta = 0.005$ (which means that $\Delta$ is twice larger than what they assume) they would use the $k$-NN-ERM with a value of $k = 17$ as opposed to the recommended choice of $k = 8$. While the error on $\Delta$ is large, the worst-case regret incurred with $k=17$ is only $4\%$ higher than the one with $k = 8$ when $\Delta = 0.005$. Furthermore, the worst-case regret they expect to incur goes from $0.018$ to an actual worst-case regret of $0.026$ which represents a $44 \%$ increase and is starkly lower than the $100\%$ increase in $\Delta$.

We note that beyond this sensitivity analysis, our work also raises the broader question of how to estimate dissimilarities between different contexts. While a comprehensive treatment of this problem lies beyond the scope of the present paper, we provide and evaluate a simple estimation procedure in \Cref{sec:apx_estimation} for illustrative purposes.

\section{Conclusion}
In the present paper, we investigate the impact of relevance and quantity of data on performance for the prototypical contextual data-driven Newsvendor problem. 
We develop a new methodology to quantify exactly the worst-case regret of the broad class of Weighted ERM policies which encompasses several classical policies. Our method relies on an optimization approach we refer to as ``learning without concentration'' because it departs from the common concentration-based arguments previously developed in the literature. 

We leverage our exact analysis to derive insights on this class of problems and show that these insights contrast with those implied by state-of-the-art upper bounds. We show that, certain performances which were considered unachievable by previous bounds for ERM, even with infinitely many samples, can actually be achieved with very few samples. Furthermore, we show that, when historical distributions are different from the out-of-sample one, the worst-case regret of ERM reaches a global minimum for a finite sample size and then deteriorates when the sample size goes to $\infty$, whereas previous upper bounds on the worst-case regret which are decreasing implicitly prescribe to accumulate more data in order to improve performance. 

All in all, our analysis highlights the need to develop problem-specific bounds in order to capture the \textit{actual} shape and scale of the learning curve. Indeed, concentration-based analysis is powerful as it can be applied to a general class of problems. However, the shape and scale suggested by these bounds are an artifact of the analysis rather than the actual performance of the algorithm.

We note that our analysis is tailored to the Newsvendor loss, and it remains an open question to what extent such exact characterizations can be extended to other decision-making problems. Several exciting research avenues remain regarding the generalizability of our method. A first natural direction is to use our tractable characterization of the worst-case performance to design new algorithms, for example by optimizing over the weights of the Weighted ERM policy. A second direction is to extend the analysis beyond the Kolmogorov distance to alternative notions of distributional proximity; for instance, in \Cref{sec:apx_Wasserstein} we illustrate how a Lagrangian relaxation can yield bounds under a Wasserstein distance, and further work could establish “learning without concentration’’ guarantees more broadly. A third avenue is to study distributions with additional structure, such as moment or shape constraints, which may yield bounds even closer to the realized performance of algorithms. Finally, it would be valuable to explore the applicability of our approach to entirely different problem classes beyond Newsvendor, and to investigate whether (and how) analogous sharp worst-case characterizations can be obtained in these more general decision-making settings.

{\setstretch{1.0}
\bibliographystyle{agsm}
\bibliography{ref}}

\newpage

\appendix

\renewcommand{\theequation}{\thesection-\arabic{equation}}
\renewcommand{\theproposition}{\thesection-\arabic{proposition}}
\renewcommand{\thelemma}{\thesection-\arabic{lemma}}
\renewcommand{\thetheorem}{\thesection-\arabic{theorem}}
\renewcommand{\thedefinition}{\thesection-\arabic{definition}}
\pagenumbering{arabic}
\renewcommand{\thepage}{App-\arabic{page}}

\setcounter{equation}{0}
\setcounter{proposition}{0}
\setcounter{definition}{0}
\setcounter{lemma}{0}
\setcounter{theorem}{0}

For the sake of simple notations, we do not include in the appendix the dependence in the features when not necessary. For instance, the decision of a policy will be denoted as $\pi(\mathbf{y})$ (as opposed to $\pi(\mathbf{x},\mathbf{y})$. 

%The out-of-sample distribution will be $F$ instead of $F_{x_0}$ and the historical distributions will be for every $i \in \{1,\ldots,\numS\},$ referred to as $H_i$ instead of $F_{x_i}$.

\section{Proofs of Results Presented in \Cref{sec:space_reductions}}
\label{sec:apx_A}
\begin{proof}[\textbf{Proof of \Cref{lem:WERM-OS}}]
To prove this result we first derive properties necessarily satisfied by order statistic policies (see \Cref{sec:define_OS} for a formal definition). We then construct a Weighted ERM policy which does not satisfy these properties and therefore conclude that it cannot be an order statistic policy.

Fix $\numS \geq 1$. For any vector $\mathbf{y} \in \Y^\numS$ and any permutation $\sigma$ on $\{1,\ldots,\numS\}$, we denote by $\mathbf{y}^{\sigma}$ the vector such that for every $i \in \{1,\ldots,\numS\}$, $y^\sigma_i = y_{\sigma(i)}$. Furthermore, consider an order statistic policy $\pi$. $\pi$ is characterized by a subset $S \subset \{1,\ldots,\numS\}$ and an index $r \in \{0,\ldots,\numS\}$. We next show that the order statistic policy $\pi$ must satisfy the following two properties.
\begin{enumerate}[label=Property \arabic*.,itemindent=*,leftmargin=0pt]
    \item If there exist \textit{distinct} $y_1,\ldots,y_\numS \in (0,1)^\numS$ such that $\pi(\mathbf{y}) = y_k$ for some $k \in \{1,\ldots,\numS\}$, then $k \in S$.
    \item For any permutation $\sigma$ supported on $S$ (i.e. such that $\sigma(i) = i$ for all $i \not \in S$), we have that $\pi(\mathbf{y}^\sigma)=\pi(\mathbf{y})$. 
\end{enumerate}
To prove property 1 we consider $k \in \{1,\ldots,\numS\}$ and we assume that there exist distinct $y_1,\ldots,y_\numS \in (0,1)^\numS$ such that, $\pi(\mathbf{y}) = y_k$. By definition of order statistic policies, it implies that $y_{(r),S} = y_{k}$ and since $y_k \not \in \{0,1\}$ it implies that $k \in S$. Property 2 is trivially implied by the definition of an order statistic.

We now construct a Weigted ERM policy which is not an order statistic policy. Consider the critical ratio $q=.5$ and the Weighted ERM policy $\pi'$ defined by the weights $w_1 = 2$ and $w_2=w_3 = w_4 = 1$. We note that, for any $\mathbf{y} \in \Y^4$,
\begin{equation*}
    \pi'(\mathbf{y}) = \begin{cases}
          \min \left\{ y_2, y_3, y_4 \right\} \quad \text{if $y_1$ is the smallest value}\\
          \max \left\{ y_2, y_3, y_4 \right\} \quad \text{if $y_1$ is the largest value}\\
          y_1 \quad \text{o.w.}
    \end{cases}
\end{equation*}
Let us show that $\pi'$ is not an order statistic policy. Assume for the sake of contradiction that $\pi'$ is an order statistic policy and let $S$ be the associated subset. We note that $\pi(0.1,0.2,0.3,0.4) = \pi(0.1,0.4,0.2,0.3)= \pi(0.1,0.3,0.4,0.2) = \pi(0.2,0.4,0.3,0.1)= 0.2$, therefore by applying property 1 we obtain that $S = \{1,2,3,4\}$. Furthermore $\pi(0.1,0.2,0.3,0.4) = 0.2$ whereas $\pi(0.3,0.2,0.1,0.4) = 0.3$ this contradicts property 2. As a consequence $\pi'$ is not an order statistic policy.
\end{proof}

\begin{proof}[\textbf{Proof of \Cref{lem:cdf_integral_reduction}}]
Let $\mesOut, \mesIn{1},\ldots, \mesIn{\numS} \in \Delta(\Y)$, $\numS \geq 1$, and any separable policy $\pi$ with associated function $P^{\pi}$.
We first show the following result on the expected loss of separable policies.
\begin{equation}
\label{eq:sep_policy_cost}
\mathbb{E}_{ \mathbf{y} \sim \mesIn{1} \times \ldots \times \mesIn{\numS}} \big[L(\pi(\mathbf{y}),\mesOut) \big] = (\cu+\co) \left[\int_0^{1}  \left[ (1- P^{\pi}(\mathbf{F}(y)) )  \cdot(F(y)-q)  + q \cdot (1-F(y)) \right] dy \right],    
\end{equation}
where $\mathbf{F}(y) = (F_i(y))_{i \in \{1,\ldots,\numS\}}$.

We let $G^{\pi} : y \mapsto \mathbb{P}_{ \mathbf{y} \sim \mesIn{1} \times \ldots \times \mesIn{\numS}, a \sim \pi(\mathbf{y}) } \left( a \leq y \right)$ denote the cumulative distribution function of decisions under the policy $\pi$. We then remark that,
\begin{equation*}
\mathbb{E}_{ \mathbf{y} \sim \mesIn{1} \times \ldots \times \mesIn{\numS}} \big[L(\pi (\mathbf{y}),\mesOut) \big] = \mathbb{E}_{\actVr \sim G^{\pi}} [L(a,\mesOut)].
\end{equation*}
In what follows, we use $\overline{F}$ to denote the complementary cumulative distribution, i.e., $\overline{F} = 1 - F$.

We have that,
\begin{align*}
\mathbb{E}_{ \actVr \sim G^{\pi}} [L(\actVr,\mesOut)] &\stackrel{(a)}{=} \cu \cdot \left( \mathbb{E}_{\mesOut}[y] - \mathbb{E}_{G^{\pi}} [\actVr] \right) + (\cu+\co) \cdot \int_0^{1} \left( \int_0^s {\mesOut}(y)dy \right) dG^{\pi}(s)\\
& \stackrel{(b)}{=} \cu \cdot \left( \mathbb{E}_{\mesOut}[y] - \mathbb{E}_{G^{\pi}} [\actVr] \right)  + (\cu+\co) \cdot \int_0^{1} \left( \int_y^{1} dG^{\pi}(s) \right) \mesOut(y)dy\\
& =  \cu \cdot \left(\mathbb{E}_{\mesOut}[y] - \mathbb{E}_{G^{\pi}} [\actVr] \right)  + (\cu+\co) \cdot \int_0^{1} \overline{G}^{\pi}(y) \cdot {\mesOut}(y)dy\\
& = \cu \cdot \left (\int_0^{1} \overline{\mesOut}(y) dy  - \int_0^1 \overline{G}^{\pi}(y) dy \right) + (\cu+\co)\int_0^{1} \overline{G}^{\pi}(y) \cdot {\mesOut}(y)dy\\
& = (\cu+\co) \cdot \left [ q \cdot \left (\int_0^{1} \overline{\mesOut}(y) dy  - \int_0^{1} \overline{G}^{\pi}(y) dy \right) + \int_0^{1} \overline{G}^{\pi}(y) \cdot \mesOut(y)dy \right] \\
& = (\cu+\co) \cdot  \int_0^{1} \left( \overline{G}^{\pi}(y) \cdot (\mesOut(y)-q)  + q \cdot (1-\mesOut(y)) \right)dz  \\
& \stackrel{(c)}{=} (\cu+\co) \left[\int_0^{1}  \left[ (1- P^{\pi}(\mathbf{F}(y)) )  \cdot(F(y)-q)  + q \cdot (1-F(y)) \right] dy \right].
\end{align*}
Here, $(a)$ follows from the expression of the cost,
\begin{equation*}
L(\actVr,F) = \cu \cdot (\mathbb{E}_{y \sim F}[y] - \actVr) + (\cu+\co) \cdot  \int_{0}^\actVr F(y)dy,
\end{equation*}
derived in \citet[Lemma A-1]{besbes2021big}. Equality $(b)$ follows from Fubini-Tonelli which holds because, $s \mapsto 1$ is a positive function and $(\mathbb{R},dG^{\pi})$ and $(\mathbb{R},dx)$ are complete, $\sigma$-finite measure spaces. $(c)$ follows from the definition of separable policies (see \Cref{def:separable_policies}): it ensures that for any $y \geq 0$, $G^{\pi}(y) = P^{\pi}(\mesIn{1}(y), \ldots, \mesIn{\numS}(y))$. This concludes the proof of \eqref{eq:sep_policy_cost}.

By using the simplified expressions of $\mathbb{E}_{ \mathbf{y} \sim \mesIn{1} \times \ldots \times \mesIn{\numS}} \big[L(\pi(\mathbf{y}),\mesOut) \big] $ derived in \eqref{eq:sep_policy_cost} and the following expression of the oracle cost,
\begin{equation*}
 L(a^*_{\mesOut},\mesOut)   = (\cu+\co) \cdot \int_0^{1}   \min\{(1-q) \cdot F_{0}(y) , q \cdot (1-F_{0}(y)) \} dy,
\end{equation*}
established in \citet[Proposition 1]{besbes2021big},  we obtain that for any distributions $\mesOut, \mesIn{1}, \ldots, \mesIn{\numS} \in \Delta(\Y)$ we have that,
\begin{equation*}
\mathbb{E}_{ \mathbf{y} \sim \mesIn{1} \times \ldots \times \mesIn{\numS}} \big[R(\pi(\mathbf{y}),\mesOut) \big]= (\cu+\co) \cdot \int_0^{1} \Psi^{\pi}(\mesOut(y), \mesIn{1}(y), \ldots, \mesIn{\numS}(y)) dy,
\end{equation*}
where $\Psi^{\pi}$ is a mapping from $[0,1]^{\numS+1}$ to $\mathbb{R}$ which satisfies, for every $(z_0,\ldots,z_{\numS}) \in [0,1]^{\numS+1}$, 
\begin{align*}
\Psi^{\pi}(z_0,\ldots,z_{\numS}) &= (1- P^{\pi} (z_1,\ldots,z_\numS))  \cdot(z_0-q)  + q \cdot (1-z_0) -  \min \{(1-q) \cdot z_0,  q \cdot (1-z_0) \}\\
&= P^{\pi} (z_1,\ldots,z_\numS)  \cdot(q-z_0)  + \max \{ z_0 -q,0 \}.
\end{align*}
\end{proof}

\begin{proof}[\textbf{Proof of \Cref{prop:reduction_bern}}]
Fix a sample size $\numS \geq 1$. Let $\pi$ be a separable policy associated to the function $P^{\pi}$. By inclusion we have that,
\begin{align*}
\sup_{\mesOut \in \Delta(\Y)} \sup_{\substack{\mesIn{1},\ldots\mesIn{\numS} \in \Delta(\Y)\\ \|\mesOut - \mesIn{i} \|_{K} \leq d(x_0,x_i) \, \forall i}} &\mathbb{E}_{ \mathbf{y} \sim \mesIn{1} \times \ldots \times \mesIn{\numS}} \big[R(\pi(\mathbf{y}),\mesOut) \big]\\
&\geq \sup_{\mu_0 \in [0,1]}  \sup_{\substack{\mu_1,\ldots,\mu_\numS \in [0,1]\\ | \mu_i - \mu_0 | \leq d(x_0,x_i) \, \forall i}} \mathbb{E}_{ \mathbf{y} \sim \mathcal{B}(\mu_1) \times \ldots \times \mathcal{B}(\mu_{\numS})} \big[R(\pi(\mathbf{y}),\mathcal{B}(\mu_{0})) \big].
\end{align*}
We next prove the reverse inequality. 
\Cref{lem:cdf_integral_reduction} implies that, for any distributions $\mesOut, \mesIn{1}, \ldots, \mesIn{\numS} \in \Delta(\Y)$ we have that,
\begin{equation*}
\mathbb{E}_{ \mathbf{y} \sim \mesIn{1} \times \ldots \times \mesIn{\numS}} \big[R(\pi(\mathbf{y}),\mesOut) \big]= (\cu+\co) \cdot \int_0^{1} \Psi^{\pi}(\mesOut(y), \mesIn{1}(y), \ldots, \mesIn{\numS}(y)) dy,
\end{equation*}
where $\Psi^{\pi}$ is a mapping from $[0,1]^{\numS+1}$ to $\mathbb{R}$ which satisfies, for every $(z_0,\ldots,z_{\numS}) \in [0,1]^{\numS+1}$, 
\begin{equation*}
\Psi^{\pi}(z_0,\ldots,z_{\numS}) = P^{\pi} (z_1,\ldots,z_\numS)  \cdot(q-z_0)  + \max \{ z_0 -q,0 \}.
\end{equation*}
Therefore, 
\begin{align}
\label{eq:psi_gen}
\sup_{\mesOut \in \Delta(\Y)} \sup_{\substack{\mesIn{1},\ldots\mesIn{\numS} \in \Delta(\Y)\\ \|\mesOut - \mesIn{i} \|_{K} \leq d(x_0,x_i) \, \forall i}} &\mathbb{E}_{ \mathbf{y} \sim \mesIn{1} \times \ldots \times \mesIn{\numS}} \big[R(\pi(\mathbf{y}),\mesOut) \big] \nonumber\\
&\quad =  (\cu+\co) \sup_{\mesOut \in \Delta(\Y)} \sup_{\substack{\mesIn{1},\ldots\mesIn{\numS} \in \Delta(\Y)\\ \|\mesOut - \mesIn{i} \|_{K} \leq d(x_0,x_i) \, \forall i}} \int_0^{1} \Psi^{\pi}(\mesOut(y), \mesIn{1}(y), \ldots, \mesIn{\numS}(y)) dy.
\end{align}
Furthermore, for any sequence of parameters $(\mu_i)_{i \in \{0, \ldots, \numS \}} \in [0,1]^{\numS}$, by setting $\mesOut = \mathcal{B}\left(\mu_0 \right)$  and $\mesIn{i} = \mathcal{B}\left(\mu_i \right)$ for every $i \in \{1,\ldots, \numS\}$, we have that for every $y \in [0,1]$,
\begin{equation*}
\Psi^{\pi}(\mesOut(y), \mesIn{1}(y), \ldots, \mesIn{\numS}(y)) = \Psi^{\pi}(1-\mu_0,\ldots, 1-\mu_{\numS}). 
\end{equation*}
Hence,
\begin{align}
\label{eq:psi_bern}
\sup_{\mu_0 \in [0,1]}  \sup_{\substack{\mu_1,\ldots,\mu_\numS \in [0,1]\\ | \mu_i - \mu_0 | \leq d(x_0,x_i) \, \forall i}} &\mathbb{E}_{ \mathbf{y} \sim \mathcal{B}(\mu_1) \times \ldots \times \mathcal{B}(\mu_{\numS})} \big[R(\pi(\mathbf{y}),\mathcal{B}(\mu_{0})) \big] \nonumber \\
&\quad = (\cu + \co)  \sup_{\mu_0 \in [0,1]}  \sup_{\substack{\mu_1,\ldots,\mu_\numS \in [0,1]\\ | \mu_i - \mu_0 | \leq d(x_0,x_i) \, \forall i}} \Psi^{\pi}(1-\mu_0,\ldots, 1-\mu_{\numS}).
\end{align}
It follows from \eqref{eq:psi_gen} and \eqref{eq:psi_bern} that, to conclude the proof, it is sufficient to show that,
\begin{align*}
\sup_{\mesOut \in \Delta(\Y)} \sup_{\substack{\mesIn{1},\ldots\mesIn{\numS} \in \Delta(\Y)\\ \|\mesOut - \mesIn{i} \|_{K} \leq d(x_0,x_i) \, \forall i}}  &\int_0^{1} \Psi^{\pi}(\mesOut(y), \mesIn{1}(y), \ldots, \mesIn{\numS}(y)) dy \\
&\quad\leq \sup_{\mu_0 \in [0,1]}  \sup_{\substack{\mu_1,\ldots,\mu_\numS \in [0,1]\\ | \mu_i - \mu_0 | \leq d(x_0,x_i) \, \forall i}} \Psi^{\pi}(1-\mu_0,\ldots, 1-\mu_{\numS}).
\end{align*}
By setting $z_i  =1- \mu_i$ for all $i \in \{0,\ldots,\numS\}$ we obtain that,
\begin{equation}
\label{eq:compact_problem}
\bar{\psi} := \sup_{z_0 \in [0,1]} \sup_{ \substack{ z_1,\ldots,z_\numS \in [0,1] \\ | z_i - z_{0} | \leq d(x_0,x_i) \, \forall i }}  \Psi^{\pi}(z_0,\ldots, z_{\numS}) = \sup_{\mu_0 \in [0,1]} \sup_{ \substack{ \mu_1,\ldots,\mu_{\numS} \in [0,1] \\ | \mu_i - \mu_{0} | \leq d(x_0,x_i) \, \forall i }}  \Psi^{\pi}(1-\mu_0,\ldots, 1-\mu_{\numS}).
\end{equation}
For every $\mesOut \in \Delta(\Y)$ and any $\mesIn{1},\ldots,\mesIn{\numS} \in \Delta(\Y)$ which satisfies for every $i \in \{1,\ldots, \numS\}$ that $\|\mesOut - \mesIn{i} \|_{K} \leq d(x_0,x_i)$, we have that
\begin{equation*}
\int_0^{1} \Psi^{\pi}(\mesOut(y), \mesIn{1}(y), \ldots, \mesIn{\numS}(y)) dy \stackrel{(a)}{\leq} \int_0^1 \bar{\psi} dy = \bar{\psi},
\end{equation*}
where $(a)$ follows from \eqref{eq:compact_problem} and from the fact that for every $y \in [0,1]$, we have, by definition of the Kolmogorov norm, that $|\mesOut(y) - \mesIn{i}(y)| \leq d(x_0,x_i)$ for every $i \in \{1,\ldots,\numS\}$.
By taking the supremum over $\mesOut, \mesIn{1},\ldots,\mesIn{\numS}$, we obtain the desired inequality.
\end{proof}

 \begin{proof}[\textbf{Proof of \Cref{prop:worst_history}}]
 Let $\pi$ be a non-decreasing separable policy with associated function $P^{\pi}$.
 To prove this statement we derive a stronger structural statement on the expected regret of a non-decreasing separable policy $\pi$ against Bernoulli distributions. Let $\mu_0 \in [0,1]$ and $i \in \{ 1,\ldots,\numS\}$. Fix $(\mu_j)_{j\in \{1,\ldots, \numS\} \setminus \{i\}} \in [0,1]$ and define,
 \begin{equation*}
 \rho_i: \mu \mapsto
 \mathbb{E}_{ \mathbf{y} \sim \mathcal{B}\left( \mu_1\right) \times \ldots \times \mathcal{B}\left( \mu_{i-1}\right) \times \mathcal{B}\left( \mu\right) \times \mathcal{B}\left( \mu_{i+1}\right) \times \ldots \times \mathcal{B}\left( \mu_{\numS}\right)}  \big[R(\pi(\mathbf{y}),\mathcal{B}\left( \mu_0\right)) \big].
 \end{equation*}
 We next show that $\rho_i$ is non-decreasing on $[0,1]$ when $\mu_0 \in [0,1-q]$ and non-increasing on $[0,1]$ when $\mu_0 \in [1-q,1]$. Assuming this fact, we conclude the proof by noting that when $\mu_0 \in [0,1-q]$, the regret is non-decreasing in any of the parameters $\mu_1,\ldots,\mu_\numS$. Therefore, we set their value to the largest feasible value, which means that for every $i$, we set the mean of the $i^{th}$ historical distribution to $\min(\mu_0+d(x_0,x_i),1)$. This implies that, when $\mu_0 \in [0,1-q]$
 \begin{equation*}
 \sup_{\substack{\mu_1,\ldots,\mu_\numS \in [0,1]\\ | \mu_i - \mu_0 | \leq d(x_0,x_i) \, \forall i}} \mathbb{E}_{ \mathbf{y} \sim \mathcal{B}(\mu_1) \times \ldots \times \mathcal{B}(\mu_{\numS})} \big[R(\pi(\mathbf{y}),\mathcal{B}(\mu_{0})) \big]
 = \mathbb{E}_{ \mathbf{y} \sim \mathcal{B}(\mu_0 + d(x_0,x_1)) \times \ldots \times \mathcal{B}(\mu_0 + d(x_0,x_\numS))} \big[R(\pi(\mathbf{y}),\mathcal{B}(\mu_{0})) \big].
\end{equation*}
A similar argument enables to show that when $\mu_0 \in [1-q,1]$ we have that,
 \begin{equation*}
 \sup_{\substack{\mu_1,\ldots,\mu_\numS \in [0,1]\\ | \mu_i - \mu_0 | \leq d(x_0,x_i) \, \forall i}} \mathbb{E}_{ \mathbf{y} \sim \mathcal{B}(\mu_1) \times \ldots \times \mathcal{B}(\mu_{\numS})} \big[R(\pi(\mathbf{y}),\mathcal{B}(\mu_{0})) \big]
 = \mathbb{E}_{ \mathbf{y} \sim \mathcal{B}(\mu_0 - d(x_0,x_1)) \times \ldots \times \mathcal{B}(\mu_0 - d(x_0,x_\numS))} \big[R(\pi(\mathbf{y}),\mathcal{B}(\mu_{0})) \big],
\end{equation*}
where $\mathcal{B}(\mu)$ is a Bernoulli with mean $1$ whenever $\mu \geq 1$ and $0$ if $\mu \leq 0$.

We next show that $\rho_i$ is non-decreasing on $[0,1]$ when $\mu_0 \in [0,1-q]$ and non-increasing on $[0,1]$ when $\mu_0 \in [1-q,1]$. 

We established in the proof of Proposition \ref{prop:reduction_bern} that for any sequence of parameters $(\mu_j)_{j \in \{0, \ldots, \numS \}} \in [0,1]^{\numS}$ the expected regret of a separable policy $\pi$ when facing Bernoulli distributions is given by,
 \begin{equation*}
\mathbb{E}_{ \mathbf{y} \sim \mathcal{B}(\mu_1) \times \ldots \times \mathcal{B}(\mu_{\numS})} \big[R(\pi(\mathbf{y}),\mathcal{B}(\mu_{0})) \big] = (\cu + \co) \cdot \Psi^{\pi}(1-\mu_0,\ldots, 1-\mu_{\numS}),
 \end{equation*}
where $\Psi^{\pi}$ is a mapping from $[0,1]^{\numS+1}$ to $\mathbb{R}$ which satisfies, for every $(z_0,\ldots,z_{\numS}) \in [0,1]^{\numS+1}$, 
\begin{equation*}
\Psi^{\pi}(z_0,\ldots,z_{\numS}) = P^{\pi} (z_1,\ldots,z_\numS)  \cdot(q-z_0)  + \max \{ z_0 -q,0 \}.
\end{equation*}
From the expression of $\Psi^{\pi}$, we note that when $\mu_0 \in [0,1-q]$ (resp. $\mu_0 \in [1-q,1]$), we have that $\rho_i$ is non-decreasing (resp. non-increasing) if and only if, $\mu \mapsto P^{\pi} (1-\mu_1,\ldots,1-\mu_{i-1},1-\mu,1-\mu_{i+1},\ldots1-\mu_{\numS})$ is non-increasing. But this follows from the definition of a non-decreasing separable policy (see \Cref{def:separable_policies}).
 \end{proof}

\begin{proposition}
\label{prop:mean_not_separable}
Fix $\numS = 2$ and consider the mean policy define for every $\bm{y} \in \Y^2$ as, 
\begin{equation*}
\pi(\bm{y}) = \frac{y_1+y_2}{2}.
\end{equation*}
Then, $\pi$ is not a separable policy.
\end{proposition}

\begin{proof}[\textbf{Proof of \Cref{prop:mean_not_separable}}]
Assume for the sake of contradiction that $\pi$ is a separable policy and let $P^\pi$ be the associated function such that for every $F_1,F_2 \in \Delta(\Y)$ and every $y \in [0,1]$ we have that,
\begin{equation*}
\mathbb{P}_{\bm{y} \sim F_1 \times F_2}\left(\pi(\bm{y}) \leq y \right) = P^{\pi}(F_1(y),F_2(y)).
\end{equation*}
Let $F_1$ be the distribution which puts all mass at $0$, $F_2$ be the distribution which puts all mass at $1$ and $F_3$ be the distribution which puts all mass at $0.6$. We then have that,
\begin{equation*}
P^{\pi}(1,0) = P^{\pi}(F_1(0.4),F_2(0.4)) = \mathbb{P}_{\bm{y} \sim F_1 \times F_2} \left(\pi(\bm{y}) \leq 0.4 \right)  = \mathbb{P} \left(\pi(0,1) \leq 0.4 \right) =0.
\end{equation*}
\begin{equation*}
P^{\pi}(1,0) = P^{\pi}(F_1(0.4),F_3(0.4)) = \mathbb{P}_{\bm{y} \sim F_1 \times F_3} \left(\pi(\bm{y}) \leq 0.4 \right)  = \mathbb{P} \left(\pi(0,0.6) \leq 0.4 \right) =1.
\end{equation*}
This leads to a contradiction. Hence $\pi$ is not a separable policy.
\end{proof}

 \begin{lemma}
 \label{lem:mixture_closed}
     The space of non-decreasing separable policies is closed under mixtures.
 \end{lemma}

 \begin{proof}[\textbf{Proof of \Cref{lem:mixture_closed}}]
     Let $\pi_1$ and $\pi_2$ be two non-decreasing separable policies with associated functions $P^{\pi_1}$ and $P^{\pi_2}$. Fix $\lambda \in [0,1]$ and consider the policy $\pi$ which selects the inventory decision selected by $\pi_1$ (resp. $\pi_2$) with probability $\lambda$ (resp. $1-\lambda$). Then for any historical distributions $\mesIn{1},\ldots,\mesIn{\numS}$ and every $y \in \Y$,
     \begin{align*}
    \mathbb{P}_{\mathbf{y} \sim \mesIn{1}  \times \ldots \times \mesIn{\numS}, a \sim \pi(\bm{y})} \left(  a \leq y \right) &= \lambda \cdot    \mathbb{P}_{\mathbf{y} \sim \mesIn{1}  \times \ldots \times \mesIn{\numS},a \sim \pi_1(\bm{y})} \left( a \leq y \right)+ (1-\lambda) \cdot    \mathbb{P}_{\mathbf{y} \sim \mesIn{1} \times \ldots \times \mesIn{\numS},a \sim \pi_2(\bm{y})} \left( a \leq y \right)\\
    &\stackrel{(a)}{=} \lambda \cdot P^{\pi_1}(\mesIn{1}(y), \ldots, \mesIn{\numS}(y)) + (1-\lambda) \cdot P^{\pi_2}(\mesIn{1}(y), \ldots, \mesIn{\numS}(y)),
    \end{align*}
    where $(a)$ holds because $\pi_1$ and $\pi_2$ are separable policies. Let $P^\pi = \lambda \cdot P^{\pi_1} + (1-\lambda) \cdot P^{\pi_2}$. We note that this implies that $\pi$ is a separable policy, furthermore it is non-decreasing because the convex combination of non-decreasing functions is still non-decreasing.
 \end{proof}

 \section{Properties of Policies and Proofs of Results Presented in \Cref{sec:OS}}
\label{sec:apx_prop_pol}

\subsection{Order statistic policies}\label{sec:define_OS}

An important property of the ERM policy under the Newsvendor loss is that it has a simple closed form solution in terms of order statistics. For a sequence of samples $\mathbf{y} \in \Y^\numS$ and a subset of indices $S \subset \{1,\ldots,\numS\}$, we define for every $i \in\{1,\ldots, |S|\}$  the quantity $y_{(i),S}$ as the $i^{th}$ order statistic, i.e., $i^{th}$ smallest element, in
$(y_j)_{j \in S}$. We will refer to $S$ as the subset of the order statistic and $i$ as its rank.
For example, if $\mathbf{y}=(10,20,30,40,40,50)$ and $S=\{3,4,5,6\}$, then $y_{(3),S}=40$. In particular, the ERM policy is defined by setting $S = \{1,\ldots,\numS\}$ as
\begin{equation}
\label{eq:closed_form}
\pi^{\mathrm{ERM}}(\mathbf{y})  =  y_{( \lceil q \cdot \numS  \rceil),S}.
\end{equation}
More generally, we define the set of order statistic policies as follows.
\begin{definition}[Order statistic policies]
    \label{def:OS}
    We say that a data-driven policy $\pi$ is an order statistic policy if and only if there exit a subset $S \subset \{1,\ldots,\numS\}$ and an index $i \in \{0,\ldots,\numS\}$, such that for every $\mathbf{y} \in \Y^{\numS}$,
    \begin{equation*}
        \pi(\mathbf{y}) = y_{(i),S},
    \end{equation*}
    where we slightly abuse notation and define $y_{(0),S} = 0$ and $y_{(i),S} = 1$ if $i > |S|$. 
\end{definition}
The set of order statistic policies contains the ERM policy. These policies will more broadly be used to derive alternative policies in \Cref{sec:apx_improve_mix} to improve the performance over ERM. We note that the class of order statistic policies is included in the one of counting policies as formalized below.
\begin{proposition}
\label{prop:OS_to_counting}
Every order statistic policy is a counting policy.
\end{proposition}

\begin{proof}[\textbf{Proof of \Cref{prop:OS_to_counting}}]
Let $\pi$ associated with a subset $S \subset \{1,\ldots,\numS\}$ and an index $i \in \{0,\ldots,\numS\}$. If $i = 0$, we note that $\pi(\mathbf{y})$ is the constant policy equal to $0$ which is a counting policy associated with the constant function $\kappa^{\pi}$ which always takes value $1$. A similar argument implies that if $i = |S|+1$, $\pi$ is a counting policy.

In all other cases, we have that, for every $y \in [0,1]$,
$\mathbbm{1} \{ \pi(\mathbf{y}) \leq y \} = \mathbbm{1} \{ y_{(i),S} \leq y \}.$ Therefore, by defining the function $\kappa^\pi$ for every $\mathbf{b} \in \{0,1\}^\numS$ as, 
\begin{equation*}
    \kappa^\pi(\mathbf{b}) = \begin{cases}
        1 \quad \text{if $|\{ j \in S \; \text{s.t.} \; b_j = 1 \}| \geq i$}\\
        0 \quad \text{otherwise},
    \end{cases}
\end{equation*}
we obtain that for every $\mathbf{y} \in \Y^\numS$,
\begin{equation*}
    \kappa^\pi(\mathbbm{1} \{y_1 \leq y \},\ldots, \mathbbm{1} \{y_\numS \leq y \}) = \mathbbm{1} \{ y_{(i),S} \leq y \}.
\end{equation*}
This therefore implies that $\pi$ is a counting policy. 
\end{proof}

\subsection{Proofs of Results Presented in \Cref{sec:OS}}

\begin{proof}[\textbf{Proof of \Cref{prop:counting_to_sep}}]
Consider a counting policy $\pi$ with associated counting function $\kappa^{\pi}$. 
Furthermore, for every subset $S \subset \{1,\ldots,\numS\}$, let $\kappa^{\pi}(1_{S},0_{-S}) = \kappa^{\pi}(b_1,\ldots,b_\numS)$, where $b_i = 1$ if $i \in S$ and $0$ if $i \not \in S$. 
Let $\mesIn{1},\ldots, \mesIn{\numS} \in \Delta(\Y)$ and let $y \in \mathbb{R}$. We have that,
\begin{align*}
\mathbb{P}_{ \mathbf{y} \sim \mesIn{1} \times \ldots \times \mesIn{\numS}} \left( \pi(\mathbf{y}) \leq y \right)
&= \mathbb{E}_{ \mathbf{y} \sim \mesIn{1} \times \ldots \times \mesIn{\numS}} \left[ \mathbbm{1} \left\{ \pi(\mathbf{y}) \leq y \right \} \right]\\
 &\stackrel{(a)}{=}  \mathbb{E}_{ \mathbf{y} \sim \mesIn{1} \times \ldots \times \mesIn{\numS}} \left[ \kappa^{\pi} \left( \mathbbm{1} \{ y_1 \leq y \},\ldots,\mathbbm{1} \{ y_{\numS} \leq y  \} \right)\right]\\
 &= \sum_{S \subset \{1,\ldots,\numS\}} \kappa^{\pi}(1_{S},0_{-S}) \cdot \mathbb{P}_{ \mathbf{y} \sim \mesIn{1} \times \ldots \times \mesIn{\numS}} \left( y_i \leq y \; \text{for all $i \in S$}; y_i > y \; \text{for all $i \not \in S$} \right)\\
 & \stackrel{(b)}{=}\sum_{S \subset \{1,\ldots,\numS\}} \kappa^{\pi}(1_{S},0_{-S}) \cdot \prod_{i \in S} H_i(y) \cdot \prod_{i \not \in S} \left(1 - H_i(y) \right),
\end{align*}
where $(a)$ follows from the definition of counting policies and $(b)$  holds as the $(y_i)_{i \in \{1,\ldots,\numS\}}$ are independent random variables.

Therefore, by setting $P^{\pi}$ defined for every $h_1,\ldots, h_\numS \in [0,1]^{\numS}$ as
\begin{equation*}
P^{\pi}(h_1,\ldots, h_\numS ) = \sum_{S \subset \{1,\ldots,\numS\}} \kappa^{\pi}(1_{S},0_{-S}) \cdot \prod_{i \in S} h_i \cdot \prod_{i \not \in S} \left(1 - h_i \right),
\end{equation*}
we have that,
\begin{equation*}
\mathbb{P}_{ \mathbf{y} \sim \mesIn{1} \times \ldots \times \mesIn{\numS}} \left( \pi(\mathbf{y}) \leq y \right) = P^{\pi}(\mesIn{1}(y), \ldots, \mesIn{\numS}(y)).  
\end{equation*}
This shows that $\pi$ is a separable policy.

Furthermore \Cref{prop:nondec_counting} implies that the mapping $\kappa^{\pi}$ is non-decreasing in the sense that, for every $\mathbf{b},\mathbf{b}' \in  \{0,1\}^{\numS}$ and for every $i \in \{1,\ldots, \numS\}$ we have that $b_i \leq b_i'$, then $\kappa(\mathbf{b}) \leq \kappa(\mathbf{b}')$. We next show that $P^{\pi}$ is non-decreasing in the sense defined in \Cref{def:WERM}. 

Let $i \in \{1,\ldots,\numS\}$. For every subset $S \in \{1,\ldots,\numS\} \setminus \{i\}$, we let, 
\begin{align*}
\kappa^{\pi}_{i+}(1_{S},0_{-S}) &= \kappa^{\pi}(b_1,\ldots,b_\numS) \quad \text{where $b_j = 1$ if $j \in S \cup \{i\}$ and $0$ if $j \not \in S$}\\
\kappa^{\pi}_{i-}(1_{S},0_{-S}) &= \kappa^{\pi}(b_1,\ldots,b_\numS) \quad \text{where $b_j = 1$ if $j \in S$ and $0$ if $j=i$ or $j \not \in S$}.
\end{align*}
Furthermore, for every $h_1,\ldots,h_{i-1},h_{i+1},\ldots h_\numS \in  [0,1]$, we let 
\begin{equation*}
p_i : h \mapsto P^\pi(h_1,\ldots,h_{i-1},h,h_{i+1},\ldots h_\numS).
\end{equation*}
We next show that $p_i$ is non-decreasing on $[0,1]$. We remark that for every $h$,
\begin{align*}
p_i(h) &= \sum_{S \subset \{1,\ldots,\numS\} \setminus \{i\} } \Big[ \kappa^{\pi}_{i+}(1_{S},0_{-S}) \cdot h \cdot \prod_{j \in S} h_j \cdot \prod_{j  \in \{1,\ldots,\numS\} \setminus \left(S \cup \{i\} \right)} (1-h_j) \\
&\qquad+ \kappa^{\pi}_{i-}(1_{S},0_{-S}) \cdot (1-h) \cdot \prod_{j \in S} h_j \cdot \prod_{j  \in \{1,\ldots,\numS\} \setminus \left(S \cup \{i\} \right)} (1-h_j) \Big]\\
&= \sum_{S \subset \{1,\ldots,\numS\} \setminus \{i\} } \prod_{j \in S} h_j \cdot \prod_{j  \in \{1,\ldots,\numS\} \setminus \left(S \cup \{i\} \right)} (1-h_j)  \cdot \left[ \left(\kappa^{\pi}_{i+}(1_{S},0_{-S}) - \kappa^{\pi}_{i-}(1_{S},0_{-S})\right) \cdot h + \kappa^{\pi}_{i-}(1_{S},0_{-S}) \right].
\end{align*}
Furthermore, for every $S \subset \{1,\ldots,\numS\} \setminus \{i\}$, $\kappa^{\pi}$ is non-decreasing therefore, $\kappa^{\pi}_{i+}(1_{S},0_{-S}) \geq \kappa^{\pi}_{i-}(1_{S},0_{-S})$ and thus, $p$ is a linear mapping with non-negative slope. 

This shows that $p$ is non-decreasing and thus implies that $\pi$ is a non-decreasing separable policy.
\end{proof}

\begin{proof}[\textbf{Proof of \Cref{prop:WEM_to_counting}}]
Consider a vector of weights $\mathbf{w} \in [0,1]^{\numS}$ and consider the Weighted ERM policy $\pi^{\mathbf{w}}$. 
The action selected by Weighted ERM satisfies,
\begin{align*}
\pi^{\mathbf{w}}(\mathbf{y}) &= \inf \left \{\actVr \in \Y \, \text{s.t.} \, \frac{\sum_{i=1}^\numS w_i \cdot \mathbbm{1} \left\{ y_i \leq \actVr \right\} }{\sum_{i=1}^{\numS} w_{i}} \geq \frac{\cu}{\cu + \co} \right \} = \min \left \{\actVr \in \Y \, \text{s.t.} \, \frac{\sum_{i=1}^\numS w_i \cdot \mathbbm{1} \left\{ y_i \leq \actVr \right\} }{\sum_{i=1}^{\numS} w_{i}} \geq \frac{\cu}{\cu + \co} \right \},
\end{align*}
where the last equality holds because $a \mapsto \frac{\sum_{i=1}^\numS w_i \cdot \mathbbm{1} \left\{ y_i \leq \actVr \right\} }{\sum_{i=1}^{\numS} w_{i}}$ is upper semicontinuous, which implies that the superlevel sets are closed.
We define $\kappa^{\pi}$ such that for every $\mathbf{b} \in \{0,1\}^{\numS}$,
\begin{equation*}
\kappa^{\pi}(\mathbf{b}) =  \begin{cases}
1 \quad \text{ if $\frac{\sum_{i=1}^\numS w_i \cdot b_{i} }{\sum_{i=1}^{\numS} w_{i}} \geq \frac{\cu}{\cu+\co}$},\\
0 \quad \text{otherwise}.
\end{cases}
\end{equation*}
Using this mapping, we re-express $\pi^{\mathbf{w}}$ for every $\mathbf{y} \in \Y^{\numS}$ as,
\begin{equation*}
\pi^{\mathbf{w}}(\mathbf{y}) = \min \left \{\actVr \in \Y \, \text{s.t.} \, \kappa^{\pi}(\mathbbm{1} \left\{ y_1 \leq \actVr \right\}, \ldots, \mathbbm{1} \left\{ y_\numS \leq \actVr \right\}) = 1 \right \}.
\end{equation*}
Fix $\mathbf{y} \in \Y^{\numS}$ and  $y \in \mathbb{R}$.
 First, if $\mathbbm{1} \left\{ \pi^{\mathbf{w}}(\mathbf{y}) \leq y \right\} =0$, we have that for every $\actVr \leq y$, 
\begin{equation*}
 \kappa^{\pi}(\mathbbm{1} \left\{ y_1 \leq \actVr \right\}, \ldots, \mathbbm{1} \left\{ y_\numS \leq \actVr \right\}) = 0, 
\end{equation*}
and in particular this holds for $\actVr = y$. Hence, $\mathbbm{1} \left\{ \pi^{\mathbf{w}}(\mathbf{y}) \leq y \right\} = \kappa^{\pi}(\mathbbm{1} \left\{ y_1 \leq y \right\}, \ldots, \mathbbm{1} \left\{ y_\numS \leq y \right\}) = 0$.

Moreover, if $\mathbbm{1} \left\{ \pi^{\mathbf{w}}(\mathbf{y}) \leq y \right\} =1$, we know that there exists $\actVr_0 \leq y$ such that, 
\begin{equation*}
 \kappa^{\pi}(\mathbbm{1} \left\{ y_1 \leq \actVr_0 \right\}, \ldots, \mathbbm{1} \left\{ y_\numS \leq \actVr_0 \right\}) = 1.
\end{equation*}
Fix such $\actVr_0$. We remark that $\kappa^{\pi}$ is non-decreasing in the sense that, for every $\mathbf{b},\mathbf{b}' \in  \{0,1\}^{\numS}$, if for every $i \in \{1,\ldots, \numS\}$ we have that $b_i \leq b_i'$, then $\kappa^\pi(\mathbf{b}) \leq \kappa^\pi(\mathbf{b}')$. Therefore,
\begin{equation*}
 \kappa^{\pi}(\mathbbm{1} \left\{ y_1 \leq y \right\}, \ldots, \mathbbm{1} \left\{ y_\numS \leq y \right\})  \geq  \kappa^{\pi}(\mathbbm{1} \left\{ y_1 \leq \actVr_0 \right\}, \ldots, \mathbbm{1} \left\{ y_\numS \leq \actVr_0 \right\})  = 1.
\end{equation*}
We conclude that, $\mathbbm{1} \left\{ \pi^{\mathbf{w}}(\mathbf{y}) \leq y \right\}  =  \kappa^{\pi}(\mathbbm{1} \left\{ y_1 \leq y \right\}, \ldots, \mathbbm{1} \left\{ y_\numS \leq y \right\} = 1$.

Hence, we proved  that $\pi$ is a counting policy with associated counting function $\kappa^{\pi}$.
\end{proof}

\begin{proposition}
\label{prop:nondec_counting}
    Fix a counting policy $\pi$. Then the counting function $\kappa^{\pi}$ associated with $\pi$  is non-decreasing in the sense that, for every $\mathbf{b},\mathbf{b}' \in  \{0,1\}^{\numS}$, if for every $i \in \{1,\ldots, \numS\}$ we have that $b_i \leq b_i'$, then $\kappa^\pi(\mathbf{b}) \leq \kappa^\pi(\mathbf{b}')$.
\end{proposition}
\begin{proof}[\textbf{Proof of \Cref{prop:nondec_counting}}]
Let $\pi$ be a counting policy and let $\kappa^\pi$ be its associated counting function. Assume for the sake of contradiction that $\kappa^\pi$ is not non-decreasing. Therefore, there exist $\mathbf{b}\neq\mathbf{b'} \in \{0,1\}^\numS$ such that for every $i \in \{1,\ldots,\numS\}$, we have $b_i \leq b'_i$ and $\kappa^\pi(\mathbf{b}) = 1$ whereas $\kappa^\pi(\mathbf{b'}) = 0$. 

Consider the sequence $(y_i)_{i \in \{1,\ldots,\numS\}}$ which satisfies for every $i \in \{1,\ldots,\numS\}$
\begin{equation*}
    y_i = \begin{cases}
        0.75 \quad \text{if $b_i = 0$ and $b'_i = 1$}\\
        1 \quad \text{if $b_i = b'_i = 0$}\\
        0.5 \quad \text{if $b_i = b'_i = 1$}.
    \end{cases}
\end{equation*}
 By definition of counting policies we have that,
\begin{equation}
    \label{eq:to_contradict}
    \mathbbm{1} \{ \pi(\mathbf{y}) \leq 0.5 \} = \kappa^\pi \left( \mathbbm{1} \{  y_1 \leq 0.5 \}, \ldots, \mathbbm{1} \{ y_\numS \leq 0.5 \} \right) \stackrel{(a)}{=} \kappa^\pi(\mathbf{b}) = 1,
\end{equation}
where $(a)$ follows from the fact that, by the construction of $(y_i)_{i \in \{1,\ldots,\numS\}}$, we have $y_i \le 0.5$ if and only if $b_i = 1$. Therefore, \eqref{eq:to_contradict} implies that $\pi(\mathbf{y})$ is less than or equal to $0.5$. Furthermore we note that we also have
\begin{equation*}    
\mathbbm{1} \{ \pi(\mathbf{y}) \leq 0.75 \} = \kappa \left( \mathbbm{1} \{  y_1 \leq 0.75 \}, \ldots, \mathbbm{1} \{ y_\numS \leq 0.75 \} \right) \stackrel{(a)}{=} \kappa^{\pi}(\mathbf{b'}) = 0,
\end{equation*}
where $(a)$ follows again from the fact that by the construction of $(y_i)_{i \in \{1,\ldots,\numS\}}$, $y_i \le 0.75$ if and only if $b_i' = 1$. This implies that $\pi(\mathbf{y})$ is greater than $0.75$, but this  contradicts \eqref{eq:to_contradict}. This concludes the proof.
\end{proof}

\section{Improving ERM with the $\LERM$ Policy}
\label{sec:apx_improve_mix}
We observed in \Cref{sec:shape} that, when all past dissimilarities are identical, i.e., $d(x_1,x_0) = \cdots =d(x_n,x_0)=\zeta$ for a fixed $\zeta \geq 0$, the worst-case regret of the ERM policy is non-monotone as a function of the sample size. We distinguished in \Cref{fig:LERM_tilde} two forms of non-monotonicity for the ERM policy. On the one hand, ERM suffers from a local non-monotonicity, which materializes when the performance of the policy may deteriorate when adding a \textit{single} sample. On the other hand, ERM also suffers from a more acute form of non-monotonicity which we call global because the policy achieves its best performance for a finite sample size and then its performance deteriorates non-trivially with more data. %for arbitrarily large sample sizes.

In this section, we progressively design an alternative policy which we dub $\LERM$ which overcomes the shortcomings observed for ERM and considerably improves its performance. $\LERM$ is a randomized policy over the set of order statistic policies, a class of policies defined in \Cref{sec:define_OS}.
 We refer to these randomized policies as mixture of order statistic policies and we define them formally as follows.
\begin{definition}[Mixture of order statistics policies]
\label{def:mix_pol}
Fix $\numS \geq 1$ and denote by $\Lambda$ the set of collections of matrix coefficients $\bm{\lambda} = (\lambda_{S,r})_{S \subset \{1,\ldots,\numS\}, r \in \{0,\ldots,\numS\}}$ such that $\sum_{S,r} \lambda_{S,r} =1$. We note that the rows of the matrix are indexed by subsets and the columns by indices.
Then, for every $\bm{\lambda} \in \Lambda$, we say that $\pi^{\bm{\lambda}}$ is a mixture of order statistic policies if and only if, for any realization of past outcomes $\mathbf{y} \in [0,1]^{\numS}$, we have that,
\begin{equation*}
    \pi^{\bm{\lambda}}(\mathbf{y}) = y_{(r),S} \qquad \text{with probability $\lambda_{S,r}$},
\end{equation*}
where $y_{(r),S}$ is an order statistic as defined in \Cref{sec:define_OS}.
\end{definition}
Given that ERM is an element of the set, a natural improvement over ERM would be to consider the \textit{best} possible mixture of order statistic policies defined for every $\numS \geq 1$ and every context configuration  $\mathbf{x} \in \X^{\numS+1}$ as, $\pi^{\bm{\lambda}^*}$, where $\bm{\lambda}^*$ satisfies,
\begin{equation}
\label{eq:best_mix_main}
    \bm{\lambda}^* \in \argmin_{ \bm{\lambda} \in \Lambda} \mathsf{Reg}_n(\pi^{\bm{\lambda}},\mathbf{x}).
\end{equation}
However, solving \eqref{eq:best_mix_main} is in general computationally challenging given the size of the space $\Lambda$. Instead, we suggest to use a policy which searches over a restricted set of possible weights. Inspired, by 
\cite{besbes2021big}, we first formally define $\ERMt$, a policy which uses the subset $S = \{1,\ldots,\numS\}$ and adequately randomizes over ranks $r$. 

Assume without loss of generality that the context vector $(x_i)_{i\in \{1,\ldots,\numS\}}$ are ordered such that, $d(x_1,x_0) \leq \ldots \leq d(x_\numS,x_0)$. For every $k \in \{1,\ldots,\numS\}$, let $\tilde{\Lambda}_k$ be a subset of weights which satisfies,
\begin{equation*}
    \tilde{\Lambda}_k = \{ \bm{\lambda} \in \Lambda \text{ s.t. } \lambda_{S,(r)} =0 \text{ if } S \neq \{1,\ldots,k\} \}.
\end{equation*}
$\ERMt$ is therefore defined as the mixture of order statistics policy $\pi^{\bm{\lambda}}$ such that,
\begin{equation*}
    \bm{\lambda} = \argmin_{ \bm{\lambda} \in \tilde{\Lambda}_\numS} \mathsf{Reg}_n(\pi^{\bm{\lambda}},\mathbf{x}).
\end{equation*}
We compare in \Cref{fig:SAA_vs_mix} the performance of ERM and the one of $\ERMt$ in the setting where all past contexts have the same dissimilarity $\zeta$ with the new out-of-sample context.
\begin{figure}[h]
\centering
\begin{tikzpicture}[scale=.8][h]
\begin{axis}[
            title={},
            xmin=0,xmax=200,
            ymin=0.0,ymax=0.15,
            scaled y ticks={base 10:2},
            width=10cm,
            height=8cm,
            table/col sep=comma,
            xlabel = number of samples $\numS$,
            ylabel = worst-case regret,
            grid=both,
            skip coords between index={0}{1},
            legend pos=north east]
    
\addplot [blue,mark=square,mark options={scale=.2}] table[x={Ns},y={epsilon=0.1}] {Data/paper_all_data_SAA_q9.csv};
    \addlegendentry{ERM}
     \addplot [red,mark=square,mark options={scale=.2}] table[x={Ns},y={epsilon=0.1}] {Data/paper_all_data_mix_q9.csv};
    \addlegendentry{$\ERMt$}                    
\end{axis}
\end{tikzpicture}
\caption{\textbf{Comparison of ERM and $\ERMt$.} The figure depicts the worst-case regret of the Empirical Risk Minimization and $\ERMt$ policies for $\zeta = .1$ and as a function of the sample size $\numS$ (q=.9).}\label{fig:SAA_vs_mix}
\end{figure}
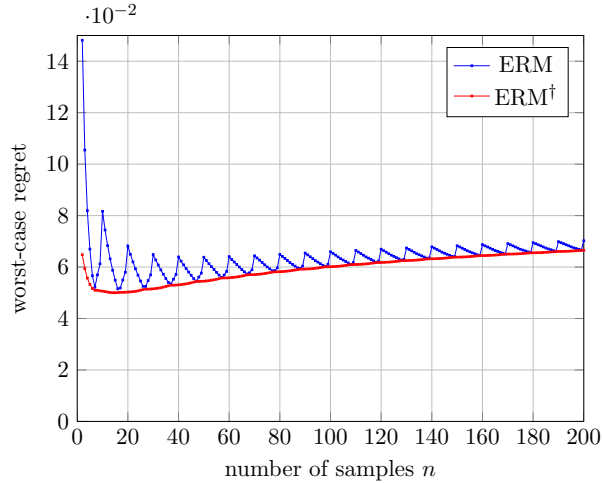
We observe that $\ERMt$ alleviates the local non-monotonicity by randomizing over the ranks of order statistics. However, this randomization does not yet resolve the global non-monotonicity behavior and the worst-case regret still deteriorates with more samples after achieving a minimal value.

A natural way to avoid the global non-monotonicity of $\ERMt$ is to consider the policy which uses $\ERMt$ on a subset of the total samples which achieves the lowest worst-case regret even when having more samples available. We refer to this policy as $\LERM$ and formally define it as follows.
The $\LERM$ policy is a mixture of order statistics policy $\pi^{\bm{\lambda}}$ such that,
\begin{equation}
\label{eq:def_LERM}
    \bm{\lambda} \in  \argmin_{ \bm{\lambda} \in \cup_{k=1}^\numS \tilde{\Lambda}_k} \mathsf{Reg}_n(\pi^{\bm{\lambda}},\mathbf{x}).
\end{equation}
A key difference between $\LERM$ and $\ERMt$ is that the latter restricts attention to mixture of order statistics using all samples, whereas the former potentially discards some samples when considered necessary.
Indeed, for every $ k,k' \in \{1,\ldots,\numS\}$ such that $k \neq k'$ we have that $\tilde{\Lambda}_k \cap \tilde{\Lambda}_{k'} = \emptyset$ by definition. Therefore, a weight which minimizes \eqref{eq:def_LERM} is included in a certain set $\tilde{\Lambda}_{k^*}$. We refer to $k^*$ as the ``effective sample size''. 

\section{Alternative Tie-breaking for WERM}
\label{sec:tie-break}
In this section we show that \Cref{cor:regret_WERM} holds for alternative tie-breaking rules for Weighted ERM policies. 

Fix $\numS \geq 1$. For every non-negative sequence of weights $\mathbf{w} = (w_i)_{i \in \{1,\ldots,\numS\}}$ and for every historical observations $\bm{y} \in \Y^{\numS}$, we note that the weighted empirical loss defined for every $\actVr \in \Y$ as, $\hat{L}(a) = {\sum_{i=1}^{\numS} w_{i} \cdot \ell(a,y_{i})}$ is convex as it is a non-negative linear combination of convex functions.
Consequently, the set $\argmin_{a} \sum_{i=1}^{\numS} \hat{L}(a)$, is a non-empty interval of $\Y$.

Furthermore, for every $a \in \Y$, we have that,
\begin{equation*}
\hat{L}(a) = \sum_{i =1}^{\numS} w_i \cdot \left[ \co \cdot ( \actVr - y_i)^{+} + \cu \cdot (y_i - \actVr)^{+} \right] = \sum_{\substack{i=1\\ y_i \leq a} }^\numS w_i \cdot \co \cdot ( \actVr - y_i) + \sum_{\substack{i=1\\ y_i \geq a} }^\numS w_i \cdot \cu \cdot ( y_i - \actVr).
\end{equation*}
Therefore, the sub-gradient of $\hat{L}$ satisfies that for every $a \in \Y$,
\begin{equation*}
\partial \hat{L}(a) = \left[ \sum_{\substack{i=1\\ y_i < a}}^{\numS} \co \cdot w_{i}  - \sum_{\substack{i=1\\ y_i \geq a}}^{\numS} \cu \cdot w_{i} ; \sum_{\substack{i=1\\ y_i \leq a}}^{\numS} \co \cdot w_{i}  - \sum_{\substack{i=1\\ y_i > a}}^{\numS} \cu \cdot w_{i} \right].
\end{equation*}
Hence, $0$ is in  $\partial  \hat{L}(a)$ if and only if,
\begin{equation*}
\frac{\sum_{i = 1}^{\numS} w_{i} \cdot \mathbbm{1} \left\{ y_{i} < a \right\} }{\sum_{i=1}^{\numS} w_{i}} \leq \frac{\cu}{\cu + \co} \leq \frac{\sum_{i = 1}^{\numS} w_{i} \cdot \mathbbm{1} \left\{ y_{i} \leq a \right\} }{\sum_{i=1}^{\numS} w_{i}}. 
\end{equation*}
This implies that $\actVr$ is a minimizer of $\hat{L}$ if and only if, $\actVr \in [a_{min}; a_{max}]$, where,
\begin{align*}
a_{min}(\bm{y}) &= \inf \left \{ a \text{ s.t. }  \frac{\sum_{i=1}^\numS w_i \cdot \mathbbm{1} \left \{ y_i \leq a \right\}}{\sum_{j=1}^\numS w_{j}} \geq \frac{\cu}{\cu+\co}  \right \},\\
a_{max}(\bm{y}) &= \sup \left \{ a \text{ s.t. }  \frac{\sum_{i=1}^\numS w_i  \cdot \mathbbm{1} \left \{ y_i < a \right\}}{\sum_{j=1}^\numS w_j}  \leq \frac{\cu}{\cu+\co}  \right \}. 
\end{align*}
For instance, when $\numS=2$, $y_1=0$, $y_2=1$, $w_1=w_2=1$, and $\cu=\co=1$, we obtain $a_{\min}(\bm{y})=0$ and $a_{\max}(\bm{y})=1$.

In what follows, we will consider more general Weighted ERM policies which are defined by a sequence of non-negative weights $\bm{w}$ and by a tie-breaking parameter $\lambda \in [0,1]$. Specifically, these more general Weighted ERM policies selects for every $\bm{y}$ the action $\lambda \cdot a_{min}(\bm{y}) + (1-\lambda) \cdot a_{max}(\bm{y})$. We denote this policy by $\pi^{\bm{w}}_{\lambda}$. Note that we do not allow the convex combination parameter $\lambda$ to depend on the historical observations $\bm{y}$.  

We next prove the following generalization of \Cref{cor:regret_WERM}.

\begin{theorem}\label{thm:general_WERM}
Let  $\numS \geq 1$. Let $\bm{w}$ be a sequence of non-negative weights and let $\lambda \in [0,1]$. Then, for any sequence of contexts $\mathbf{x} = (x_i)_{i \in \{0,\ldots,\numS\}} \in \X^{\numS+1}$, the general Weighted ERM policy $\pi^{\bm{w}}_{\lambda}$ satisfies,
\begin{align*}
\wreg[\pi^{\bm{w}}_{\lambda}]{\mathbf{x}}  &= \max \Big\{ \sup_{\mu_0 \in [0,1-q]} \mathbb{E}_{\mathbf{y} \sim \mathcal{B}(\mu_0 + d(x_0,x_1)) \times \ldots \times \mathcal{B}\left( \mu_0 + d(x_0,x_\numS) \right)} \left[  R \left(\pi^{\bm{w}}_{\lambda}(\mathbf{x},\mathbf{y}),\mathcal{B}(\mu_0) \right) \right],\\
&\quad \sup_{\mu_0 \in [1-q,1]} \mathbb{E}_{\mathbf{y} \sim \mathcal{B}(\mu_0 - d(x_0,x_1)) \times \ldots \times \mathcal{B}\left( \mu_0 - d(x_0,x_\numS) \right)} \left[  R \left(\pi^{\bm{w}}_{\lambda}(\mathbf{x},\mathbf{y}),\mathcal{B}(\mu_0) \right) \right] \Big \}.
\end{align*}
\end{theorem}
\begin{proof}[\textbf{Proof of \Cref{thm:general_WERM}}]
For the sake of simple notations, we are not marking the dependence in $\mathbf{x}$ when not necessary.

Remark that the policy $\pi^{\bm{w}}_{\lambda}$ selects the convex combination of the actions prescribed by the policies $\tilde{\pi}^{\bm{w}}_{0}$ and $\tilde{\pi}^{\bm{w}}_{1}$, which respectively selects $a_{max}(\bm{y})$ and $a_{min}(\bm{y})$. We first establish the following result on the worst-case regret of convex combinations.
\begin{lemma}
\label{lem:cvx_to_random}
Fix $\numS \geq 1$ and $\bm{x} \in \mathcal{X}^{\numS+1}$. Let $\pi_{0}$ and $\pi_{1}$ be two non-decreasing separable policies. Let $\lambda \in [0,1]$ and, consider the policy $\pi_{\lambda}$  defined for every $\bm{y} \in [0,1]^{\numS}$ as,
\begin{equation*}
\pi_{\lambda}(\mathbf{x},\bm{y}) = \lambda \cdot \pi_{1}(\mathbf{x},\bm{y}) + (1-\lambda) \pi_{0}(\mathbf{x},\bm{y}) 
\end{equation*}
Then we have that,
\begin{align*}
\wreg[\pi_{\lambda}]{\mathbf{x}}  &= \max \Big\{ \sup_{\mu_0 \in [0,1-q]} \mathbb{E}_{\mathbf{y} \sim \mathcal{B}(\mu_0 + d(x_0,x_1)) \times \ldots \times \mathcal{B}\left( \mu_0 + d(x_0,x_\numS) \right)} \left[  R \left(\pi_{\lambda}(\mathbf{x},\mathbf{y}),\mathcal{B}(\mu_0) \right) \right],\\
&\quad \sup_{\mu_0 \in [1-q,1]} \mathbb{E}_{\mathbf{y} \sim \mathcal{B}(\mu_0 - d(x_0,x_1)) \times \ldots \times \mathcal{B}\left( \mu_0 - d(x_0,x_\numS) \right)} \left[  R \left(\pi_{\lambda}(\mathbf{x},\mathbf{y}),\mathcal{B}(\mu_0) \right) \right] \Big \}.
\end{align*}
\end{lemma}
Hence, to obtain the desired result it suffices to establish that $\tilde{\pi}^{\bm{w}}_{0}$ and $\tilde{\pi}^{\bm{w}}_{1}$ are non-decreasing separable policies.

The policy $\tilde{\pi}^{\bm{w}}_{1}$ corresponds to our definition of Weighted ERM policies (see \Cref{def:WERM}) and is therefore a counting policy (by \Cref{prop:WEM_to_counting}). Hence, \Cref{prop:counting_to_sep} implies that it is a non-decreasing separable policy.
We next establish that $\tilde{\pi}^{\bm{w}}_{0}$ is also a counting policy.

We first argue that for every $\bm{y} \in \Y^{\numS}$, we have that,
\begin{equation}
\label{eq:amax_inf}
a_{max}(\bm{y}) = \sup \left \{ a \text{ s.t. }  \frac{\sum_{i=1}^\numS w_i  \cdot \mathbbm{1} \left \{ y_i < a \right\}}{\sum_{j=1}^\numS w_j}  \leq \frac{\cu}{\cu+\co}  \right \} = 
\inf \left \{ a \text{ s.t. }  \frac{\sum_{i=1}^\numS w_i  \cdot \mathbbm{1} \left \{ y_i \leq a \right\}}{\sum_{j=1}^\numS w_j}  > \frac{\cu}{\cu+\co}  \right \}
\end{equation}
Denote by $A_{1}$ the supremum and by $A_{2}$ the infimum. Let $\epsilon >0.$ We have that,
\begin{equation*}
\frac{\cu}{\cu+\co} < \frac{\sum_{i=1}^\numS w_i  \cdot \mathbbm{1} \left \{ y_i < A_{1} + \epsilon \right\}}{\sum_{j=1}^\numS w_j}  \leq  \frac{\sum_{i=1}^\numS w_i  \cdot \mathbbm{1} \left \{ y_i \leq A_{1} + \epsilon \right\}}{\sum_{j=1}^\numS w_j}.
\end{equation*}
This implies that $A_{2} \leq A_{1} + \epsilon$. As this hold for every $\epsilon$, we conclude that $A_{2} \leq A_{1}$.

Assume for the sake of contradiction that $A_{2} < A_{1}$. There exists $\epsilon > 0$ and $A \in (A_{2},A_{1})$  such that $A + \epsilon < A_{1}$. We then have that,
\begin{equation}
\label{eq:contradiction}
\frac{\cu}{\cu+\co} \stackrel{(a)}{<} \frac{\sum_{i=1}^\numS w_i  \cdot \mathbbm{1} \left \{ y_i \leq A \right\}}{\sum_{j=1}^\numS w_j} \stackrel{(b)}{\leq}  \frac{\sum_{i=1}^\numS w_i  \cdot \mathbbm{1} \left \{ y_i < A + \epsilon \right\}}{\sum_{j=1}^\numS w_j} \stackrel{(c)}{\leq} \frac{\cu}{\cu+\co},
\end{equation}
where $(a)$ and $(c)$ holds respectively because $A > A_{2}$ and $A+\epsilon < A_{1}$, and $(b)$ follows from the fact that for every $y,a,a' \in Y$ such that $a < a'$, we have that  $\mathbbm{1} \left \{ y \leq a \right\} \leq \mathbbm{1} \left \{ y < a' \right\}$.

The equation \eqref{eq:contradiction} leads to a contradiction which implies that $A_{2} \geq A_{1}$. We thus conclude that both quantities are equal, and that \eqref{eq:amax_inf} holds.

Furthermore we note that $\inf \left \{ a \text{ s.t. }  \frac{\sum_{i=1}^\numS w_i  \cdot \mathbbm{1} \left \{ y_i \leq a \right\}}{\sum_{j=1}^\numS w_j}  > \frac{\cu}{\cu+\co}  \right \}$ is achieved. For the sake of contradiction, assume it is not, and denote by $A$ the infimum. We have that, $\frac{\sum_{i=1}^\numS w_i  \cdot \mathbbm{1} \left \{ y_i \leq A \right\}}{\sum_{j=1}^\numS w_j} \leq  \frac{\cu}{\cu+\co}$. Moreover, $a \mapsto \frac{\sum_{i=1}^\numS w_i  \cdot \mathbbm{1} \left \{ y_i \leq a \right\}}{\sum_{j=1}^\numS w_j}$ is right-continuous and piecewise constant. Hence there exists $\epsilon > 0$ such that,
\begin{equation*}
\frac{\cu}{\cu+\co} \geq \frac{\sum_{i=1}^\numS w_i  \cdot \mathbbm{1} \left \{ y_i \leq A \right\}}{\sum_{j=1}^\numS w_j} = \frac{\sum_{i=1}^\numS w_i  \cdot \mathbbm{1} \left \{ y_i \leq A + \epsilon \right\}}{\sum_{j=1}^\numS w_j}.
\end{equation*}
This contradicts the fact that $A = \inf \left \{ a \text{ s.t. }  \frac{\sum_{i=1}^\numS w_i  \cdot \mathbbm{1} \left \{ y_i \leq a \right\}}{\sum_{j=1}^\numS w_j}  > \frac{\cu}{\cu+\co}  \right \}$. Hence, the infimum must be achieved.
 
 Thus, we have established that, for every $\bm{y} \in \Y^{\numS}$,
 \begin{equation*}
 \tilde{\pi}^{\bm{w}}_{0}(\bm{y}) = \min \left \{ a \text{ s.t. }  \frac{\sum_{i=1}^\numS w_i  \cdot \mathbbm{1} \left \{ y_i \leq a \right\}}{\sum_{j=1}^\numS w_j}  > \frac{\cu}{\cu+\co}  \right \} = \min \left \{ a \text{ s.t. }  \kappa(\mathbbm{1} \left\{ y_1 \leq \actVr \right\}, \ldots, \mathbbm{1} \left\{ y_\numS \leq \actVr \right\}) = 1 \right \},
 \end{equation*}
where the function $\kappa$ is defined for every $\mathbf{b} \in \{0,1\}^{\numS}$ as
\begin{equation*}
\kappa(\mathbf{b}) =  \begin{cases}
1 \quad \text{ if $\frac{\sum_{i=1}^\numS w_i \cdot b_{i} }{\sum_{i=1}^{\numS} w_{i}} > \frac{\cu}{\cu+\co}$,}\\
0 \quad \text{otherwise}.
\end{cases}
\end{equation*}
By applying the same argument as in the proof of \Cref{prop:WEM_to_counting} we conclude that $ \tilde{\pi}^{\bm{w}}_{0}$ is a counting policy. \Cref{prop:counting_to_sep} implies that it is a non-decreasing separable policy.
\end{proof}

\begin{proof}[\textbf{Proof of \Cref{lem:cvx_to_random}}]
For the sake of simple notations, we are not marking the dependence in $\mathbf{x}$ when not necessary.
Consider the \textit{randomized} policy $\tilde{\pi}_{\lambda}$ defined for every $\bm{y} \in [0,1]^{\numS}$ as,
\begin{equation*}
\tilde{\pi}_{\lambda}(\bm{y}) = 
\begin{cases}
 \pi_{1}(\bm{y}) \text{ with probability $\lambda$,}\\
  \pi_{0}(\bm{y}) \text{ with probability $1-\lambda$.}
\end{cases}
\end{equation*}
Remark that $\tilde{\pi}_{\lambda}(\bm{y})$ is a mixture of non-decreasing separable policies. Hence, \Cref{lem:mixture_closed} implies that it is a non-decreasing separable policy, and by \Cref{thm:regret_MOS} we have that,
\begin{align*}
\wreg[\tilde{\pi}_{\lambda}]{\mathbf{x}}  &= \max \Big\{ \sup_{\mu_0 \in [0,1-q]} \mathbb{E}_{\mathbf{y} \sim \mathcal{B}(\mu_0 + d(x_0,x_1)) \times \ldots \times \mathcal{B}\left( \mu_0 + d(x_0,x_\numS) \right)} \left[  R \left(\tilde{\pi}_{\lambda}(\mathbf{x},\mathbf{y}),\mathcal{B}(\mu_0) \right) \right],\\
&\quad \sup_{\mu_0 \in [1-q,1]} \mathbb{E}_{\mathbf{y} \sim \mathcal{B}(\mu_0 - d(x_0,x_1)) \times \ldots \times \mathcal{B}\left( \mu_0 - d(x_0,x_\numS) \right)} \left[  R \left(\tilde{\pi}_{\lambda}(\mathbf{x},\mathbf{y}),\mathcal{B}(\mu_0) \right) \right] \Big \}.
\end{align*}
Furthermore, we note that for every 
Let $\mesOut, \mesIn{1},\ldots, \mesIn{\numS} \in \Delta \left([0,1] \right)$, 
\begin{align*}
\mathbb{E}_{ \mathbf{y} \sim \mesIn{1} \times \ldots \times \mesIn{\numS}} \big[L(\pi_{\lambda}(\mathbf{y}),\mesOut) \big] 
&\stackrel{(a)}{\leq}  \mathbb{E}_{ \mathbf{y} \sim \mesIn{1} \times \ldots \times \mesIn{\numS}} \big[\lambda \cdot L(\pi_{1}(\mathbf{y}),\mesOut) + (1-\lambda)\cdot L(\pi_{0}(\mathbf{y}),\mesOut)  \big] \nonumber \\
&=\mathbb{E}_{ \mathbf{y} \sim \mesIn{1} \times \ldots \times \mesIn{\numS}} \big[L(\tilde{\pi}_{\lambda}(\mathbf{y}),\mesOut) \big], 
\end{align*}
where $(a)$ follows from the convexity of $L$. Importantly, we note that this inequality is an equality when $\mesOut$ is a Bernoulli distribution as $a \mapsto L(a,\mathcal{B}(\mu))$ is a linear function in $[0,1]$ for every $\mu \in [0,1]$.

This implies that, 
\begin{equation}
\label{eq:cvx_lower_random}
\mathbb{E}_{ \mathbf{y} \sim \mesIn{1} \times \ldots \times \mesIn{\numS}} \big[R(\pi_{\lambda}(\mathbf{y}),\mesOut) \big] \leq \mathbb{E}_{ \mathbf{y} \sim \mesIn{1} \times \ldots \times \mesIn{\numS}} \big[R(\tilde{\pi}_{\lambda}(\mathbf{y}),\mesOut) \big],
\end{equation}
with equality when $F_0$ is a Bernoulli distribution.
By taking a supremum we obtain that,
\begin{equation*}
\wreg[\pi_{\lambda}]{\mathbf{x}}  \leq \wreg[\tilde{\pi}_{\lambda}]{\mathbf{x}}.
\end{equation*}
We conclude that,
\begin{align*}
\wreg[\pi_{\lambda}]{\mathbf{x}}  &\leq \wreg[\tilde{\pi}_{\lambda}]{\mathbf{x}}\\
&= \max \Big\{ \sup_{\mu_0 \in [0,1-q]} \mathbb{E}_{\mathbf{y} \sim \mathcal{B}(\mu_0 + d(x_0,x_1)) \times \ldots \times \mathcal{B}\left( \mu_0 + d(x_0,x_\numS) \right)} \left[  R \left(\tilde{\pi}_{\lambda}(\mathbf{x},\mathbf{y}),\mathcal{B}(\mu_0) \right) \right],\\
&\quad \sup_{\mu_0 \in [1-q,1]} \mathbb{E}_{\mathbf{y} \sim \mathcal{B}(\mu_0 - d(x_0,x_1)) \times \ldots \times \mathcal{B}\left( \mu_0 - d(x_0,x_\numS) \right)} \left[  R \left(\tilde{\pi}_{\lambda}(\mathbf{x},\mathbf{y}),\mathcal{B}(\mu_0) \right) \right] \Big \} \\
&\stackrel{(a)}{=}  \max \Big\{ \sup_{\mu_0 \in [0,1-q]} \mathbb{E}_{\mathbf{y} \sim \mathcal{B}(\mu_0 + d(x_0,x_1)) \times \ldots \times \mathcal{B}\left( \mu_0 + d(x_0,x_\numS) \right)} \left[  R \left(\pi_{\lambda}(\mathbf{x},\mathbf{y}),\mathcal{B}(\mu_0) \right) \right],\\
&\quad \sup_{\mu_0 \in [1-q,1]} \mathbb{E}_{\mathbf{y} \sim \mathcal{B}(\mu_0 - d(x_0,x_1)) \times \ldots \times \mathcal{B}\left( \mu_0 - d(x_0,x_\numS) \right)} \left[  R \left(\pi_{\lambda}(\mathbf{x},\mathbf{y}),\mathcal{B}(\mu_0) \right) \right] \Big \}\\
&\leq \wreg[\pi_{\lambda}]{\mathbf{x}},
\end{align*}
where $(a)$ follows from the equality case of \eqref{eq:cvx_lower_random}.
Therefore all the inequalities are in fact equalities and we have established that,
\begin{align*}
\wreg[\pi_{\lambda}]{\mathbf{x}}  &= \max \Big\{ \sup_{\mu_0 \in [0,1-q]} \mathbb{E}_{\mathbf{y} \sim \mathcal{B}(\mu_0 + d(x_0,x_1)) \times \ldots \times \mathcal{B}\left( \mu_0 + F_0d(x_0,x_\numS) \right)} \left[  R \left(\pi_{\lambda}(\mathbf{x},\mathbf{y}),\mathcal{B}(\mu_0) \right) \right],\\
&\quad \sup_{\mu_0 \in [1-q,1]} \mathbb{E}_{\mathbf{y} \sim \mathcal{B}(\mu_0 - d(x_0,x_1)) \times \ldots \times \mathcal{B}\left( \mu_0 - d(x_0,x_\numS) \right)} \left[  R \left(\pi_{\lambda}(\mathbf{x},\mathbf{y}),\mathcal{B}(\mu_0) \right) \right] \Big \}.
\end{align*}
\end{proof}

\section{Translation of Previous State-of-the-art Bounds to our Setting}
\label{sec:apx_Mohri}
In this section, we provide a self-contained explanation of how to translate the bound in \cite{mohri2012new} to our setting. For any fixed out-of-sample distribution $F_{0}$ and for $\mathbf{y} \in [0,1]^{\numS}$, let 
$\tilde{R}_n(\mathbf{y}) = R(\pi^{\mathrm{ERM}}(\mathbf{y}),F_{0}).$
\citet[eq.~(7)]{mohri2012new}\footnote{We believe \citet[eq.~(7)]{mohri2012new} missed a factor of $2$ in the third term.} show that for $\delta \in (0,1)$,
\begin{equation}
\label{eq:mohri_bound}
    \mathbb{P}\left( \tilde{R}_n(\mathbf{y}) \geq 4 \cdot \mathfrak{comp}_{\numS} + \frac{2}{\numS} \sum_{i=1}^{\numS} \sup_{\actVr \in \Y}|L(\actVr,F_i) - L(\actVr,F_0)| + \max(q,1-q) \cdot \sqrt{\frac{8\log(2\delta^{-1})}{\numS}} \right) \leq  \delta,
\end{equation}
where the probability is taken with respect to outcomes sampled from historical distributions $\mathbf{y} \sim F_1 \times \ldots \times F_{\numS}$ and $\mathfrak{comp}_{\numS}$ is a notion of sequential Rademacher complexity defined as,
\begin{equation*}
    \mathfrak{comp}_{\numS} = \mathbb{E}_{\mathbf{y} \sim F_1 \times \ldots \times F_{\numS}} \left[ \mathbb{E}_{\bm{\sigma}} \left[ \sup_{\actVr \in \Y} \sum_{i=1}^\numS \sigma_i \cdot \ell(a,y_i) \right] \right],
\end{equation*}
with $\bm{\sigma}$ being a uniform variable sampled from $\{-1,1\}^\numS$.

Let $\beta_\numS = 4 \cdot \mathfrak{comp}_{\numS} + \frac{2}{\numS} \sum_{i=1}^{\numS} \sup_{\actVr \in \Y}|L(\actVr,F_i) - L(\actVr,F_0)|$ and note that for now, we let the dependence in $F_0,\ldots,F_\numS$ be implicit.

We next convert this probabilistic bound into a bound on the expected regret. By applying the change of variable $\eta = \max(q,1-q) \cdot \sqrt{\frac{8\log(2\delta^{-1})}{\numS}}$, one can rewrite \eqref{eq:mohri_bound} as,
\begin{equation*}
    \mathbb{P}\left( \tilde{R}_n(\mathbf{y}) - \beta_\numS \geq \eta \right) \leq 2 \exp \left(- \frac{\eta^2 \cdot \numS}{8\max(q,1-q)^2} \right).
\end{equation*}
Therefore,
\begin{align*}
\label{eq:converting_to_expectation}
    \mathbb{E}\left[ \tilde{R}_n(\mathbf{y})\right] - \beta_\numS  
    &= \int_{0}^\infty \mathbb{P} \left( \tilde{R}_n(\mathbf{y}) - \beta_\numS  > \eta \right) d \eta - \int_{-\infty}^0\mathbb{P} \left( \tilde{R}_n(\mathbf{y}) - \beta_\numS  < \eta \right) d\eta \nonumber \\
    &\leq2 \int_{0}^1 \exp \left(- \frac{\eta^2 \cdot \numS}{8\max(q,1-q)^2} \right) d\eta \nonumber \\
    &=4\max(q,1-q) \sqrt{\frac{2\pi}{\numS}}   \int_{0}^1 \frac{\sqrt{\numS}}{2\max(q,1-q)} \cdot \frac{1}{\sqrt{2\pi}} \exp \left(- \frac{\eta^2 \cdot \numS}{8\max(q,1-q)^2} \right) d\eta \nonumber \\
    &\stackrel{(a)}{=} 4\max(q,1-q) \sqrt{\frac{2\pi}{\numS}} \cdot \left \{ \Phi \left(\frac{2\max(q,1-q)}{\sqrt{\numS}} \right) - \Phi(0) \right \}, 
\end{align*}
where $(a)$ follows by remarking that the integrand is the pdf of a normal with mean $0$ and standard deviation $\frac{\sqrt{\numS}}{2\max(q,1-q)}$ and $\Phi$ denotes the cdf of the standard Gaussian distribution.

By explicitly marking the dependence in all distributions, we have established that an upper bound on the expected regret of ERM derived using the probabilistic bound of \cite{mohri2012new} takes the form,
\begin{align*}
\mathbb{E}_{\bm{y} \sim F_1\times \ldots \times F_\numS}\left[ \tilde{R}_n(\mathbf{y})\right] 
&\leq \beta_\numS(F_0,\ldots,F_\numS) + 4\max(q,1-q) \sqrt{\frac{2\pi}{\numS}} \cdot \left \{ \Phi \left(\frac{2\max(q,1-q)}{\sqrt{\numS}} \right) - \Phi(0) \right \}\\
&\stackrel{(a)}{\leq}4\max(q,1-q) \sqrt{\frac{2\pi}{\numS}} \cdot \left \{ \Phi \left(\frac{2\max(q,1-q)}{\sqrt{\numS}} \right) - \Phi(0) \right \}\\
&\qquad + \sup_{F_0 \in \Delta(\Y)} \sup_{ \substack{F_1,\ldots, F_n \in \Delta(\Y)\\ \|F_i - F_0\|_K \leq \zeta \; \forall i}} \beta_\numS(F_0,\ldots,F_\numS),
\end{align*}
where $(a)$ is necessary to derive an upper bound which does not depend on  $F_0,\ldots,F_\numS$. 

Let $\bar{\beta}_\numS = \sup_{F_0 \in \Delta(\Y)} \sup_{ \substack{F_1,\ldots, F_n \in \Delta(\Y)\\ \|F_i - F_0\|_K \leq \zeta \; \forall i}} \beta_\numS(F_0,\ldots,F_\numS)$. Evaluating this quantity is challenging even in the i.i.d. case where we impose that $F_0 = F_1 \ldots = F_\numS$. The main challenge comes from the fact that $\beta_\numS$ involves the Rademacher complexity defined previously. The common approach in statistical learning is to upper bound $\beta_\numS$ using combinatorial arguments and the VC-dimension.  We refer the reader to \cite{mohri2012new} for an example of such derivations.

In this paper, we propose to use a \textit{lower bound} on $\bar{\beta}_\numS$. By doing so, we compare our results to bounds that are better than ones that could be derived from the literature.

We note that,
\begin{align*}
\bar{\beta}_\numS &\geq \beta_\numS(\mathcal{B}(\zeta),\mathcal{B}(0),\ldots,\mathcal{B}(0))\\
 &= 4  \mathbb{E}_{\mathbf{y} \sim \mathcal{B}(0) \times \ldots \times \mathcal{B}(0)} \left[ \mathbb{E}_{\bm{\sigma}} \left[ \sup_{\actVr \in \Y} \sum_{i=1}^\numS \sigma_i \cdot \ell(a,y_i) \right] \right] + 2 \sup_{\actVr' \in \Y}|L(\actVr',\mathcal{B}(0)) - L(\actVr',\mathcal{B}(\zeta))|\\
 &\stackrel{(a)}{=} 4 \mathbb{E}_{\bm{\sigma}} \left[ \sup_{\actVr \in \Y} \sum_{i=1}^\numS \sigma_i \cdot \ell(a,0) \right] + 2 \max(q,1-q) \cdot \zeta,
\end{align*}
where the second term in $(a)$ follows by noting that the supremum is achieved for $\actVr' = 0$ or $1$.

In \Cref{sec:mmBound}, we compare our result to \cite{mohri2012new} by using the favorable bound,
\begin{equation*}
4 \mathbb{E}_{\bm{\sigma}} \left[ \sup_{\actVr \in \Y} \sum_{i=1}^\numS \sigma_i \cdot \ell(a,0) \right] + 2 \max(q,1-q) \cdot \zeta + 4\max(q,1-q) \sqrt{\frac{2\pi}{\numS}} \cdot \left \{ \Phi \left(\frac{2\max(q,1-q)}{\sqrt{\numS}} \right) - \Phi(0) \right \}.
\end{equation*}

\section{ERM Performance for ``Mild'' Instances}
\label{sec:beyond_worst_case}

In \Cref{sec:mmBound}, we showed that our exact characterization of the worst-case performance of data-driven policies demonstrates that the achievable regret with a small number of samples is much lower than suggested by state-of-the-art upper bounds. Our exact analysis also allowed us to uncover new insights on the \textit{shape} of the learning curve of ERM. In \Cref{sec:shape}, we observed three salient features of the worst-case performance of ERM: $i)$ the performance improves dramatically after the first few samples, $ii)$ ERM may exhibits a ``local'' non-monotonicity behavior which can be corrected by considering convex combination of order statistics and, $iii)$ ERM exhibits a ``global'' non-monotonicity suggesting that the decision-maker should use a smaller number of samples even when having access to many more.

In what follows, we investigate whether these insights hold in a less adversarial setting where the instance does not vary as a function of the number of samples and the demand distributions are not necessarily Bernoulli distributions. To numerically evaluate the regret of ERM for ``milder'' instances, we fix an out-of-sample distribution $F_{\text{future}}$ and given $\zeta > 0$, we compute the expected regret of ERM when accessing $\numS$ samples from the distribution with cumulative distribution function $F_{\text{past}}(y) = \min(F_{\text{future}}(y) + \zeta,1)$ for all $y \in [0,1]$. Our goal is to compute for every $\numS$ the quantity $\mathbb{E}_{y_1\sim F_{\text{past}},\ldots,y_n\sim F_{\text{past}} } \big[R(\ERM(\bm{y},F_{\text{future}}) \big].$

For every instance, we estimate this quantity as follows. We generate $K = 10^5$ samples $\{\tilde{d}_1,\ldots, \tilde{d}_K\}$ from $F_{\text{future}}$ to compute the out-of-sample cost. We then draw $M = 1000$ in-sample demand vectors such that for every $m \in \{1,\ldots,M\}$, we have that $\bm{y}^m$ is an $n$-dimensional vector where each component is sampled independently from  $F_{\text{past}}$. Our estimator of the expected regret of ERM is defined as,
\begin{equation*}
\frac{1}{K} \sum_{k = 1}^K \left[  \frac{1}{M} \sum_{m=1}^M\ell(\ERM(\bm{y}^m), d_k)  - \ell(a^*_{F_{\text{future}}}, d_k)  \right].
\end{equation*}

We plot in \Cref{fig:mild_instances} the regret of ERM for three distributions supported on $[0,1]$: truncated\footnote{For a distribution defined on $\mathbb{R}$ with pdf $f$ and cdf $F$ we define its truncated distribution on $[0,1]$ as the distribution with pdf $g(z) = \mathbbm{1}\left(0 \leq z \leq 1  \right) \cdot \frac{f(z)}{F(1)-F(0)}$.} normal, truncated exponential and uniform.

\begin{figure}[h!]
\centering
\subfigure[Truncated Normal $\mu = 0.6$, $\sigma = 0.3$.]{
\begin{tikzpicture}[scale=.7]
\begin{axis}[
            title={},
            xmin=0,xmax=1000,
            ymin=0.03,ymax=0.04,
            scaled y ticks={base 10:2},
            width=10cm,
            height=8cm,
            table/col sep=comma,
            xlabel = number of samples $\numS$,
            ylabel = regret,
            grid=both,
            skip coords between index={0}{1},
            legend pos=north east]
    
%\addplot [blue,mark=square,mark options={scale=.1}] table[x={N},y={mean-eps=0.1}] {Data/res_norm_mu06_sigma03.csv};
%    \addlegendentry{eps 0.1}
     \addplot [blue,mark=square,mark options={scale=.01}] table[x={N},y={min-eps=0.3}] {Data/res_norm_mu06_sigma03.csv};
    \addlegendentry{Cumulative Minimum}                   
     \addplot [red,mark=square,mark options={scale=.01}, very thin] table[x={N},y={mean-eps=0.3}] {Data/res_norm_mu06_sigma03.csv};
    \addlegendentry{ERM}                   
\end{axis}

\end{tikzpicture}
}
\subfigure[Truncated Exponential $\lambda = 3$]{
\begin{tikzpicture}[scale=.7]
\begin{axis}[
            title={},
            xmin=0,xmax=1000,
            ymin=0.045,ymax=0.055,
            scaled y ticks={base 10:2},
            width=10cm,
            height=8cm,
            table/col sep=comma,
            xlabel = number of samples $\numS$,
            ylabel = regret,
            grid=both,
            skip coords between index={0}{1},
            legend pos=north east]
    
%\addplot [blue,mark=square,mark options={scale=.1}] table[x={N},y={mean-eps=0.1}] {Data/res_norm_mu06_sigma03.csv};
%    \addlegendentry{eps 0.1}
     \addplot [blue,mark=square,mark options={scale=.01}] table[x={N},y={min-eps=0.3}] {Data/res_exp_lambda3.csv};
    \addlegendentry{Cumulative Minimum}                   
     \addplot [red,mark=square,mark options={scale=.01}, very thin] table[x={N},y={mean-eps=0.3}] {Data/res_exp_lambda3.csv};
    \addlegendentry{ERM}                   
\end{axis}
\end{tikzpicture}
}
\subfigure[Uniform]{
\begin{tikzpicture}[scale=.7]
\begin{axis}[
            title={},
            xmin=0,xmax=1000,
            ymin=0.044,ymax=0.054,
            scaled y ticks={base 10:2},
            width=10cm,
            height=8cm,
            table/col sep=comma,
            xlabel = number of samples $\numS$,
            ylabel = regret,
            grid=both,
            skip coords between index={0}{1},
            legend pos=north east]
    
%\addplot [blue,mark=square,mark options={scale=.1}] table[x={N},y={mean-eps=0.1}] {Data/res_norm_mu06_sigma03.csv};
%    \addlegendentry{eps 0.1}
     \addplot [blue,mark=square,mark options={scale=.01}] table[x={N},y={min-eps=0.3}] {Data/res_uniform.csv};
    \addlegendentry{Cumulative Minimum}                   
     \addplot [red,mark=square,mark options={scale=.01},very thin ] table[x={N},y={mean-eps=0.3}] {Data/res_uniform.csv};
    \addlegendentry{ERM}                   
\end{axis}
\end{tikzpicture}
}
\caption{\textbf{Performance of ERM for different instances $(\zeta = 0.3)$.} Each figure depicts the regret of ERM for a fixed out-of-sample distribution as a function of the number of samples. The cumulative minimum curve corresponds to the lowest regret achieved by ERM by using at most $n$ samples $(q = 0.9)$.}
\label{fig:mild_instances}
\end{figure}

We remark that most of the insights derived through the worst-case analysis are still widely applicable when the instance is \textit{fixed} across sample sizes and the demand distribution is not a Bernoulli distribution. In particular, the regret of ERM still decays sharply after tens of samples (see \Cref{fig:mild_different_scaling}   for a scaling which highlights more this behavior) and the  regret curve exhibits the local non-monotonicity across all distributions considered. We note that, while the global non-monotonicity still happens for the truncated normal and the truncated exponential distributions, it is much less marked than the one observed with Bernoulli distributions. The effective sample size is much larger than the one suggested by the worst-case analysis and the excess loss incurred by a decision-maker who uses all samples as opposed to the effective one is much smaller for these mild distributions that it is for the Bernoulli distribution. Finally and unsurprisingly, we remark that the regret achieved by ERM for these distributions is lower than suggested by the worst-case analysis.
\begin{figure}[h!]
\centering
\subfigure[Truncated Normal $\mu = 0.6$, $\sigma = 0.3$.]{
\begin{tikzpicture}[scale=.7]
\begin{axis}[
            title={},
            xmin=0,xmax=200,
            ymin=0,ymax=0.25,
            scaled y ticks={base 10:2},
            width=10cm,
            height=8cm,
            table/col sep=comma,
            xlabel = number of samples $\numS$,
            ylabel = regret,
            grid=both,
            %skip coords between index={0}{1},
            legend pos=north east]
    
%\addplot [blue,mark=square,mark options={scale=.1}] table[x={N},y={mean-eps=0.1}] {Data/res_norm_mu06_sigma03.csv};
%    \addlegendentry{eps 0.1}
%     \addplot [blue,mark=square,mark options={scale=.01}] table[x={N},y={min-eps=0.3}] {Data/res_norm_mu06_sigma03.csv};
%    \addlegendentry{Cumulative Minimum}                   
     \addplot [red,mark=square,mark options={scale=.2}, thick] table[x={N},y={mean-eps=0.3}] {Data/res_norm_mu06_sigma03.csv};
    \addlegendentry{ERM}                   
\end{axis}

\end{tikzpicture}
}
\subfigure[Truncated Exponential $\lambda = 3$]{
\begin{tikzpicture}[scale=.7]
\begin{axis}[
            title={},
            xmin=0,xmax=200,
            ymin=0,ymax=0.25,
            scaled y ticks={base 10:2},
            width=10cm,
            height=8cm,
            table/col sep=comma,
            xlabel = number of samples $\numS$,
            ylabel = regret,
            grid=both,
            %skip coords between index={0}{1},
            legend pos=north east]
    
%\addplot [blue,mark=square,mark options={scale=.1}] table[x={N},y={mean-eps=0.1}] {Data/res_norm_mu06_sigma03.csv};
%    \addlegendentry{eps 0.1}
%     \addplot [blue,mark=square,mark options={scale=.01}] table[x={N},y={min-eps=0.3}] {Data/res_exp_lambda3.csv};
%    \addlegendentry{Cumulative Minimum}                   
     \addplot [red,mark=square,mark options={scale=.2}, thick] table[x={N},y={mean-eps=0.3}] {Data/res_exp_lambda3.csv};
    \addlegendentry{ERM}                   
\end{axis}
\end{tikzpicture}
}
\subfigure[Uniform]{
\begin{tikzpicture}[scale=.7]
\begin{axis}[
            title={},
            xmin=0,xmax=200,
            ymin=0,ymax=0.25,
            scaled y ticks={base 10:2},
            width=10cm,
            height=8cm,
            table/col sep=comma,
            xlabel = number of samples $\numS$,
            ylabel = regret,
            grid=both,
            %skip coords between index={0}{1},
            legend pos=north east]
    
%\addplot [blue,mark=square,mark options={scale=.1}] table[x={N},y={mean-eps=0.1}] {Data/res_norm_mu06_sigma03.csv};
%    \addlegendentry{eps 0.1}
%     \addplot [blue,mark=square,mark options={scale=.01}] table[x={N},y={min-eps=0.3}] {Data/res_uniform.csv};
%    \addlegendentry{Cumulative Minimum}                   
     \addplot [red,mark=square,mark options={scale=.2}, thick] table[x={N},y={mean-eps=0.3}] {Data/res_uniform.csv};
    \addlegendentry{ERM}                   
\end{axis}
\end{tikzpicture}
}
\caption{\textbf{Performance of ERM for different instances $(\zeta = 0.3)$.} Each figure depicts the regret of ERM for a fixed out-of-sample distribution as a function of the number of samples. The cumulative minimum curve corresponds to the lowest regret achieved by ERM by using at most $n$ samples $(q = 0.9)$}
\label{fig:mild_different_scaling}
\end{figure}

\subsection{Drifting environment}
\label{sec:beyond_worstcase_drfit}
We now explore the performance of weighted ERM policies for settings where the demand distribution is drifting over time as in \Cref{sec:many_dissimilarities} but when the distributions are not necessarily Bernoulli distributions. Given a drift parameter $\Delta$ and an out-of-sample distribution $F_0$, we define the sequence of historical distributions $(F_i)_{i \in \{1,\ldots \numS\}}$ such that for every $i \in \{1,\ldots, \numS\}$ and every $y \in [0,1]$,   $F_i(y) = \min(F_0(y) + i \cdot \Delta, 1)$. 

In \Cref{tab:mild_instance_misspecification} we report the performance of $k$-NN-ERM (formally defined in \Cref{sec:many_dissimilarities}) for various values of the $k$ parameter and the drift $\Delta$.

\begin{table}[h!]
\centering
\begin{tabular}{cccc|c}
  & \multicolumn{4}{c}{expected regret}   \\ 
 \cline{2-5}
  % & & best & $k^*$ & best  \\
 $\Delta$ & $k = 19$ & $k = 29 $ & $k = 59$ & Robust $k$ \\
\hline
\hline
$0.0010$&   $1.8 \cdot 10^{-3}$ & $1.3 \cdot 10^{-3}$   & $\mathbf{8.9 \cdot 10^{-4}}$ & $\mathit{1.5 \cdot 10^{-3}}$ ($k=27$)  \\
\hline
$0.0025$ &   $1.8 \cdot 10^{-3}$ &  $\mathbf{1.6 \cdot 10^{-3}}$  &  $2.7 \cdot 10^{-3}$ & $\mathit{2.3 \cdot 10^{-3}}$ ($k=17$) \\
\hline
$0.0050$ & $\mathbf{3.1 \cdot 10^{-3}}$ &   $3.8 \cdot 10^{-3}$ & $9.1 \cdot 10^{-3}$ & $\mathit{4.7 \cdot 10^{-3}}$ ($k=8$) \\
\end{tabular}
\caption{\textbf{Average regret of $k$-NN-ERM.} The table reports the average regret for different values of $k$ and $\Delta$ when $F_0$ is a truncated normal with parameter $\mu = 0.6$ and $\sigma = 0.3$. The bolded numbers correspond to the best average regret across all values of $k \in \{1,\ldots,100\}$ and the italicized numbers corresponds to the regret incurred by using the $k$ prescribed by our worst-case analysis (see \Cref{sec:many_dissimilarities}) ($\numS = 100$ and $q=0.9$).}
\label{tab:mild_instance_misspecification}
\end{table}

We first note that in contrast with the results in \Cref{tab:misspecification} the decision-maker cannot compute a priori the values in \Cref{tab:mild_instance_misspecification} as this requires to know the out-of-sample distribution $F_0$ (recall that, \Cref{tab:misspecification} was computed using worst-case distributions which may not be the true ones). Therefore, the bolded number is the best regret achievable with $k$-NN-ERM in the idealized scenario where the decision-maker is able to compute these values. 

The key takeaways from \Cref{tab:mild_instance_misspecification} is that  $i)$ the  $k$ prescribed by our worst-case provides the right order of magnitude for the number of historical samples that should be used with drifting distributions and $ii)$ the regret achieved by our worst-case $k$ allows to achieve a performance relatively close to the one of the idealized scenario. In fact, not knowing the shape of the demand distribution implies a performance deterioration of the same magnitude as using a misspecified value of $\Delta$. We see for instance that when $\Delta = 0.001$, the performance of our robust choice of $k =17$ yields a regret of $1.5 \cdot 10^{-3}$ as opposed to the ideal regret of $8.9 \cdot 10^{-4}$ when knowing all the distributions and committing to $k$-NN-ERM. However, we also remark that even when the decision-maker knows the shape and parameters of the out-of-sample distribution but wrongfully believe that $\Delta  = 0.0025$ they would incur a regret of $1.3 \cdot 10^{-3}$ by using $k=29$. Consequently, our exact characterization of the worst-case performance provides a relatively robust way to derive algorithmic insights about the choice of weights even in settings where the demand distributions are not Bernoulli distributions.

Finally, we would like to mention that the value of the achieved regret is much lower than the one suggested in the worst-case which leaves open the interesting question of deriving a characterization of the worst-case performance of WERM policies for a subclass of distributions.

\section{Extension to the Wasserstein distance case}
\label{sec:apx_Wasserstein}

In this section, we provide results showing how our methodology can be applied without leading to an exact characterization. We consider the setting in which the distance between distribution in the local condition (\Cref{def:local_const}) is measured with respect to the Wasserstein distance defined for any $F,H \in \Delta(\Y)$ as
\begin{equation}
\label{eq:W_def}
\|F - H \|_{W} = \int_{0}^1 |F(y) - H(y)| dy. 
\end{equation}

Our next result shows that by using a Lagrangian relaxation, one can bound the worst-case regret for any separable policy by finite dimensional optimization problem.

\begin{theorem}
\label{thm:upper_bound_bern_Wasserstein}
In what follows, assume that the distance between distributions is measured with respect to the Wasserstein distance.
For every $\numS \geq 1$, any sequence of contexts $\mathbf{x} = (x_i)_{i \in \{0,\ldots,\numS\}} \in \X^{\numS+1}$  and any separable policy $\pi$ we have,
\begin{align*}
\sup_{\mesOut \in \Delta(\Y)} \sup_{\substack{\mesIn{1},\ldots, \mesIn{\numS} \in \Delta(\Y)\\ \|\mesOut - \mesIn{i} \|_{W} \leq d(x_0,x_i) \, \forall i}} \mathbb{E}_{ \mathbf{y} \sim \mesIn{1} \times \ldots \times \mesIn{\numS}} \big[R(\pi(\mathbf{y}),\mesOut) \big]  
&\leq \inf_{\substack{\bm{\lambda} \in \mathbb{R}_{+}^{\numS}}} \sup_{\mu_0,\ldots,\mu_\numS \in [0,1]}  \mathbb{E}_{ \mathbf{y} \sim \mathcal{B}(\mu_1) \times \ldots \times \mathcal{B}(\mu_{\numS})} \big[R(\pi(\mathbf{x},\mathbf{y}),\mathcal{B}(\mu_{0})) \big] \\
&\quad+ \sum_{i=1}^\numS \lambda_i  \left( d(x_0,x_i) -  |\mu_0 -  \mu_i | \right) .
\end{align*}
\end{theorem}
The proof of this result is presented in \Cref{sec:apx_proof_Wasserstein}.

We note that \Cref{thm:upper_bound_bern_Wasserstein} allows to bound the initial infinite dimensional optimization problem by a $2\numS+1$ dimensional min max optimization problem. While this new problem can still be computationally challenging in general, we next illustrate how it can be used to derive upper bounds for certain instances. Our next result simplifies the inner maximization problem.
\begin{proposition}
\label{prop:ERM_Wasserstein_inner}
For any $\numS \geq 1$, any $\bm{\lambda} \in \mathbb{R}_{+}^\numS$ any $\bm{x} \in \X^{\numS+1}$ and any counting policy $\pi$, we have that,
\begin{equation*}
\sup_{(\mu_1,\ldots,\mu_\numS) \in [0,1]^\numS} \mathcal{L}(\bm{\mu},\bm{\lambda}) = \begin{cases}
\sup_{(\mu_1,\ldots,\mu_\numS) \in \{\mu_0,1\}^\numS} \mathcal{L}(\bm{\mu},\bm{\lambda}) \quad \text{if $\mu_0 \leq 1-q$}\\
\sup_{(\mu_1,\ldots,\mu_\numS) \in \{0,\mu_0\}^\numS} \mathcal{L}(\bm{\mu},\bm{\lambda}) \quad \text{o.w.}
\end{cases}
\end{equation*}
where
$\mathcal{L}(\bm{\mu},\bm{\lambda}) = \mathbb{E}_{ \mathbf{y} \sim \mathcal{B}(\mu_1) \times \ldots \times \mathcal{B}(\mu_{\numS})} \big[R(\pi(\mathbf{x},\mathbf{y}),\mathcal{B}(\mu_{0})) \big] + \sum_{i=1}^\numS \lambda_i  \left( d(x_0,x_i) -  |\mu_0 -  \mu_i | \right).$
\end{proposition}
The proof of this result is presented in \Cref{sec:apx_proof_Wasserstein}.

We next evaluate the worst-case performance of ERM when the local condition uses the Wasserstein distance by applying \Cref{thm:upper_bound_bern_Wasserstein} and \Cref{prop:ERM_Wasserstein_inner}. To illustrate this bound, we consider the setting studied in \Cref{sec:mmBound} where $d(x_0,x_i) = \zeta$ for all $i \in \{1,\ldots,\numS\}$. Substituting $\pi=\ERM$ and $d(x_0,x_i)=\zeta$ into the definition of $\mathcal{L}(\bm{\mu},\bm{\lambda})$, we have 
\begin{equation*}
\mathcal{L}(\bm{\mu},\bm{\lambda}) = \mathbb{E}_{ \mathbf{y} \sim \mathcal{B}(\mu_1) \times \ldots \times \mathcal{B}(\mu_{\numS})} \big[R(\ERM(\mathbf{x},\mathbf{y}),\mathcal{B}(\mu_{0})) \big] + \sum_{i=1}^\numS \lambda_i  \left( \zeta -  |\mu_0 -  \mu_i | \right).
\end{equation*}
By combining \Cref{thm:upper_bound_bern_Wasserstein} and \Cref{prop:ERM_Wasserstein_inner} we have established that,
\begin{align*}
\sup_{\mesOut \in \Delta(\Y)} &\sup_{\substack{\mesIn{1},\ldots, \mesIn{\numS} \in \Delta(\Y)\\ \|\mesOut - \mesIn{i} \|_{W} \leq d(x_0,x_i) \, \forall i}}  \mathbb{E}_{ \mathbf{y} \sim \mesIn{1} \times \ldots \times \mesIn{\numS}} \big[R(\ERM(\mathbf{y}),\mesOut) \big]  \\
&\qquad \leq \inf_{\substack{\bm{\lambda} \in \mathbb{R}_{+}^{\numS}}} \max \left( \sup_{\mu_0 \in [0,1-q]} \sup_{(\mu_1,\ldots,\mu_\numS) \in \{\mu_0,1\}^\numS} \mathcal{L}(\bm{\mu},\bm{\lambda}) , \sup_{\mu_0 \in [1-q,1]}  \sup_{(\mu_1,\ldots,\mu_\numS) \in \{0,\mu_0\}^\numS} \mathcal{L}(\bm{\mu},\bm{\lambda})  \right)\\
&\qquad \leq \inf_{\lambda \in \mathbb{R}_{+}} \max \left( \sup_{\mu_0 \in [0,1-q]} \sup_{(\mu_1,\ldots,\mu_\numS) \in \{\mu_0,1\}^\numS} \mathcal{L}(\bm{\mu},\lambda \cdot \bm{1}) , \sup_{\mu_0 \in [1-q,1]}  \sup_{(\mu_1,\ldots,\mu_\numS) \in \{0,\mu_0\}^\numS} \mathcal{L}(\bm{\mu},\lambda \cdot \bm{1})  \right),
\end{align*}
where $\bm{1}$ is an $\numS$-dimensional vector where all components are $1$.

Furthermore, given that the ERM policy is symmetric in the samples one can note that $\mathcal{L}(\bm{\mu},\lambda \cdot \bm{1}) $ is symmetric in $\bm{\mu}$. This allows to show that for every $\{a,b\} \in [0,1]$, 
\begin{equation*}
 \sup_{(\mu_1,\ldots,\mu_\numS) \in \{a,b\}^\numS} \mathcal{L}(\bm{\mu},\lambda \cdot \bm{1}) =  \sup_{(\mu_1,\ldots,\mu_\numS) \in M(a,b)} \mathcal{L}(\bm{\mu},\lambda \cdot \bm{1}),
\end{equation*}
where $M(a,b) = \{ \sum_{j=1}^i a \cdot \bm{e_j} + \sum_{j={i+1}}^\numS b \cdot \bm{e_j} \vert \text{for all $i \in \{0,\ldots,\numS\}$}\}.$ and $\bm{e_i}$ is the $i^{th}$ canonical vector in dimension $\numS$ whose coordinates are equal to $0$ except the $i^{th}$ one which is equal to $1$.

We have thus established that,
\begin{align*}	
\sup_{\mesOut \in \Delta(\Y)} & \sup_{\substack{\mesIn{1},\ldots, \mesIn{\numS} \in \Delta(\Y)\\ \|\mesOut - \mesIn{i} \|_{W} \leq d(x_0,x_i) \, \forall i}}  \mathbb{E}_{ \mathbf{y} \sim \mesIn{1} \times \ldots \times \mesIn{\numS}} \big[R(\ERM(\mathbf{y}),\mesOut) \big]  \\
&\qquad \leq \inf_{\lambda \in \mathbb{R}_{+}} \max \left( \sup_{\mu_0 \in [0,1-q]} \sup_{(\mu_1,\ldots,\mu_\numS) \in M(\mu_0,1)} \mathcal{L}(\bm{\mu},\lambda \cdot \bm{1}) , \sup_{\mu_0 \in [1-q,1]}  \sup_{(\mu_1,\ldots,\mu_\numS) \in M(0,\mu_0)} \mathcal{L}(\bm{\mu},\lambda \cdot \bm{1})  \right).
\end{align*}
This implies that that the worst-case performance of ERM can be upper bounded by evaluating $\numS$ functions on a two-dimensional grid.

We present in \Cref{fig:SAA_Wasserstein} the upper bound on the worst-case regret of ERM under the Wasserstein local condition that we obtain using our Lagrangian relaxation. We also show a lower bound implied by the worst-case regret under the Kolmogorov distance (see Remark 1 in \cite{besbes2022beyond}). Finally, we add for reference the evaluation of the bound \cite{mohri2012new} for Bernoulli distributions satisfying the local condition both for the Kolmogorov and the Wasserstein distance.
\begin{figure}[h]
\centering
\begin{tikzpicture}[scale=1][h]
\begin{axis}[
            title={},
            xmin=0,xmax=70,
            ymin=0.0,ymax=1,
            width=10cm,
            height=8cm,
            table/col sep=comma,
            xlabel = number of samples $\numS$,
            ylabel = worst-case regret,
            grid=both,
            skip coords between index={0}{1},
            legend pos=north east]
    \addplot [gray,thick,mark=square,mark options={scale=.4}] table[x={Ns},y={eps = 0.1}] {Data/mohri_bound_correct_eps_e-1_q9.csv};
    \addlegendentry{Upper bound (M \&MM12)}
\addplot [blue,mark=square,mark options={scale=.4}] table[x={Ns},y={epsilon=0.1}] {Data/paper_all_data_SAA_q9.csv};
    \addlegendentry{Lower bound}
     \addplot [red,mark=square,mark options={scale=.4}] table[x={N},y={upper_bound}] {Data/lagrangian_relaxation_bound_eps1e-1.csv};
    \addlegendentry{Upper bound (Our analysis)}                    
\end{axis}
\end{tikzpicture}
\caption{\textbf{Bounds on the worst-case regret of ERM under the Wasserstein distance}. The figure depicts bounds on the worst-case regret of the Empirical Risk Minimization for the Wasserstein distance for $\zeta = .1$ as a function of the sample size $\numS$. The lower bound corresponds to the worst-case regret under the Kolmogorov distance (see Remark 1 in \cite{besbes2022beyond}) (q=.9).}\label{fig:SAA_Wasserstein}
\end{figure}
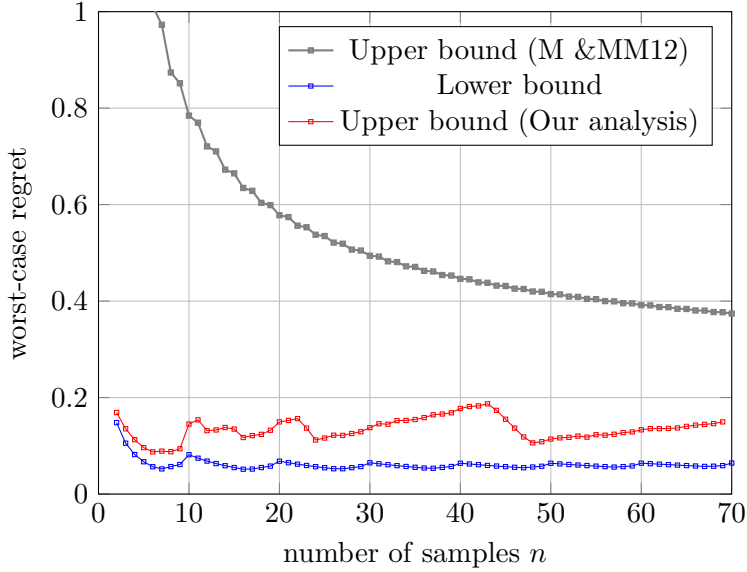

\Cref{fig:SAA_Wasserstein} shows that our Lagrangian relaxation approach provides much tighter bounds than the concentration-based ones derived in the literature. This highlights that our optimization-based approach can lead to a meaningful characterization of the worst-case regret of central policies even without resorting to an exact characterization of the worst-case distribution.

\subsection{Proofs}
\label{sec:apx_proof_Wasserstein}

\begin{proof}[\textbf{Proof of \Cref{thm:upper_bound_bern_Wasserstein}}]
For any $\mesOut, \mesIn{1},\ldots,\mesIn{\numS} \in \Delta(\Y)$ and for every $\bm{\lambda} \in \mathbb{R}_{+}^\numS$, we define the Lagrangian operator,
\begin{equation*}
\mathcal{L}( \mesOut, \bm{H}, \bm{\lambda}) = \mathbb{E}_{ \mathbf{y} \sim \mesIn{1} \times \ldots \times \mesIn{\numS}} \big[R(\pi(\mathbf{x},\mathbf{y}),\mesOut) \big] + \sum_{i=1}^\numS \lambda_i  \left( d(x_0,x_i) -  \|\mesOut -  \mesIn{i} \|_{W} \right), 
\end{equation*}
where $\bm{H} = (\mesIn{i})_{i \in \{1,\ldots,\numS\}}$. 
We note that,
\begin{align*}
\sup_{\mesOut \in \Delta(\Y)} \sup_{\substack{\mesIn{1},\ldots, \mesIn{\numS} \in \Delta(\Y)\\ \|\mesOut - \mesIn{i} \|_{W} \leq d(x_0,x_i) \, \forall i}} \mathbb{E}_{ \mathbf{y} \sim \mesIn{1} \times \ldots \times \mesIn{\numS}} \big[R(\pi(\mathbf{y}),\mesOut) \big]
&= \sup_{\mesOut, \mesIn{1},\ldots, \mesIn{\numS}  \in \Delta(\Y)}  \inf_{\bm{\lambda} \in\mathbb{R}_{+}^{\numS}} \mathcal{L}( \mesOut, \bm{H}, \bm{\lambda}) \\
&\stackrel{(a)}{\leq} \inf_{\bm{\lambda} \in\mathbb{R}_{+}^{\numS}} \sup_{\mesOut, \mesIn{1},\ldots, \mesIn{\numS}  \in \Delta(\Y)} \mathcal{L}( \mesOut, \bm{H}, \bm{\lambda}) 
\end{align*}
where $(a)$ holds by weak duality.

We next show that, for every $\bm{\lambda} \in \mathbb{R}_{+}^\numS$, we have that,
\begin{equation}
\label{eq:reduction_Bernoulli_W}
\sup_{\mesOut, \mesIn{1}, \ldots, \mesIn{\numS} \in \Delta(\Y)} \mathcal{L}( \mesOut, \bm{H}, \bm{\lambda}) = \sup_{\mu_0,\ldots,\mu_{\numS} \in [0,1]} \mathcal{L}( \mathcal{B}(\mu_0), \ldots,  \mathcal{B}(\mu_\numS), \bm{\lambda}).
\end{equation}
Therefore, by applying \Cref{lem:cdf_integral_reduction} and using \eqref{eq:W_def} we obtain that the Lagrangian operator can be rewritten for every $\mesOut, \mesIn{1}, \ldots, \mesIn{\numS} \in \Delta(\Y)$, as
\begin{equation}
\label{eq:integral_W}
\mathcal{L}( \mesOut, \bm{H}, \bm{\lambda}) = \int_0^1 \Psi_{W}^\pi(\mesOut(y),\mesIn{1}(y), \ldots, \mesIn{\numS}(y)) dy,
\end{equation}
where for any $(z_0,\ldots,z_\numS) \in [0,1]^{\numS+1}$, 
\begin{align*}
\Psi_W^{\pi}(z_0,\ldots,z_\numS) = (\cu + \co) \cdot \left[ P^{\pi} (z_1,\ldots,z_\numS)  \cdot(q-z_0)  + \max \{ z_0 -q,0 \} \right] + \sum_{i=1} \lambda_i \cdot (d(x_0,x_i) - |z_0 - z_i|),
\end{align*}
and $P^\pi$ is the function associated to the separable policy $\pi$. 

We note that \eqref{eq:integral_W} implies that,
\begin{align*}
\sup_{\mesOut, \mesIn{1}, \ldots, \mesIn{\numS} \in \Delta(\Y)} \mathcal{L}( \mesOut, \bm{H}, \bm{\lambda}) 
&= \sup_{\mesOut, \mesIn{1}, \ldots, \mesIn{\numS} \in \Delta(\Y)}  \int_0^1 \Psi_{W}^\pi(\mesOut(y),\mesIn{1}(y), \ldots, \mesIn{\numS}(y)) dy\\
&\leq \int_0^1 \sup_{(\alpha_0,\ldots,\alpha_\numS) \in [0,1]^{\numS+1}} \Psi_{W}^\pi(\alpha_0,\ldots,\alpha_\numS) dy\\
&= \sup_{(\alpha_0,\ldots,\alpha_\numS) \in [0,1]^{\numS+1}} \Psi_{W}^\pi(\alpha_0,\ldots,\alpha_\numS) = \sup_{\mu_0,\ldots,\mu_{\numS} \in [0,1]} \mathcal{L}( \mathcal{B}(\mu_0), \ldots,  \mathcal{B}(\mu_\numS), \bm{\lambda}),
\end{align*}
where the last equality follows from \eqref{eq:psi_bern}. 

We note that the reverse equality holds as the set of Bernoulli distributions is included in the set of all distributions. This implies \eqref{eq:reduction_Bernoulli_W}. We finally conclude the proof by taking a supremum over $\bm{\lambda}$.

\end{proof}

\begin{proof}[\textbf{Proof of \Cref{prop:ERM_Wasserstein_inner}}]
For every $\bm{\lambda} \in \mathbb{R}_{+}^\numS$ and $\bm{\mu} = (\mu_i)_{i \in \{0,\ldots,\numS\}} \in [0,1]^\numS$, recall that we defined
\begin{equation*}
\mathcal{L}(\bm{\mu},\bm{\lambda}) = \mathbb{E}_{ \mathbf{y} \sim \mathcal{B}(\mu_1) \times \ldots \times \mathcal{B}(\mu_{\numS})} \big[R(\pi(\mathbf{x},\mathbf{y}),\mathcal{B}(\mu_{0})) \big] + \sum_{i=1}^\numS \lambda_i  \left( d(x_0,x_i) -  |\mu_0 -  \mu_i | \right).
\end{equation*}

Fix $i \in \{1,\ldots,\numS\}$ and $\bm{\lambda} \in \mathbb{R}_{+}^\numS$. For every $(\mu_0,\ldots,\mu_{i-1},\mu_{i+1}, \ldots, \mu_{\numS})$ define the function,
\begin{equation*}
\mathcal{L}_i : \begin{cases}
[0,1] \to \mathbb{R}\\
\mu \mapsto \mathcal{L}((\mu_0,\ldots,\mu_{i-1}, \mu, \mu_{i+1}, \ldots, \mu_{\numS}); \bm{\lambda}).
\end{cases}
\end{equation*}
\textit{Step 1:} We next show that,
\begin{equation}
\label{eq:optimization_single_mu}
\sup_{\mu \in [0,1]}  \mathcal{L}_{i}(\mu) = \begin{cases}
\max(\mathcal{L}_{i}(\mu_0), \mathcal{L}_{i}(1)) \quad \text{when $\mu_0 \leq 1-q$}\\
\max(\mathcal{L}_{i}(0), \mathcal{L}_{i}(\mu_0)) \quad \text{o.w.}
\end{cases}
\end{equation}
 
Let,
\begin{equation*}
\rho_i: \mu \mapsto
\mathbb{E}_{ \mathbf{y} \sim \mathcal{B}\left( \mu_1\right) \times \ldots \times \mathcal{B}\left( \mu_{i-1}\right) \times \mathcal{B}\left( \mu\right) \times \mathcal{B}\left( \mu_{i+1}\right) \times \ldots \times \mathcal{B}\left( \mu_{\numS}\right)}  \big[R(\pi(\mathbf{y}),\mathcal{B}\left( \mu_0\right)) \big].
\end{equation*}
Then, $\mathcal{L}_i(\mu) = \rho_i(\mu) - \lambda_i \cdot | \mu_0 - \mu| + C_i$, where $C_i$ is a constant independent of $\mu$.
 
First, assume that $\mu_0 \leq 1-q$.
In the proof of \Cref{prop:worst_history} we show that $\rho_i$ is non-decreasing. Furthermore, one can see that for any counting policy $\pi$, $\rho_i$ is in fact an affine function of $\mu$. Therefore, there exists $\alpha \geq 0$ and $C \in \mathbb{R}$ such that $\rho_{i}(\mu) = \alpha \cdot \mu + C$ for every $\mu \in [0,1]$. This implies that,
\begin{equation*}
\mathcal{L}_i(\mu) = \begin{cases}
\left( \alpha + \lambda_i \right)  + C'_i  \quad \text{if $\mu \leq \mu_0$}\\
\left( \alpha - \lambda_i \right)  + C''_i  \quad \text{if $\mu > \mu_0$}.
\end{cases}
\end{equation*}
Hence,
\begin{equation*}
\sup_{\mu \in [0,1]} \mathcal{L}_i(\mu)  \stackrel{(a)}{=} \sup_{\mu \in [\mu_0,1]} \mathcal{L}_i(\mu) \stackrel{(b)}{=} \max(\mathcal{L}_i(\mu_0),\mathcal{L}_i(1)), 
\end{equation*}
where $(a)$ holds because $\alpha$ and $\lambda_i$ are non-negative and $(b)$ holds because $\mathcal{L}_i$ is affines on $[\mu_0,1]$. 

A similar argument can be used when $\mu > 1-q$ to prove that \eqref{eq:optimization_single_mu} holds.
 
\textit{Step 2:} By successively applying \Cref{eq:optimization_single_mu} we obtain that,
\begin{equation*}
\sup_{(\mu_1,\ldots,\mu_\numS) \in [0,1]^\numS} \mathcal{L}(\bm{\mu},\bm{\lambda}) = \begin{cases}
\sup_{(\mu_1,\ldots,\mu_\numS) \in \{\mu_0,1\}^\numS} \mathcal{L}(\bm{\mu},\bm{\lambda}) \quad \text{if $\mu_0 \leq 1-q$}\\
\sup_{(\mu_1,\ldots,\mu_\numS) \in \{0,\mu_0\}^\numS} \mathcal{L}(\bm{\mu},\bm{\lambda}) \quad \text{o.w.}
\end{cases}
\end{equation*}

\end{proof}

\section{Estimation of the dissimilarity: Illustrative Examples}\label{sec:apx_estimation}

In this section, we provide a brief proof of concept of a method that could be used to estimate the dissimilarity in practice. We note that tackling this question in depth, which is beyond the scope of the current work, is itself a very interesting avenue for future research.

\subsection{Single dissimilarity estimation}\label{sec:estimate_single_diss}
Consider the simplified scenario in which a decision-maker is selling a new white T-shirt, and let  $x_{0}$ denote the feature vector of this product. Suppose the seller has historical data available from a similar black T-shirt with feature vector $x_{1}$. Furthermore, assume the seller has historically sold a white shirt with feature $x_{0}'$ and a similar black shirt with feature $x_{1}'$. The seller may reasonably assume that $d(x_{0}, x_{1}) = d(x_{0}', x_{1}')$, implying that the dissimilarity between white and black T-shirts equals the dissimilarity between white and black shirts. Consequently, the seller can leverage historical shirt sales data to estimate the dissimilarity for the T-shirts. We emphasize in this simplified example that the seller does not use shirt sales data directly for T-shirt inventory decisions, implicitly assuming the dissimilarity between shirts and T-shirts is large, while the effect of color remains consistent across product categories. That is, $x_1$ is part of our historical data points to use but $x'_0,x'_1$ are not---$x'_0,x'_1$ are only used to help estimate the heterogeneity between $x_1$ and the new product of interest.

Formally, let $m$ be the number of past samples observed for each shirt. and let $\bm{Y}^{\mathrm{white}}, \bm{Y}^{\mathrm{black}} \in \Y^{m}$, be the vectors of sample for the white and black shirts. We assume these vectors are independently and identically distributed (i.i.d.) samples drawn from their respective demand distributions. Given these assumptions, we define the following estimator for $d(x_{0},x_{1})$,
\begin{equation*}
\hat{d}_{m}(\bm{Y}^{\mathrm{white}}, \bm{Y}^{\mathrm{black}}) = \| \hat{F}_{\mathrm{white}} - \hat{F}_{\mathrm{black}} \|_{\infty},
\end{equation*}
where $\hat{F}_{\mathrm{white}}$ (resp. $\hat{F}_{\mathrm{black}}$) is the empirical distribution of the samples $\bm{Y}^{\mathrm{white}}$ (resp. $\bm{Y}^{\mathrm{black}}$).

We next numerically illustrate the dissimilarity estimates obtained as a function of the sample size $m$. 
For each instance, we take a base distribution $F_{\mathrm{white}}$ (examples shown in \Cref{fig:estimate_zeta}) and define its shifted counterpart
$F_{\mathrm{black}}(y) = \min\!\big(F_{\mathrm{white}}(y)+\zeta,\,1\big)$, for some $\zeta>0.$
Given independent samples of size $m$ from both distributions, we compute the empirical dissimilarity $\hat{d}_{m}(\bm{Y}^{\mathrm{white}},\bm{Y}^{\mathrm{black}})$ and examine how these estimates vary with $m$ across different distributional instances.

In \Cref{fig:estimate_zeta}, we plot the average dissimilarity together with the 95\% empirical region (between the 2.5th and 97.5th percentiles) of the estimator, computed over $K = 1000$ independent instances, for three distributions supported on $[0,1]$: truncated normal, truncated exponential, and uniform.

\begin{figure}[h!]
\centering
\subfigure[Truncated Normal $\mu=0.6,\ \sigma=0.3$]{
\begin{tikzpicture}
\begin{axis}[
  scale only axis=true,
  width=6cm, height=4.5cm,
  legend pos=north east,
  grid=both,
  tick label style={font=\small},
  label style={font=\small},
  legend style={font=\small, fill=none, draw=none},
  scaled y ticks=false,
  y tick label style={/pgf/number format/fixed},
  xmin=0, xmax=2000,
  ymin=0, ymax=0.3,
  xtick distance=400,
  table/col sep=comma,
  xlabel=number of samples,
  ylabel=dissimilarity,
]

\addplot[name path=Ublue, draw=none, forget plot]
  table[x={n}, y={p97p5_truncnormmu0.6sigma0.3_d0p05}] {Data/kolmogorov_estimates_all_deltas.csv};
\addplot[name path=Lblue, draw=none, forget plot]
  table[x={n}, y={p2p5_truncnormmu0.6sigma0.3_d0p05}] {Data/kolmogorov_estimates_all_deltas.csv};
\addplot[blue, fill=blue, fill opacity=0.12, forget plot]
  fill between[of=Ublue and Lblue];
\addplot [blue, mark=square, mark options={scale=0.3}, line width=0.4mm]
  table[x={n}, y={mean_truncnormmu0.6sigma0.3_d0p05}] {Data/kolmogorov_estimates_all_deltas.csv};
\addlegendentry{$\zeta = 0.05$}

\addplot[name path=Ured, draw=none, forget plot]
  table[x={n}, y={p97p5_truncnormmu0.6sigma0.3_d0p1}] {Data/kolmogorov_estimates_all_deltas.csv};
\addplot[name path=Lred, draw=none, forget plot]
  table[x={n}, y={p2p5_truncnormmu0.6sigma0.3_d0p1}] {Data/kolmogorov_estimates_all_deltas.csv};
\addplot[red, fill=red, fill opacity=0.12, forget plot]
  fill between[of=Ured and Lred];
\addplot [red, mark=square, mark options={scale=0.3}, line width=0.4mm]
  table[x={n}, y={mean_truncnormmu0.6sigma0.3_d0p1}] {Data/kolmogorov_estimates_all_deltas.csv};
\addlegendentry{$\zeta = 0.1$}

\addplot [blue, dashed, domain=0:2000, line width=0.4mm] {0.05};
\addplot [red,  dashed, domain=0:2000, line width=0.4mm] {0.1};

\end{axis}
\end{tikzpicture}
}
\subfigure[Truncated Exponential $\lambda = 3$]{
\begin{tikzpicture}
\begin{axis}[
  scale only axis=true,
  width=6cm, height=4.5cm,
  legend pos=north east,
  grid=both,
  tick label style={font=\small},
  label style={font=\small},
  legend style={font=\small, fill=none, draw=none},
  scaled y ticks=false,
  y tick label style={/pgf/number format/fixed},
  xmin=0, xmax=2000,
  ymin=0, ymax=0.3,
  xtick distance=400,
  table/col sep=comma,
  xlabel=number of samples,
  ylabel=dissimilarity,
]

\addplot[name path=Ublue, draw=none, forget plot]
  table[x={n}, y={p97p5_truncexponlambda3_d0p05}] {Data/kolmogorov_estimates_all_deltas.csv};
\addplot[name path=Lblue, draw=none, forget plot]
  table[x={n}, y={p2p5_truncexponlambda3_d0p05}] {Data/kolmogorov_estimates_all_deltas.csv};
\addplot[blue, fill=blue, fill opacity=0.12, forget plot]
  fill between[of=Ublue and Lblue];
\addplot [blue, mark=square, mark options={scale=0.3}, line width=0.4mm]
  table[x={n}, y={mean_truncexponlambda3_d0p05}] {Data/kolmogorov_estimates_all_deltas.csv};
\addlegendentry{$\zeta = 0.05$}

\addplot[name path=Ured, draw=none, forget plot]
  table[x={n}, y={p97p5_truncexponlambda3_d0p1}] {Data/kolmogorov_estimates_all_deltas.csv};
\addplot[name path=Lred, draw=none, forget plot]
  table[x={n}, y={p2p5_truncexponlambda3_d0p1}] {Data/kolmogorov_estimates_all_deltas.csv};
\addplot[red, fill=red, fill opacity=0.12, forget plot]
  fill between[of=Ured and Lred];
\addplot [red, mark=square, mark options={scale=0.3}, line width=0.4mm]
  table[x={n}, y={mean_truncexponlambda3_d0p1}] {Data/kolmogorov_estimates_all_deltas.csv};
\addlegendentry{$\zeta = 0.1$}

\addplot [blue, dashed, domain=0:2000, line width=0.4mm] {0.05};
\addplot [red,  dashed, domain=0:2000, line width=0.4mm] {0.1};

\end{axis}
\end{tikzpicture}
}
\subfigure[Uniform]{
\begin{tikzpicture}
\begin{axis}[
  scale only axis=true,
  width=6cm, height=4.5cm,
  legend pos=north east,
  grid=both,
  tick label style={font=\small},
  label style={font=\small},
  legend style={font=\small, fill=none, draw=none},
  scaled y ticks=false,
  y tick label style={/pgf/number format/fixed},
  xmin=0, xmax=2000,
  ymin=0, ymax=0.3,
  xtick distance=400,
  table/col sep=comma,
  xlabel=number of samples,
  ylabel=dissimilarity,
]

\addplot[name path=Ublue, draw=none, forget plot] table[x={n}, y={p97p5_uniform01_d0p05}] {Data/kolmogorov_estimates_all_deltas.csv};
\addplot[name path=Lblue, draw=none, forget plot] table[x={n}, y={p2p5_uniform01_d0p05}] {Data/kolmogorov_estimates_all_deltas.csv};
\addplot[blue, fill=blue, fill opacity=0.12, forget plot] fill between[of=Ublue and Lblue];

\addplot [blue, mark=square, mark options={scale=0.3}, line width=0.4mm, skip coords between index={0}{1}] table[x={n}, y={mean_uniform01_d0p05}] {Data/kolmogorov_estimates_all_deltas.csv};
\addlegendentry{$\zeta = 0.05$}

\addplot[name path=Ured, draw=none, forget plot]
    table[x={n}, y={p97p5_uniform01_d0p1}] {Data/kolmogorov_estimates_all_deltas.csv};
\addplot[name path=Lred, draw=none, forget plot]
    table[x={n}, y={p2p5_uniform01_d0p1}] {Data/kolmogorov_estimates_all_deltas.csv};
\addplot[red, fill=red, fill opacity=0.12, forget plot]
    fill between[of=Ured and Lred];
\addplot [red, mark=square, mark options={scale=0.3}, line width=0.4mm,
          skip coords between index={0}{1}]
    table[x={n}, y={mean_uniform01_d0p1}] {Data/kolmogorov_estimates_all_deltas.csv};
\addlegendentry{$\zeta = 0.1$}

\addplot [blue, dashed, domain=0:2000, line width=0.4mm] {0.05};
\addplot [red, dashed, domain=0:2000, line width=0.4mm] {0.1};

\end{axis}
\end{tikzpicture}
}
\caption{\textbf{Estimation of dissimilarity as a function of the sample size.} Solid lines represent the average estimated dissimilarity, shaded regions show the 2.5th--97.5th percentiles across instances (an empirical 95\% interval), and dashed lines represent the corresponding true dissimilarities, for two values of the parameter~$\zeta$.}
\label{fig:estimate_zeta}
\end{figure}

\Cref{fig:estimate_zeta} presents the empirical dissimilarity between two products estimated from finite samples. 
Each plot shows how the estimated dissimilarity approaches the true dissimilarity (dashed lines) as the sample size increases. 
The estimates remain close to the truth even with moderate sample sizes: for example, with $m=400$ samples the estimated dissimilarity is within a factor of $2$ of the true value when $\zeta=0.1$, and within a factor of $3$ when $\zeta=0.05$, in about $95\%$ of the instances. 
We also note that our estimator is positively biased, so the estimated dissimilarity is more conservative than the actual value. 
Hence, regret guarantees derived from plugging in these estimates naturally provide an upper bound on the true performance.

\subsection{Estimating the time drift}\label{sec:drift_estimation}
We next consider a setting with time drift.
Formally, let $m$ be the number of samples observed from each distribution. Consider a base cumulative distribution function (cdf) $F_0$ defined on the interval $[0,1]$. Given a drift parameter $\Delta > 0$, we define a series of shifted distributions as follows: for each integer $i \in {1,2,3}$ and for every $x \in [0,1]$, the shifted cdf is defined by $F_i(x) = \min\left(F_0(x) + i \cdot \Delta, 1\right)$. We investigate numerically the behavior of an estimator for $\Delta$ based on the observed empirical Kolmogorov distances between these distributions.

Specifically, for each distributional instance (truncated normal with $\mu = 0.6, \sigma = 0.3$, truncated exponential with $\lambda = 3$, and uniform on $[0,1]$) and each true value of $\Delta \in \{0.01, 0.05, 0.1\}$, we independently generate $K=5000$ repetitions. In each repetition $k \in \{1,\ldots,K\}$, we draw samples $\bm{Y}^{(0)}_k, \ldots, \bm{Y}^{(3)}_{k}$, each consisting of $m$ i.i.d. realizations from their respective distributions $F_0, \ldots, F_3$. We then compute the empirical Kolmogorov distances
\begin{equation*}
\hat{d}_{i,k} = \|\hat{F}_0^{(k)} - \hat{F}_i^{(k)}\|_{\infty}, \quad i \in \{1,2,3\},
\end{equation*}
where $\hat{F}_i^{(k)}$ denotes the empirical distribution of sample $\bm{Y}^{(i)}_k$. 
Next, we estimate the shift parameter $\Delta$ by solving the following linear regression problem for each repetition $k$:
\begin{equation*}
(\hat{\alpha}_{k},\hat{\Delta}_{k}) = \argmin_{\alpha,\Delta} \sum_{i=1}^{3} \left(\hat{d}_{i,k} - (\alpha + i \cdot \Delta)\right)^2.
\end{equation*}
$\hat{\Delta}_{k}$ is the slope-based estimate of $\Delta$. 
 Finally, we report the average estimator across all repetitions, defined as,
$\hat{\Delta} = \frac{1}{K} \sum_{k=1}^{K} \hat{\Delta}_k.$
We report our results in \Cref{tab:drift-estimates}.

\begin{table}[h!]
\centering
\resizebox{\textwidth}{!}{%
\begin{tabular}{lccc||ccc}
  & \multicolumn{3}{c||}{Estimated Drift ($m = 20$)}
  & \multicolumn{3}{c}{Estimated Drift ($m = 200$)} \\ 
\cline{2-7}
Distribution & $\Delta=0.01$ & $\Delta=0.05$ & $\Delta=0.1$
             & $\Delta=0.01$ & $\Delta=0.05$ & $\Delta=0.1$ \\
\hline
Truncated Normal ($\mu=0.6,\sigma=0.3$)
    & 0.00 [--0.10, 0.10] & 0.03 [--0.08, 0.15] & 0.09 [--0.03, 0.20]
    & 0.01 [--0.03, 0.04] & 0.05 [0.02, 0.09] & 0.10 [0.06, 0.14] \\
Truncated Exponential ($\lambda=3$)
    & 0.00 [--0.10, 0.10] & 0.03 [--0.08, 0.15] & 0.09 [--0.03, 0.20]
    & 0.01 [--0.03, 0.04] & 0.05 [0.02, 0.09] & 0.10 [0.06, 0.14] \\
Uniform
    & 0.00 [--0.10, 0.10] & 0.03 [--0.08, 0.15] & 0.09 [--0.03, 0.20]
    & 0.01 [--0.03, 0.04] & 0.05 [0.02, 0.09] & 0.10 [0.06, 0.14] \\
\end{tabular}
}
\caption{\textbf{Performance of the slope-based drift estimator $\hat{\Delta}$.} 
We report the mean estimate with its empirical 95\% quantile interval [2.5th, 97.5th percentile] across repetitions.}
\label{tab:drift-estimates}
\end{table}
As shown in \Cref{tab:drift-estimates}, our numerical results indicate that even with a relatively small sample size for each distribution ($m = 20$), the estimated drift parameters are within a reasonable distance  when the true drift parameter is high enough (e.g. above $0.05$). For lower drift parameters, more samples are needed to obtain a precision within $50\%$ of the true value. This observation holds consistently across the three distributional classes we examined. In light of the robustness results we provide in \Cref{sec:misspecification} for scenarios with misspecified $\Delta$, these numerical findings underscore that even limited sample data can yield practically useful estimates for decision-making. We emphasize that further refinement of our ``naive'' Kolmogorov-distance-based estimator can also enhance accuracy.

\end{document}